\documentclass[review]{elsarticle}
\usepackage[top=1.0in, bottom=1.0in, left=1.0in, right=1.0 in]{geometry}
\usepackage{hyperref}
\usepackage{lineno}
\usepackage{color}
\usepackage{amsmath,url}
\usepackage{amsthm}
\usepackage{amssymb}
\usepackage{natbib}

\journal{Arxiv}

\theoremstyle{plain}
\newtheorem{theorem}{Theorem}[section]
\newtheorem{lemma}[theorem]{Lemma}
\newtheorem{corollary}[theorem]{Corollary}
\newtheorem{proposition}[theorem]{Proposition}

\theoremstyle{definition}

\theoremstyle{remark}
\newtheorem{remark}[theorem]{Remark}

\newcommand{\bb}[1]{\boldsymbol{#1}}
\newcommand{\pp}[2]{\frac{\partial{#1}}{\partial{#2}}}

\biboptions{authoryear}

\begin{document}

\begin{frontmatter}

\title{
Statistical Foundation of Variational Bayes Neural Networks}

\author[mymainaddress]{Shrijita Bhattacharya\corref{mycorrespondingauthor}}
\cortext[mycorrespondingauthor]{Corresponding author}
\ead{bhatta61@msu.edu}

\author[mymainaddress]{Tapabrata Maiti}\ead{maiti@msu.edu}

\address[mymainaddress]{Department of Statistics and Probability, Michigan State University}

\begin{abstract}
	Despite the popularism of Bayesian neural networks in recent years, its use is somewhat limited in complex and big data situations due to the computational cost associated with full posterior evaluations. 
	Variational Bayes (VB) provides a useful alternative to circumvent the computational cost and time complexity associated with the generation of samples from the true posterior using Markov Chain Monte Carlo (MCMC) techniques. The efficacy of the VB methods is well established in machine learning literature. However, its potential broader impact is hindered due to a lack of theoretical validity from a statistical perspective. 
	In this paper, we establish the fundamental result of posterior consistency for the mean-field variational posterior (VP) for a feed-forward artificial neural network model. The paper underlines the conditions needed to guarantee that the VP concentrates around Hellinger neighborhoods of the true density function. Additionally, the role of the scale parameter and its influence on the convergence rates  has also been discussed. The paper mainly relies on two results (1) the rate at which the true posterior grows  (2) the rate at which the KL-distance between the  posterior and variational posterior grows. The theory provides a guideline of building prior distributions for Bayesian NN models along with an assessment of accuracy of the corresponding VB implementation.
\end{abstract}


\end{frontmatter}


\section{Introduction}
Bayesian neural networks (BNNs) have been comprehensively studied in the works of \cite{BIS1997}, \cite{Neal1992}, \cite{LAM2001}, etc. More recent developments  which establish the efficacy of BNNs can be found in the works of  \cite{PML},  \cite{MULL2018}, \cite{ALI2018}, \cite{FML2018}, \cite{JAV2020} and the references therein. The theoretical foundation of BNN by \cite{LEE} widens the scope to a broader community. However, with the age of big data applications, the conventional Bayesian approach is computationally inefficient. Thus the alternative computational approaches, such as variational Bayes (VB) become popular among  machine learning and applied researchers. Although, there have been many works on the algorithm development for VB in recent years, the  theoretical advancement on estimation accuracy is rather limited. This article provides statistical validity of neural networks models with variational inference along with some theory-driven practical guidelines for implementation.  

In this article, we mainly focus on feed-forward neural networks with a single hidden layer of inputs and a logistic activation function. Let the number of inputs be denoted by $p$ and the number of hidden nodes by $k_n$ where the number of nodes is allowed to increase as a function of $n$. The true regression function, $E(Y|X=\bb{x})=f_0(\bb{x})$ is modeled as a neural network of the form
\begin{equation}
f(\bb{x})=\beta_0+\sum_{j=1}^{k_n}\beta_j\psi(\gamma_{j0}+\sum_{h=1}^p \gamma_{jh}x_{h})
\end{equation}
where $\psi(u)=1/(1+\exp(-u))$ is the logistic activation function.
With a Gaussian-prior on each of the parameters, \cite{LEE} establishes the posterior consistency of neural networks under the simple setup where the scale parameter $\sigma=V(Y|X=\bb{x})$ is fixed at 1. The results in \cite{LEE} mainly exploit  \cite{BSW}, a fundamental contribution that laid down the framework for posterior consistency in non parametric regression  settings. In this paper, we closely mimick the regression model of \cite{LEE} by assuming $y=f_0(\bb{x})+\xi$ where $f_0(\bb{x})$ is the true regression function and $\xi$ follows $N(0,\sigma^2)$.


The joint posterior distribution of a neural network model is generally evaluated by popular Markov Chain Monte Carlo (MCMC) sampling techniques, 
like Gibbs sampling, Metropolis Hastings, etc. (see, \cite{RN1996}, \cite{Lee2004}, and  \cite{MG2004} for more details). Despite the versatility and popularity of MCMC based approach, the Bayesian estimation suffers from computational costs, scalability, time constraints along with other implementation issues such as choice of proposal densities and generating sample paths.  
Variational Bayes emerged as an important alternative to overcome the drawbacks of the MCMC implementation (see \cite{BL2017}). Many recent works have discussed the application of variational inference to Bayesian neural networks e.g.,  \cite{LG2009}, \cite{ALE2011}, \cite{CAR2012}, \cite{BLUN2015}, \cite{SUN2019}. Although, there is a plethora of literature implementing  variational inference for neural networks, the theoretical properties of the variational posterior in BNNs remain relatively unexplored and this limits the use of this powerful computational tool beyond the machine learning community. 


Some of the previous works that focused on theoretical properties of variational posterior include  the frequentist  consistency of variational inference in parametric models in the presence of latent variables (see \cite{YD2019}).  Optimal risk bounds for mean-field variational Bayes for Gaussian mixture (GM) and Latent Dirichlet allocation (LDA) models have been discussed in \cite{PAT2017}. The work of \cite{YAN2017} propose $\alpha$ variational inference  Bayes risk for GM and LDA models.  A more recent work \cite{ZHA2017}  discusses the variational posterior consistency rates in Gaussian sequence models, infinite exponential families and piece-wise constant models.  In order to   evaluate the validity of a posterior in non-parametric models, one must establish its consistency and rates of contraction. To the best of our knowledge, the problem of posterior consistency, has not been studied in the context of variational Bayes neural network models.


{\it Our contribution:} Our theoretical development of posterior consistency, an essential property in nonparametric Bayesian Statistics, provides confidence in using the variational Bayes neural networks model across the disciplines. Our theoretical results help to assess the estimation accuracy for a given training sample and model complexity. Specifically, we establish the conditions needed for the variational posterior consistency of the feedforward neural networks. We establish that a simple Gaussian mean-field approximation is good enough to achieve consistency for the variational posterior. In this direction, we show that $\varepsilon$- Hellinger neighborhood of the true density function receives close to 1 probability under the variational posterior. For the true posterior density ( \cite{LEE}), the posterior probability of an $\varepsilon$- Hellinger neighborhood  grows at the rate $1-e^{-\epsilon n^\delta}$. In contrast, we show for the variational posterior this rate becomes $1-\epsilon/n^{\delta}$. The reason for this difference is explained by two folds: (1) first, the KL-distance between the variational posterior and the true posterior does not grow at a rate greater than $n^{1-\delta}$ for some $0\leq \delta<1$, (2) second, the posterior probability of $\varepsilon$- Hellinger neighborhood grows at the rate $1-e^{-\epsilon n^\delta}$, thus, the variational posterior probability must grow at the rate $1-\epsilon/n^{\delta}$, otherwise the rate of growth of the KL-distance cannot be controlled. We also give the conditions on the approximating neural network and the rate of growth in the number of nodes needed to ensure that the variational posterior achieves consistency. As a last contribtuion, we show that the VB estimator of the regression function  converges to the true regression function.

Further, our investigation shows that although the variational posterior(VP) is asymptotically consistent, posterior probability of $\varepsilon-$Hellinger neighborhoods does not converge to 1 as fast as the true posterior. In addition, one requires that the absolute value of the parameters in the approximating neural network function grow at a controlled rate (less than $n^{1-\delta}$ for some $0\leq \delta<1$), a condition not needed in dealing with MCMC based implementation. When the absolute value of the parameters grow as a polynomial function of $n$ ($O(n^v), v>1$), one can choose a flatter prior (a prior whose variance increases with $n$) in order to guarantee VP consistency.  

VP consistency has been established irrespective of whether $\sigma$ is known or unknown and the differences in practice have been discussed. It has been shown that one must guard against using  Gaussian distributions as a variational family for $\sigma$. Since the KL-distance between variational posterior and true posterior must be controlled, one must ensure that quantities like $E(\log X)$ and $E(1/X^2)$ are defined under the variational distribution of $\sigma$. We thereby discuss two sets of variational family on $\sigma$, (1) an inverse gamma-distribution, (2) a normal distribution on the log-transformed $\sigma$. While the second approach may seem intuitively appealing if one were to use fully Gaussian variational families, it comes with a drawback. Indeed, under the reparametrized $\sigma$, the variational posterior is consistent if the rate of growth in the number of nodes is slower than under the original parameter models. However, a smaller growth in the number of nodes makes it more and more difficult to find an approximating neural network which converges fast enough to the true function.

%


The outline of the paper is as follows. In Section 2, we present the notation and the terminology of consistency for variational posterior. In Section 3, we present the consistency results when the scale parameter is known. In Section 4, we present the consistency of an unknown scale parameter under two sets of variational families. In Section 5, we show that the Bayes estimates obtained from the variational posterior converge to the true regression function and scale parameter. Finally, Section 5 ends with a discussion and conclusions from our current work.

\section{Model and Assumptions}

\noindent Suppose the true regression model has the form:
$$y_i=f_0(\bb{x}_i)+\xi_i$$
where $\xi_1, \cdots, \xi_n$ are i.i.d. $N(0,\sigma_0^2)$ random variables and  the feature vector $\bb{x}_1, \cdots \bb{x}_n$ with $\bb{x}_i \in \mathbb{R}^p$. For the purposes of this paper, we assume that the number of covariates $p$ is fixed.

\noindent Thus, the true conditional density of $Y|X=\bb{x}$ is 
\begin{equation}
\label{e:l0-def}
l_0(y, \bb{x})\propto \prod_{i=1}^n \exp(-\frac{1}{2 \sigma_0^2}(y-f_0(\bb{x}))^2)
\end{equation}
which implies the true likelihood function is
\begin{equation}
\label{e:L0-def}
L_0=\prod_{i=1}^n l_0(y_i,\bb{x}_i)
\end{equation}

\noindent {\bf Universal approximation:}
By \cite{HOR}, for every function $f_0$ such that $\int f_0^2(x) dx<\infty$, there exists a neural network $f$ such that $||f-f_0||_2<\epsilon$. This led to the ubiquitous use of neural networks as a modeling approximation to a wide class of regression functions.

In this paper, we assume that the true regression function $f_0$ can be approximated by a neural network 
\begin{equation}
\label{e:theta-def}
    f_{\bb{\theta}_{n}}(\bb{x})=\beta_{0}+\sum_{j=1}^{k_n}\beta_{j}\psi(\gamma_{j}^\top \bb{x}),\:\: \bb{\theta}_n=(\beta_j,\gamma_{jh})_{j\in J, h\in H}, \:\: J=\{0, \cdots, k_n\},\:\: H=\{0,\cdots,p\}
\end{equation}
where $k_n$, the number of nodes increases as a function of $n$, while $p$, the number of covariates is fixed. Thus, the total number of parameters grow at the same rate as the number of nodes, i.e. $K(n)=1+k_n(p+1) \sim k_n$.

Suppose there exists a neural network  $f_{\bb{\theta}_{0n}}(\bb{x})=\beta_{00}+\sum_{j=1}^{k_{n}}\beta_{j0}\psi(\gamma_{j0}^\top \bb{x})$ such that
\begin{align}
\label{e:f-f0}
(A1) &\hspace{10mm}||f_{\bb{\theta}_{0n}}-f_0||_2=o(n^{-\delta})
\end{align}

Note that if $f_0$ is a neural network function itself, then (A1) holds trivially for all  $0\leq \delta<1$ irrespective of the choice of $k_n$. Theorem 2 of \cite{SIE2019} showed that with $k_n=n$, $\delta$ can be chosen between $0\leq \delta<1/2$. Mimicking the steps of Theorem 2, \cite{SIE2019}, it can be shown that with $k_n=n^{a}, a>1/2$, $\delta$ can be chosen anywhere in the range $0\leq \delta< a-1/2$.  For a given choice of $k_n$, whether (A1) holds or not depends on the entropy of the true function. Assumptions of similar form can also be found in  \cite{SHE1997} (see conditions C and $C^{\prime}$) and \cite{SHE2019A} (see condition C3).

Note that the condition (A1) characterizes the rate at which a neural network function approaches to the true function. The next set of conditions characterize the rate at which the coefficients of the approximating neural network solution grow. Suppose, one of the following two conditions hold:
\begin{align}
\label{e:t-kc}
(A2) &\hspace{10mm} \sum_{i=1}^{K(n)}\theta_{i0n}^2 =o(n^{1-\delta}), \; 0\le \delta < 1\\
\label{e:t-kc-1}
(A3) &\hspace{10mm} \sum_{i=1}^{K(n)}\theta_{i0n}^2=O(n^v), \;\; v\ge 1
\end{align}

Note that condition (A2) ensures that sum of squares of the coefficients grow at a rate slower than $n$. \cite{WHI1990}  proved consistency properties of feed forward neural networks with $\sum_{i=1}^{K(n)}|\theta_{i0n}| =o(n^{1/4})$ which implies $\sum_{i=1}^{K(n)}|\theta_{i0n}|^2\leq (\sum_{i=1}^{K(n)} |\theta_{i0n}|)^2=o(n^{1/2})$, i.e. $0\leq \delta<1/2$.  \cite{BL2017} studied the consistency properties for parametric models wherein one requires the assumption $-\log p(\theta_0)$ be bounded (see Relations (44) and (53) in \cite{BL2017}). With a normal prior of the form $p(\bb{\theta}_n) \propto \exp(-\sum_{i=1}^{K(n)}\theta_{in}^2)$, the same condition reduces to $\sum_{i=1}^{K(n)}\theta_{i0n}^2$ bounded at a suitable rate. Indeed, condition (A2) guarantees that the rate of growth KL-distance between the true and the variational posterior is well controlled. 

Condition (A3) is a relaxed version of (A2), where the  sum of squares of the  coefficients is allowed to grow at a  rate in polynomial in $n$. A standard prior independent of $n$ might fail to guarantee convergence. We thereby assume a flatter prior whose variance increases with $n$ in order to allow for consistency through variational bayes. Note that if $f_0$ is a neural network function itself, conditions (A2) and (A3) hold trivially.
  
\noindent {\bf Kullback-Leibler divergence:}
Let $P$ and $Q$ be two probability distributions, with density $p$ and $q$ respectively, then
$$d_{KL}(q,p)=\int_{\mathcal{X}} \log \frac{p(x)}{q(x)}q(x)dx$$

%

\noindent {\bf Hellinger distance:}  Let $P$ and $Q$ be two probability distributions with density $p$ and $q$ respectively, then
$$d_{H}(q,p)=\int_{\mathcal{X}} (\sqrt{q(x)}-\sqrt{p(x)})^2 dx$$

\noindent {\bf Distribution of the feature vector:} In order to establish posterior consistency, we assume that the feature vector $\bb{x} \sim U(0,1)^p$. Although, this is not a requirement for the model, it simplifies steps of the proof since the joint density  function of (Y,X) simplifies as
\begin{equation}
    \label{e:cond-y-x}
    g_{Y,X}(y, \bb{x})=g_{Y|X}(y|\bb{x})g_X(\bb{x})=g_{Y|X}(y|\bb{x})
\end{equation}
Thus, it suffices to deal with the conditional density of $Y|X=\bb{x}$.


\section{Consistency of variational posterior with $\sigma$ known}

\label{sec:vb-s-known}
In this section, we begin with the simple model where the scale parameter $\sigma_0$ is known. For a simple Gaussian mean field family as in \eqref{e:var-family}, we establish that variational posterior is consistent as long as assumptions (A1), (A2) and (A3) hold. We also discuss, how the rates contrast with those in \cite{LEE} which established the posterior consistency of the true posterior.

\noindent {\bf Sieve Theory:}
Let $\bb{\omega}_n=\bb{\theta}_n$, then
\begin{equation}
\label{e:lw-def}
l_{\bb{\omega}_n}(y,\bb{x})=\frac{1}{\sqrt{2\pi \sigma_0^2}}\exp\Big(-\frac{1}{2\sigma_0^2} (y-f_{\bb{\theta}_n}(\bb{x}))^2\Big)
\end{equation}
where $\bb{\theta}_n$ and $f_{\bb{\theta}_n}$ are defined in \eqref{e:theta-def}.  The sieve is then defined as:
\begin{align}
\label{e:G-def}
 \mathcal{G}_n=\Big\{l_{\bb{\omega}_n}(y,\bb{x}), \bb{\omega}_n\in \mathcal{F}_n\Big\}\hspace{10mm}\mathcal{F}_n=\Big\{ (\bb{\theta}_n):|\theta_{in}|\leq C_n\Big\}
\end{align}

\noindent {\bf Likelihood:}
\begin{equation}
\label{e:Lw-def}
L(\bb{\omega}_n)=\prod_{i=1}^n l_{\bb{\omega}_n}(y_i,\bb{x}_i)
\end{equation}

\noindent{\bf Posterior:} Let $p(\bb{\omega}_n)$ denote the prior on $\bb{\omega}_n$. Then, the posterior is given by
\begin{equation}
\label{e:pi-def}
\pi(\bb{\omega}_n|\bb{y}_n,\bb{X}_n))=\frac{L(\bb{\omega}_n)p(\bb{\omega}_n)}{\int L(\bb{\omega}_n)p(\bb{\omega}_n)d\bb{\omega}_n}
\end{equation}

\noindent {\bf Variational Family:} Variational family for $\bb{\omega}_n$ is given by
\begin{equation}
\label{e:var-family}
\mathcal{Q}_n=\left\{q:q(\bb{\omega}_n)=\prod_{i=1}^{K(n)}\frac{1}{\sqrt{2\pi \tilde{s}^2_{in}}}e^{-\frac{(\theta_{in}-\tilde{m}_{in})^2}{2\tilde{s}_{in}^2}}\right\}
\end{equation}
Let the variational posterior be denoted by 
\begin{equation}
\label{e:VB-optimizer}
\pi^*(\bb{\omega}_n)=\underset{{q \in \mathcal{Q}_n}}{\text{argmin}}  d_{KL}(q(.),\pi(.|\bb{y}_n,\bb{X}_n))
\end{equation}

\noindent{\bf Hellinger neighborhood:} Define the neighborhood of the true density $l_0$ as
\begin{equation}
\label{e:hell-def}
\mathcal{V}_\varepsilon=\{\bb{\omega}_n: d_H(l_0,l_{\bb{\omega}_n})<\varepsilon \}
\end{equation}
where the Hellinger distance $d_{H}(l_0,l_{\bb{\omega}_n})$ given by
$$d_{H}(l_0,l_{\bb{\omega}_n})=\int \int \left(\sqrt{l_{\bb{\omega}_n}(\bb{x},y)}-\sqrt{l_0(\bb{x},y)}\right)^2 d\bb{x}dy$$
Note that the above simplified of the Hellinger distance is due to \eqref{e:cond-y-x}.

In the following two theorems for two class of priors, we establish the posterior consistency of $\pi^*$, i.e. the variational posterior concentrates in $\varepsilon-$ small Hellinger neighborhoods of the true density $l_0$. Note that, assumptions (A2) and (A3) impose a restriction on the  rate of growth of the sum of squares of the coefficients of the approximating neural network solution. With (A2), we show that a standard normal prior on all the parameters works. However, under the more weaker assumption (A3), a normal prior whose variance increases with $n$ is needed. Additionally, we show that for the  variational posterior to achieve consistency,  the number of parameters or  equivalenty the number of nodes $k_n$ need to grow in a controlled fashion. 

\begin{theorem}
	\label{thm:var-cons}
Suppose the number of nodes $k_n$ satisfy
\begin{align}
\label{e:k-n-cond}
(C1) \hspace{10mm}k_n &\sim n^a
\end{align} 
In addition, suppose assumptions (A1) and  (A2) hold for some $0\leq \delta<1-a, \; $.

\noindent Then, with normal prior for each entry in $\boldsymbol{\omega}_n$ as follows 
\begin{equation}
\label{e:prior-t} p(\bb{\omega}_n)=\prod_{i=1}^{K(n)}\frac{1}{\sqrt{2\pi \zeta^2}}e^{-\frac{\theta_{in}^2}{2\zeta^2}}
\end{equation}
we have
$$\pi^*(\mathcal{V}_\varepsilon^c)=o_{P_0^n}(n^{-\delta})$$
\end{theorem}

Note that conditions \eqref{e:k-n-cond} and \eqref{e:prior-t} agree with those assumed in Theorem 1 of \cite{LEE}. Since $\pi^*(\mathcal{V}^{c}_\varepsilon)=o_{P_0}(n^{-\delta})$, the variational posterior is consistent with $\delta$ as small as 0. Indeed $\delta=0$ imposes the least restriction on the convergence rate and coefficient growth rate of the true function (see assumptions (A1) and (A2)). As $\delta$ grows, restrictions on the approximating neural function increase but that guarantees faster convergence of the variational posterior. Expanding upon the Bayesian posterior consistency established in \cite{LEE}, one can show that $\pi(\mathcal{V}_\varepsilon^c|\bb{y}_n,\bb{X}_n)\leq o_{P_0^n}(e^{-n^\delta})$ for any $0\leq \delta<1$ (see Relation (88) in \cite{LEE}). Thus, probability of $\varepsilon-$ Hellinger neighborhood grows at the rate $1-\epsilon(1/n)^{\delta}$ for variational posterior in contrast to that of $1-\epsilon(e^{-n})^{\delta}$ for true posterior. For parametric models, the rate of growth of the variational posterior was found to be $1-\epsilon(1/n)$ (see second equation 2 on page 38 of \cite{BL2017}). Note that the consistency of true posterior requires no assumptions on the approximating neural network function whereas for the variational posterior, both assumptions (A1) and (A2) must be satisfied to guarantee convergence.


\begin{theorem}
	\label{thm:var-cons-1}
	Suppose the number of nodes $k_n$ satisfy condition (C1).
In addition, suppose assumptions (A1) and (A3)  hold for some $0\leq \delta<1-a$ and $v>1$. 

 \noindent  Then, with normal prior for each entry in $\boldsymbol{\omega}_n$ as follows 
	\begin{equation}
	\label{e:prior-t-1} p(\bb{\omega}_n)=\prod_{i=1}^{K(n)}\frac{1}{\sqrt{2\pi \zeta^2n^u}}e^{-\frac{\theta_{in}^2}{2\zeta^2n^u}}, u>v
	\end{equation}
	we have
	$$\pi^*(\mathcal{V}_\varepsilon^c)=o_{P_0^n}(n^{-\delta})$$
\end{theorem}

Observe that the consistency rate in Theorem \ref{thm:var-cons-1}  agrees to the one in Theorem \ref{thm:var-cons}. In order to prove both theorems \ref{thm:var-cons} and \ref{thm:var-cons-1}, a crucial step is to show that $d_{KL}(\pi^*(.),\pi(.|\bb{y}_n,\bb{X}_n))=o_{P_0^n}(n^{1-\delta})$. In order to show this, we show that $d_{KL}(q(.),\pi(.|\bb{y}_n,\bb{X}_n))=o_{P_0}(n^{1-\delta})$ for some $q  \in \mathcal{Q}_n$. Indeed this choice of $q$ varies in order to adjust for changing nature of the prior from \eqref{e:prior-t} to \eqref{e:prior-t-1} (see statements (1) and (2) in Lemma \ref{lem:f-bound}).

We next present the proof of Theorems \ref{thm:var-cons} and \ref{thm:var-cons-1}. The first crucial step of the proof is to establish that the $d_{KL}(\pi^*(.),\pi(.|\bb{y}_n,\bb{X}_n) )$ is bounded below by a quantity which is determined by the rate of consistency of the true posterior (see the quantities $A_n$ and $B_n$ in the proof below). The second crucial step towards the proof is to show $d_{KL}(\pi^*(.),\pi(.|\bb{y}_n,\bb{X}_n) )$ is bounded above at a rate which can be greater than the rate of its lower bound iff the variation posterior is consistent.

\begin{proof}[Proof of Theorems \ref{thm:var-cons} and \ref{thm:var-cons-1}]
With $\mathcal{V}_{\varepsilon}$ as in \eqref{e:hell-def}, we have		\begin{align}
		\label{e:dk-b}
	d_{KL}(\pi^*(.), \pi(.|\bb{y}_n,\bb{X}_n))=\underbrace{\int_{\mathcal{V}_\varepsilon} \pi^*(\bb{\omega}_n)\log \frac{\pi^*(\bb{\omega}_n)}{\pi(\bb{\omega}_n|\bb{y}_n,\bb{X}_n)}d\bb{\omega}_n}_{\textcircled{3}}+\underbrace{\int_{\mathcal{V}_\varepsilon^c} \pi^*(\bb{\omega}_n)\log \frac{\pi^*(\bb{\omega}_n)}{\pi(\bb{\omega}_n|\bb{y}_n,\bb{X}_n)}d\bb{\omega}_n}_{\textcircled{4}}
	\end{align}
Without loss of generality, $\pi^*(\mathcal{V}_\varepsilon)>0$, $\pi^*(\mathcal{V}_\varepsilon^c)>0$.

	\begin{align} 
	\textcircled{3} \nonumber&=-\pi^*(\mathcal{V}_\varepsilon)\int_{\mathcal{V}_\varepsilon} \frac{\pi^*(\bb{\omega}_n)}{\pi^*(\mathcal{V}_\varepsilon)}\log \frac{\pi(\bb{\omega}_n|\bb{y}_n,\bb{X}_n)}{\pi^*(\bb{\omega}_n)}d\bb{\omega}_n \\
	\nonumber&\geq -\pi^*(\mathcal{V}_\varepsilon)\log \int_{\mathcal{V}_\varepsilon}  \frac{\pi^*(\bb{\omega}_n)}{\pi^*(\mathcal{V}_\varepsilon)}\frac{\pi(\bb{\omega}_n|\bb{y}_n,\bb{X}_n)}{\pi^*(\bb{\omega}_n)}d\bb{\omega}_n \hspace{5mm}\text{ Jensen's inequality}
	\\
	\nonumber&\geq \pi^*(\mathcal{V}_\varepsilon) \log 
	\frac{\pi^*(\mathcal{V}_\varepsilon)}{\pi(\mathcal{V}_\varepsilon|\bb{y}_n,\bb{X}_n)}\geq \pi^*(\mathcal{V}_\varepsilon)\log \pi^*(\mathcal{V}_\varepsilon)\hspace{10mm} \text{ since } \log \pi(\mathcal{V}_\varepsilon|\bb{y}_{n},\bb{X}_n)\leq 0
	\end{align}
	
\noindent Similarly,
	\begin{align}
	\label{e:4-lb}
\nonumber	\textcircled{4}&\geq   \pi^*(\mathcal{V}_\varepsilon^c) \log 
	\frac{\pi^*(\mathcal{V}_\varepsilon^c)}{\pi(\mathcal{V}_\varepsilon^c|\bb{y}_n,\bb{X}_n)}\\
	&\geq \pi^*(\mathcal{V}_\varepsilon^c) \log \pi^*(\mathcal{V}_\varepsilon^c)-\pi^*(\mathcal{V}_\varepsilon^c)\log \pi(\mathcal{V}_\varepsilon^c|\bb{y}_n,\bb{X}_n)
 \end{align}
 Now let us consider
 \begin{align}
 \label{e:pi-post}
\nonumber \log   \pi(\mathcal{V}_\varepsilon^c|\bb{y}_n,\bb{X}_n)&= \log \int_{\mathcal{V}_\varepsilon^c} \frac{L(\bb{\omega}_n)p(\bb{\omega}_n)d\bb{\omega}_n}{\int L(\bb{\omega}_n)p(\bb{\omega}_n)d\bb{\omega}_n}\\
&=\underbrace{\log \int_{\mathcal{V}_\varepsilon^c} (L(\bb{\omega}_n)/L_0)p(\bb{\omega}_n)d\bb{\omega}_n}_{A_n}\underbrace{-\log \int( L(\bb{\omega}_n)/L_0)p(\bb{\omega}_n)d\bb{\omega}_n}_{B_n}
\end{align}
Using \eqref{e:pi-post} in \eqref{e:4-lb}, we get 
\begin{align}
\label{e:4-lb-1}
\textcircled{4}\geq\pi^*(\mathcal{V}_\varepsilon^c) \log \pi^*(\mathcal{V}_\varepsilon^c)-\pi^*(\mathcal{V}_\varepsilon^c)A_n-\pi^*(\mathcal{V}_\varepsilon^c)B_n
\end{align}

\noindent Combining \eqref{e:dk-b} and \eqref{e:4-lb-1}, we get
\begin{align}
\label{e:dk-lb}
d_{KL} (\pi^*(.), \pi(.|\bb{y}_n,\bb{X}_n)) &\geq \pi^*(\mathcal{V}_\varepsilon)\log \pi^*(\mathcal{V}_\varepsilon)+ \pi^*(\mathcal{V}_\varepsilon^c)\log \pi^*(\mathcal{V}_\varepsilon^c)-\pi^*(\mathcal{V}_\varepsilon^c)A_n-\pi^*(\mathcal{V}_\varepsilon^c)B_n\\
&\geq -\log 2 -\pi^*(\mathcal{V}_\varepsilon^c)A_n-\pi^*(\mathcal{V}_\varepsilon^c)B_n
\end{align}
where the last inequality follows since $x\log x+(1-x)\log(1-x)\geq -\log 2$ for $0<x<1$. 

\noindent Therefore,
\begin{equation}
\label{e:dk-lb-0}
\boxed{
 d_{KL} (\pi^*(.), \pi(.|\bb{y}_n,\bb{X}_n))+\log 2+\pi^*(\mathcal{V}_\varepsilon^c)B_n\geq -\pi^*(\mathcal{V}_\varepsilon^c)A_n}
\end{equation}
By Proposition \ref{lem:v-bound},
\begin{align*}
&-A_n \geq -\log 2 + n\varepsilon^2+o_{P_0^n}(1)\\
&\implies -A_n\pi^*(\mathcal{V}_\varepsilon) \geq -\log 2 + n\varepsilon^2\pi^*(\mathcal{V}_\varepsilon)+o_{P_0^n}(1)\\
&\implies \pi^*(\mathcal{V}_\varepsilon^c)n\varepsilon^2\leq \nonumber d_{KL} (\pi^*(.), \pi(.|\bb{y}_n,\bb{X}_n))+2\log 2+\pi^*(\mathcal{V}_\varepsilon^c)B_n+o_{P_0^n}(1)
\end{align*}

\noindent By Proposition \ref{lem:kl-bound},
$$\pi^*(\mathcal{V}_\varepsilon^c)B_n= o_{P_0^n}(n^{1-\delta})$$

\noindent By Proposition \ref{lem:q-bound}, 
\begin{align*}
d_{KL} (\pi^*(.), \pi(.|\bb{y}_n,\bb{X}_n))= o_{P_0^n}(n^{1-\delta})
\end{align*}
Therefore,
$$\pi^*(\mathcal{V}_\varepsilon^c)\leq  o_{P_0^n}(n^{-\delta})+o_{P_0^n}(n^{-1})=o_{P_0^n}(n^{-\delta})$$
\end{proof}

In the above proof we have assumed $\pi^*(\mathcal{V}_\varepsilon)>0$, $\pi^*(\mathcal{V}_\varepsilon^c)>0$. If $\pi^*(\mathcal{V}_\varepsilon^c)=0$, there is nothing to prove. If $\pi^*(\mathcal{V}_\varepsilon)=0$, then following the steps of the proof, we will get $\varepsilon^2=o_{P_0}(n^{-\delta})$ which is a contradiction. 

\noindent The main step in the above proof is \eqref{e:dk-lb-0} which we discuss next. The quantity $e^{A_n}$ is indeed decomposed into two parts
$$e^{A_n}=\int_{\mathcal{V}_\varepsilon^c \cap \mathcal{F}_n} (L(\bb{\omega})_n)/L_0)p(\bb{\omega}_n)d\bb{\omega}_n+\int_{\mathcal{V}_\varepsilon^c \cap \mathcal{F}_n^c}  (L(\bb{\omega})_n)/L_0)p(\bb{\omega}_n)d\bb{\omega}_n$$
Whereas the first term is controlled using the Hellinger bracketing entropy of $\mathcal{F}_n$, the second term is controlled by the fact that the prior gives negligible probability outside $\mathcal{F}_n$. Thus, the main factor influencing $e^{A_n}$ is a suitable choice of the sequence of spaces $\mathcal{F}_n$. Indeed our choice of $\mathcal{F}_n$ is same as that in \cite{LEE} with $k_n \sim n^a$ and $C_n=e^{n^{b-a}}$. Such a choice allows one to control the Hellinger bracketing entropy of $\mathcal{F}_n$ while controlling the prior mass for $\mathcal{F}_n^c$ also at the same time.

The second quantity $B_n$ is controlled by the rate at which the prior gives mass to shrinking KL neighborhoods of the true density $l_0$. Indeed, the quantity $B_n$ appears again when computing bounds on  $d_{KL}(q(.),\pi(.|\bb{y}_n,\bb{X}_n)$ for some $q\in \mathcal{Q}_n$ (see $\textcircled{3}$ in Proposition \ref{lem:q-bound}). If $\delta=0$, $B_n$ can be controlled even without assumptions (A1) and (A2). However, if $\delta>0$, assumptions (A1) and (A2) are needed in order to guarantee that the $B_n$ grows at a rate less than $n^{1-\delta}$.

The last quantity, $d_{KL} (\pi^*(.), \pi(.|\bb{y}_n,\bb{X}_n)) $ is controlled at a rate less than $n^{1-\delta}$ by showing that there exists a $q \in \mathcal{Q}_n$ (see \eqref{e:q-def} and \eqref{e:q-def-1}) such that $d_{KL}(\pi^*(.), \pi(.|\bb{y}_n,\bb{X}_n))=o_{P_0^n}(n^{1-\delta})$. Both assumptions (A1) and (A2) play an important role in guaranteeing that such a $q$ does exist.

\section{Consistency of variational posterior with $\sigma$ unknown}
\label{sec:vb-s-unknown}
In this section, we now assume that the scale parameter $\sigma$ is unknown. In this case, our approximating variational family is slightly different from \eqref{e:VB-optimizer}. Whereas, we still assume a mean field Gaussian family on $\bb{\theta}_n$, our approximating family for  $\sigma$  cannot be Gaussian. An important criterion to guarantee the consistency of variational posterior is to ensure $\int d_{KL}(l_0,l_{\bb{\omega}_n})q(\bb{\omega}_n)d\bb{\omega}_n$ is well bounded (see Lemma \ref{lem:c-bound-0}). When $\sigma$ is unknown, $d_{KL}(l_0,l_{\bb{\omega}_n})$ involves terms like $\log \sigma$ and $1/\sigma^2$ both of whose integrals are undefined under a normally distributed $q$. We thereby adopt two versions of $q$ for $\sigma$, firstly an inverse gamma distribution of $\sigma$ and secondly a normal distribution on the log transformed $\sigma$ (see Sections \ref{sec:vb-s-unknown-1} and \ref{sec:vb-s-unknown-2} respectively). Both the transforms have their respective advantage in terms of determining the rate of consistency of the variational posterior. In this section, we work only with assumption (A2). We can handle (A3)  in a way exactly similar to Section \ref{sec:vb-s-known}.

\subsection{Inverse-gamma prior on $\sigma$}
\label{sec:vb-s-unknown-1}
\noindent {\bf Sieve Theory:}
Let $\bb{\omega}_n=(\bb{\theta}_n,\sigma^2)$ where $\bb{\theta}_n$ and $f_{\bb{\theta}_n}$ are defined in \eqref{e:theta-def}, then
\noindent \begin{equation}
\label{e:lik-0-v}
l_{\bb{\omega}_n}(y,\bb{x})=\frac{1}{\sqrt{2\pi \sigma^2}}\exp\Big(-\frac{1}{2\sigma^2} (y-f_{\bb{\theta}_n}(\bb{x}))^2\Big)
\end{equation}
The sieve is defined as follows.
\begin{align}
\label{e:G-def-v}
 \mathcal{G}_n=\Big\{l_{\bb{\omega}_n}(y,\bb{x}), \bb{\omega}_n\in \mathcal{F}_n\Big\}\hspace{5mm}\mathcal{F}_n=\Big\{ (\bb{\theta}_n,\sigma^2):|\theta_{in}|\leq C_n, 1/C_n^2\leq \sigma^2\leq  D_n\Big\}
\end{align}
\noindent The definitions for likelihood, posterior and  Hellinger neighborhood  agree with those given in \eqref{e:Lw-def}, \eqref{e:pi-def} and  \eqref{e:hell-def} as in Section \ref{sec:vb-s-known}.

\noindent {\bf Prior distribution:} \noindent We propose a normal prior on each $\theta_{in}$ and an inverse gamma prior of $\sigma^2$. 
\begin{equation}
\label{e:prior-t-v} p(\bb{\omega}_n)=\frac{\lambda^{\alpha}}{\Gamma(\alpha)}\Big(\frac{1}{\sigma^2}\Big)^{\alpha+1}e^{-\frac{\lambda}{\sigma^2}}\prod_{i=1}^{K(n)}\frac{1}{\sqrt{2\pi \zeta^2}}e^{-\frac{\theta_{in}^2}{2\zeta^2}}
\end{equation}

\noindent {\bf Variational Family:} Variational family for $\bb{\omega}_n$ is given by
\begin{equation}
\label{e:var-family-v}
\mathcal{Q}_n=\left\{q:q(\bb{\omega}_n)=\frac{\tilde{b}_n^{\tilde{a}_n}}{\Gamma(\tilde{a}_n)}\Big(\frac{1}{\sigma^2}\Big)^{\tilde{a}_n+1}e^{-\frac{\tilde{b}_n}{\sigma^2}}\prod_{i=1}^{K(n)}\frac{1}{\sqrt{2\pi \tilde{s}^2_{in}}}e^{-\frac{(\theta_{in}-\tilde{m}_{in})^2}{2\tilde{s}_{in}^2}}\right\}
\end{equation}
The variational posterior has the same definition as in \eqref{e:VB-optimizer}.

The following theorem shows that when the $\sigma$ parameter is unknown, the variational posterior is still consistent, however the rate decreases by an amount of  $n^\epsilon$.

\begin{theorem}
	\label{thm:var-var}
Suppose the number of nodes satisfy condition (C1). In addition, suppose assumptions (A1) and (A2) hold for some $0< \delta<1-a$. Then  for any $\epsilon>0$.
$$\pi^*(\mathcal{V}_\varepsilon^c)=o_{P_0^n}(n^{\epsilon-\delta})$$
\end{theorem}
\noindent Note that by Theorem \ref{thm:var-cons}, the posterior is consistent iff $\epsilon-\delta<0$ which is indeed the case as long as $\delta>0$. Whether such a $\delta$ exists or not depends on the entropy of the function $f_0$ (see the discussion section in \cite{SHE2019A}). Mimicking the steps of Theorem 2, \cite{SIE2019} it can be shown that with $k_n=n^a$, $a>1/2$, $\delta$ can be chosen anywhere in the range $0\leq \delta<1/2$.

\begin{proof}
	The proof mimics the steps in the proof of Theorems \ref{thm:var-cons} and \ref{thm:var-cons-1} till equation \eqref{e:dk-lb-0}. 
	
\noindent By Proposition \ref{lem:v-bound-v} for any $0<r<1$,
\begin{align*}
&-A_n \geq -\log 2 + n^r\varepsilon^2+o_{P_0^n}(1)\\
&-A_n\pi^*(\mathcal{V}_\varepsilon) \geq -\log 2 + n^r\varepsilon^2\pi^*(\mathcal{V}_\varepsilon)+o_{P_0^n}(1)\\
&\implies 
\pi^*(\mathcal{V}_\varepsilon^c)n^r\varepsilon^2\leq \nonumber d_{KL} (\pi^*(\bb{\omega}_n), \pi(\bb{\omega}_n|\bb{y}_n,\bb{X}_n))+2\log 2+\pi^*(\mathcal{V}_\varepsilon^c)B_n+o_{P_0^n}(1)
\end{align*}

\noindent By Proposition \ref{lem:kl-bound-v},
$$\pi^*(\mathcal{V}_\varepsilon^c)B_n= o_{P_0^n}(n^{1-\delta})$$

\noindent By Proposition \ref{lem:q-bound-v},
\begin{align*}
d_{KL} (\pi^*(\bb{\omega}_n), \pi(\bb{\omega}_n|\bb{y}_n,\bb{X}_n))= o_{P_0^n}(n^{1-\delta})
\end{align*}
Therefore, with $r=1-\epsilon$, we have
$$\pi^*(\mathcal{V}_\varepsilon^c)\leq  o_{P_0^n}(n^{1-\delta-r})+o_{P_0^n}(n^{-r})=o_{P_0^n}(n^{\epsilon-\delta})+o_{P_0^n}(n^{\epsilon-1})=o_{P_0^n}(n^{\epsilon-\delta})$$
\end{proof}
Similar to the proof of Theorem \ref{thm:var-cons}, the quantity $e^{A_n}$ is indeed decomposed into two parts
$$e^{A_n}=\int_{\mathcal{V}_\varepsilon^c \cap \mathcal{F}_n} (L(\bb{\omega})_n)/L_0)p(\bb{\omega}_n)d\bb{\omega}_n+\int_{\mathcal{V}_\varepsilon^c \cap \mathcal{F}_n^c}  (L(\bb{\omega})_n)/L_0)p(\bb{\omega}_n)d\bb{\omega}_n$$
Whereas the first term is controlled using the Hellinger bracketing entropy of $\mathcal{F}_n$ at the rate $e^{-n\varepsilon^2}$, the second term is controlled by the prior probability of $\mathcal{F}_n^c$ at $e^{-n^r}$, $0<r<1$. Since the prior probability of $\mathcal{F}_n^c$ is now controlled at a comparatively slightly smaller rate than that of Theorem \ref{thm:var-cons}, hence the additional $\epsilon$ term in the overall consistency rate of variational posterior.

\begin{remark}
With $k_n\sim n^a$ and $\mathcal{F}_n$ as in \eqref{e:G-def-v}, we choose $C_n=e^{n^{b-a}}$ and $D_n=e^{n^b}$, $0<a<b<1$ to prove the posterior consistency statement of Theorem \ref{thm:var-var}. Suitably choosing $\mathcal{F}_n$ as a function of $\varepsilon$ one may be able to refine the proof to obtain a rate of $o_{P_0^n}(n^{-\delta})$ instead of $o_{P_0^n}(n^{\epsilon-\delta})$. However the proof becomes more involved and such a $\varepsilon-$ dependent choice of $\mathcal{F}_n$ has been avoided for the purposes of this paper.
\end{remark}

\begin{remark}
When $\sigma$ is unknown, in order to control $d_{KL} (\pi^*(.), \pi(.|\bb{y}_n,\bb{X}_n)) $ at a rate less than $n^{1-\delta}$, $q(\bb{\theta}_n)$ has the same form as in the proof of Theorem \ref{thm:var-cons}. However, we cannot choose a normally distributed $q$ for $\sigma^2$. The convergence of $d_{KL} (\pi^*(.), \pi(.|\bb{y}_n,\bb{X}_n)) $ is determined by the term $\int d_{KL}(l_0,l_{\bb{\omega}_n})q(\bb{\omega}_n)d\bb{\omega}_n$ which involves terms like $\frac{1}{2\sigma^2}$ and $\log \sigma^2$ (see \eqref{e:dk-bound-1}). The expectation of these terms is not defined under a normal $q$ but well defined under an inverse gamma distribution, hence an inverse-gamma variational family of $q(\sigma^2)$.
\end{remark}

\subsection{Normal prior on log transformed $\sigma$}
\label{sec:vb-s-unknown-2}
Given, the wide popularity of Gaussian mean field approximation, we next use a normal variational distribution on the log-transformed $\sigma$ and compare and contrast it to the case where an inverse-gamma variational distribution on the scale parameter. In Section 3.3 of \cite{BL2017}, it has been posited that a Gaussian VB posterior can be used to approximate a wide class of posteriors. However, as mentioned in Section \ref{sec:vb-s-unknown-1}, a normal $q$ would cause  $E_{Q}d_{KL}(l_0,l_{\bb{\omega}_n})$ to be undefined. One way out of this impasse reparametrizing $\sigma$ as $\sigma_\rho=\log(1+\exp(\rho))$ with a normal prior is used for $\rho$. In the following section, we show that this approach may work but comes with the disadvantage where the number of nodes, $k_n$ needs to grow at a rate smaller than $n^{1/2}$. The main disadvantage with this approach is if the number of nodes do not grow sufficiently, it may be difficult to find a neural network which well approximates the true function.

\noindent {\bf Sieve Theory:}
Let $\bb{\omega}_n=(\bb{\theta}_n,\rho)$ where $\bb{\theta}_n$ and $f_{\bb{\theta}_n}$ are same as defined in \eqref{e:theta-def}. With $\sigma_\rho=\log(1+e^{\rho})$, we have
\begin{equation}
\label{e:lik-0-r}
l_{\bb{\omega}_n}(y,\bb{x})=\frac{1}{\sqrt{2\pi \sigma_\rho^2}}\exp\Big(-\frac{1}{2\sigma_\rho^2} (y-f_{\bb{\theta}_n}(\bb{x}))^2\Big)
\end{equation}

\noindent The sieve is defined as follows.
\begin{align}
\label{e:G-def-r}
\mathcal{G}_n&=\Big\{l_{\bb{\omega}_n}(y,\bb{x}), \bb{\omega}_n\in \mathcal{F}_n\Big\}\hspace{10mm}\mathcal{F}_n=\Big\{ (\bb{\theta}_n,\sigma^2):|\theta_{in}|\leq C_n, |\rho|<\log C_n\Big\}
\end{align}
\noindent The definitions for likelihood, posterior and   Hellinger neighborhood  agree with those given in \eqref{e:Lw-def}, \eqref{e:pi-def} and  \eqref{e:hell-def} as in Section \ref{sec:vb-s-known}.\\

\noindent {\bf Prior distribution:} \noindent We propose a normal prior on each $\theta_{in}$ and $\rho$ as follows
\begin{equation}
\label{e:prior-t-r} p(\bb{\omega}_n)=\frac{1}{\sqrt{2\pi \eta^2}}e^{-\frac{\rho^2}{2\eta^2}}\prod_{i=1}^{K(n)}\frac{1}{\sqrt{2\pi \zeta^2}}e^{-\frac{\theta_{in}^2}{2\zeta^2}}
\end{equation}

\noindent {\bf Variational Family:} Variational family for $\bb{\omega}_n$ is given by
\begin{equation}
\label{e:var-family-r}
\mathcal{Q}_n=\left\{q:q(\bb{\omega}_n)=\frac{1}{\sqrt{2\pi \tilde{s}^2_{0n}}}e^{-\frac{(\rho-\tilde{m}_{0n})^2}{2\tilde{s}_{0n}^2}}\prod_{i=1}^{K(n)}\frac{1}{\sqrt{2\pi \tilde{s}^2_{in}}}e^{-\frac{(\theta_{in}-\tilde{m}_{in})^2}{2\tilde{s}_{in}^2}}\right\}
\end{equation}
The variational posterior has the same definition as in \eqref{e:VB-optimizer}.

In the following theorem we show that even with $\sigma$ reparametrized as $\log(1+e^{\rho})$ the variational posterior is consistent.

\begin{theorem}
	\label{thm:var-rho}
	Suppose the number of nodes satisfy condition (C1) with $a<1/2$. In addition, suppose assumptions (A1) and (A2) hold for $0\leq  \delta<1-a$. Then,
	$$\pi^*(\mathcal{V}_\varepsilon^c)=o_{P_0^n}(n^{-\delta})$$
\end{theorem}

\begin{proof}
	The proof mimics the steps in the proof of \ref{thm:var-cons} and \ref{thm:var-cons-1} 
	with Propositions \ref{lem:v-bound}, \ref{lem:kl-bound} and \ref{lem:q-bound} replaced by \ref{lem:v-bound-r}, \ref{lem:kl-bound-r} and \ref{lem:q-bound-r} respectively.
	\end{proof}

\begin{remark}
	With $k_n\sim n^a$ and $\mathcal{F}_n$ as in \eqref{e:G-def-r}, we choose $C_n=e^{n^{b-a}}$ where $0<a<b<1$. In order to ensure that prior gives smaller mass outside $\mathcal{F}_n$, one requires $\pi_n(\mathcal{F}_n^c)<e^{-ns}$ for some $s>0$. With a normal prior of $\rho$ and $P(|\rho|>\log C_n)\sim \frac{1}{\log C_n}e^{-(\log C_n)^2}$ which is less than $e^{-n}$ provided $2(b-a)>1$ or $a<1/2$. Hence, the requirement of a slow growth in the number of nodes.
\end{remark}

\section{Consistency of variational bayes}

In this section, we show that if the variational posterior is consistent, the variational Bayes estimator of $\sigma$ and $f_{\bb{\theta}_n}$ converges to the true $\sigma_0$ and $f_0$. The proof uses ideas from \cite{BSW} and Corollary 1 in \cite{LEE}. Let 
\begin{align}
    \label{e:vb-def}
    \nonumber \hat{f}_n(\bb{x})&=\int f_{\bb{\theta}_n}(\bb{x})\pi^*(\bb{\theta}_n)d\bb{\theta}_n\\
\hat{\sigma}^2_n&=\int \sigma^2 \pi^*(\sigma^2)d\sigma^2
\end{align}

\begin{corollary}[Variational bayes consistency.]
Suppose $\hat{f}_n$ and $\hat{\sigma}_n^2$ are defined as in \eqref{e:vb-def}, then 
\begin{align}
    \label{e:vb-conv}
    \nonumber \int (\hat{f}_n(\bb{x})-f_0(\bb{x}))^2 d\bb{x}&=o_{P_0^n}(1)\\
    \frac{\hat{\sigma}_n}{\sigma_0}&=1+o_{P_0^n}(1)
\end{align}

\end{corollary}

\begin{proof}
	Let 
	$$\hat{l}_n(y,\bb{x})=\int l_{\bb{\omega}_n}(y,\bb{x}) \pi^*(\bb{\omega}_n)d\bb{\omega}_n$$
	\begin{align*}
	d_{H}(\hat{l}_n(y,\bb{x})),l_0(y,\bb{x}))&=d_H\left( \int l({\bb{\omega}_n}) \pi^*(\bb{\omega}_n)d\bb{\omega}_n,l_0\right)\\
	&\leq \int d_H(l(\bb{\omega}_n),l_0) \pi^*(\bb{\omega}_n)d\bb{\omega}_n \hspace{5mm} \text{Jensen's inequality}\\
	&=\int_{\mathcal{V}_\varepsilon} d_H(l(\bb{\omega}_n),l_0) \pi^*(\bb{\omega}_n)d\bb{\omega}_n+\int_{\mathcal{V}_\varepsilon^c} d_H(l(\bb{\omega}_n),l_0) \pi^*(\bb{\omega}_n)d\bb{\omega}_n\\
	&\leq \varepsilon+o_{P_0^n}(1)
	\end{align*}
	Taking $\varepsilon \to 0$, we get $d_{H}(\hat{l}_n(y,\bb{x})),l_0(y,\bb{x}))=o_{P_0^n}(1)$.
	\noindent Now,
	$$\hat{l}_n(y,\bb{x})=\frac{1}{\sqrt{2\pi\hat{\sigma}^2_n}}e^{-\frac{1}{2\hat{\sigma}_n^2}(y-\hat{f}_n(\bb{x}))^2}$$
	\noindent Now, let us consider the form of 
	\begin{align*}
	d_H(\hat{l}_n,l_0)&=\int \int \left(\sqrt{\hat{l}_n(y,\bb{x})}-\sqrt{l_0(y,\bb{x})}\right)^2dy d\bb{x}\\
	&=2-2 \int \int  \sqrt{\hat{l}_n(y,\bb{x})l_0(y,\bb{x})}dy d\bb{x}\\
	&=2-2\int \int  \frac{1}{\sqrt{2\pi\hat{\sigma}_n\sigma_0}}\exp\left\{-\frac{1}{4}\left(\frac{(y-\hat{f}_n(\bb{x}))^2}{\hat{\sigma}_n^2}+\frac{(y-f_0(\bb{x}))^2}{\sigma_0^2}\right)\right\}dy d\bb{x}\\
	&=2-2\underbrace{\sqrt{\frac{2}{\hat{\sigma}_n/\sigma_0+\sigma_0/\hat{\sigma}_n}}}_{\textcircled{1}}\underbrace{\int e^{\left\{-\frac{1}{4(\hat{\sigma}^2_n+\sigma_0^2)}(\hat{f}_n(\bb{x})-f_0(\bb{x}))^2\right\}} d\bb{x}}_{\textcircled{2}}
	\end{align*}
	Since $d_{H}(\hat{l}_n,l_0)=o_{P_0^n}(1)$, $\textcircled{1}\times \textcircled{2}\stackrel{P_0^n}{\longrightarrow}1$.
	
	\noindent Note that $\textcircled{1} \leq 1$ and $\textcircled{2} \leq 1$, thus $\textcircled{1}, \textcircled{2} \stackrel{P_0^n}{\longrightarrow} 1$.
	
\noindent Since $x+1/x\geq 2$, thus
	$$\textcircled{1}\stackrel{P_0^n}{\longrightarrow} 1 \implies\hat{\sigma}_n \stackrel{P_0^n}{\longrightarrow} \sigma_0$$
	We shall next show 
	$$\textcircled{2} \stackrel{P_0^n}{\longrightarrow}1\implies  \int(\hat{f}_n(x)-f_0(x))^2 dx\stackrel{P_0^n}{\longrightarrow}0$$
We shall instead show that for any sequence $\{n\}$, there exists a further subsequence $\{n_k\}$ such that
$\int (\hat{f}_{n_k}-f_0(x))^2 d\bb{x} \stackrel{a.s.}{\longrightarrow}0$

\noindent Since $\textcircled{2}\stackrel{P_0^n}{\to}1$, there exists a sub-sequence $\{n_k\}$ s.t. 
$$\int e^{\left\{-\frac{1}{4(\hat{\sigma}^2_{n_k}+\sigma_0^2)}(\hat{f}_{n_k}(\bb{x})-f_0(\bb{x}))^2\right\}}d\bb{x}\stackrel{a.s.}{\longrightarrow}1$$
This implies
$$ \frac{1}{4(\hat{\sigma}^2_{n_k}+\sigma_0^2)}(\hat{f}_{n_k}(\bb{x})-f_0(\bb{x}))^2 \stackrel{a.s.}{\to} 0 \:\:a.e.\:\: \bb{x}$$
(for  details see proof of Corollary 1 in \cite{LEE}).

\noindent Thus, using Scheffe's theorem in \cite{scheffe1947}, we have
$$\int \frac{1}{4(\hat{\sigma}^2_{n_k}+\sigma_0^2)}(\hat{f}_{n_k}(\bb{x})-f_0(\bb{x}))^2d\bb{x} \stackrel{a.s.}{\to} 0 $$
which implies 
$$\int \frac{1}{4(\hat{\sigma}^2_{n}+\sigma_0^2)}(\hat{f}_{n}(\bb{x})-f_0(\bb{x}))^2d\bb{x} =o_{P_0^n}(1) $$
Since $\hat{\sigma}_n\stackrel{o_{P_0^n}}{\to}\sigma_0^2$, applying Slutsky, we get
$$\int (\hat{f}_{n}(\bb{x})-f_0(\bb{x}))^2d\bb{x}=o_{P_0^n}(1)$$ 
\end{proof}

\section{Discussion}
In this paper, we have highlighted the conditions which guarantee that the variational posterior of feed-forward neural networks is consistent. A variational family, as simple as a Gaussian mean-field, is good enough to ensure that the variational posterior is consistent provided the entropy of the true function $f_0$ is well behaved. In other words, $f_0$ has an approximating  neural network solution which approximates $f_0$ at a fast enough rate while ensuring that the number of nodes and the $L_2$ norm of the NN parameters grow in a controlled manner. Conditions of this form are often needed when one tries to establish the consistency of neural networks in a frequentist set up (see condition C3 in \cite{SHE2019A}).  Whereas variational posterior presents a scalable alternative to MCMC, unlike MCMC its consistency cannot be guaranteed without certain restriction on the entropy of the true function. Two other main contributions of the paper include that (1) Gaussian family may not always work as the best choice for a variational family (see Section \ref{sec:vb-s-unknown}) and (2) One may need a prior with variance growing in $n$ when the rate of growth in the $L_2$ norm of the approximating NN is high (see Theorem \ref{thm:var-cons}). 

Although, we have quantified consistency of the variational posterior, the rate of contraction of the variational posterior still needs to be explored. We suspect that this rate would be closely related to the rate of contraction of the true posterior with mild assumptions on the entropy of the function $f_0$. By following ideas of the proofs in this paper, one may be able to quantify conditions on the entropy of $f_0$ when one uses a deep neural network instead of one layer neural network in order to guarantee the consistency of variational posterior. Similarly, the effect of hierarchical priors and hyperparameters on the rate of convergence of the variational posterior need to be explored.

\section{Appendix}

\subsection{General Lemmas}
\begin{lemma}
    \label{lem:mod-kl}
    Let $p$ and $q$ be any two density functions.
 Then
$$E_{p}\left(\left|\log\frac{p}{q}\right|\right) \leq d_{KL}(p,q)+\frac{2}{e}$$  
\end{lemma}

\begin{proof}
Proof is same as proof of Lemma 4 in \cite{LEE}.
\end{proof}
\begin{lemma}
	\label{lem:theta-bound}
	Let $f_{\bb{\theta}_{0n}}(\bb{x})=\beta_{00}+\sum_{j=1}^{k_n}\beta_{j0}\psi(\gamma_{j0}^\top \bb{x})$ be a fixed neural network satisfying
$$|\bb{\theta}_{in}-\bb{\theta}_{i0n}|\leq \epsilon,\:\: i=1, \cdots, K(n).$$
Then,	$$\int (f_{\bb{\theta}_n}(\bb{x})-f_{\bb{\theta}_{0n}}(\bb{x}))^2 dx\leq 8\left (k_n^2+(p+1)^2(\sum_{j=1}^{k_n}|\theta_{i0n}|)^2\right)\epsilon^2$$
\end{lemma}
\begin{proof} This proof uses some ideas  from Lemma 6 in \cite{LEE}.
	Note that 
	$$f_{\bb{\theta}_n}(\bb{x})=\beta_0+\sum_{j=1}^{k_n}\beta_j\psi(\gamma_j^{\top}x) \hspace{5mm}f_{\bb{\theta}_{0n}}(\bb{x})=\beta_{00}+\sum_{j=1}^{k_n}\beta_{j0} \psi(\gamma_{j0}^{\top}x)$$ 
	Therefore, \vspace{-2mm}
	\begin{align*}
	|f_{\bb{\theta}_n}(\bb{x})-f_{\bb{\theta}_{0n}}(\bb{x})|&\leq |\beta_0-\beta_{00}|+ \sum_{j=1}^{k_n}|\beta_j \psi(\gamma_j^\top \bb{x})-\beta_{j0}\psi(\gamma_{j0}^\top \bb{x})|
	\end{align*}
	\noindent Let $u_j=-\gamma_{j0}^{\top}\bb{x}$, $r_j=(\gamma_{j0}-\gamma_{j})^{\top}\bb{x}$, then
	\begin{align*}
	&=|\beta_0-\beta_{00}|+\sum_{j=1}^{k_n}\Big| \frac{\beta_j}{1+e^{u_j+r_j}}- \frac{\beta_{j0}}{1+e^{u_j}}\Big|\\
	&=|\beta_0-\beta_{00}|+\sum_{j=1}^{k_n}\Big| \frac{\beta_j(1+e^{u_j})-\beta_{j0}(1+e^{u_j+r_j})}{(1+e^{u_j+r_j})(1+e^{u_j})}\Big|\\
	&=|\beta_0-\beta_{00}|+\sum_{j=1}^{k_n} \frac{|\beta_j-\beta_{j0}|+|\beta_j e^{u_j}-\beta_{j0}e^{u_j+r_j}|}{(1+e^{u_j+r_j})(1+e^{u_j})}\\
	&=|\beta_0-\beta_{00}|+2\sum_{j=1}^{k_n}|\beta_j-\beta_{j0}|+\sum_{j=1}^{k_n}|\beta_{j0}||1-e^{r_j}|
	\end{align*}
	\noindent Since, for $\delta$ small $|r_j|<(p+1)\delta<1$,  thus $|1-e^{r_j}|<2|r_j|$.
	\begin{align*}
	|f_{\bb{\theta}_n}(\bb{x})-f_{\bb{\theta}_{0n}}(\bb{x})|&\leq 2k_n\epsilon+2\epsilon(p+1)\sum_{j=1}^{k_n}|\beta_{j0}|\leq 2k_n\epsilon+2\epsilon(p+1)\sum_{j=1}^{k_n}|\theta_{i0n}|
	\end{align*}
	Using  \begin{equation}
	    \label{e:aplusb}
	    (a+b)^2\leq 2(a^2+b^2)
	\end{equation} 
	the proof follows.
\end{proof}

\begin{lemma}
	\label{lem:sig-bound}
	With $|\sigma/\sigma_0-1|<\delta$ 
	\begin{enumerate}
		\item $$h_1(\sigma)=\frac{1}{2}\log \frac{\sigma^2}{\sigma_0^2}-\frac{1}{2}\left(1-\frac{\sigma_0^2}{\sigma^2}\right)\leq \delta^2$$
		\item $$h_2(\sigma)=\frac{1}{2\sigma^2}\leq \frac{1}{2\sigma_0^2(1-\delta)^2}$$
	\end{enumerate}
	
\end{lemma}

\begin{proof}
Let $x=\sigma/\sigma_0$, then
	\begin{enumerate}
		\item
		$$h_1(x)=\frac{1}{2}\log x^2-\frac{1}{2}\left(1-\frac{1}{x^2}\right)$$
		where $|x-1|<\delta$.
		The function $h_1(x)$ satisfies
		$$h_1(x)\leq (x-1)h_1'(1)+\frac{(x-1)^2 }{2}h_1''(1)\leq \delta h_1'(1)+\frac{\delta^2}{2}h_1''(1)=\delta^2$$
		since $h_1'''(y)\leq 0$ for every $y \in (1-\delta,1+\delta)$.
		
		\item 
		$$h_2(x)=\frac{1}{2\sigma_0^2x^2}\leq \frac{1}{2\sigma_0^2(1-\delta)^2}$$
	\end{enumerate}	
	
\end{proof}

\begin{lemma}
	\label{lem:rho-bound}
	With $\sigma_\rho=\log(1+e^{\rho})$ and $|\rho-\rho_0|<\delta \sigma_0$, $\sigma_0=\log(1+e^{\rho_0})$. 
	\begin{enumerate}
		\item $$h_1(\rho)=\frac{1}{2}\log \frac{\sigma_\rho^2}{\sigma_0^2}-\frac{1}{2}\left(1-\frac{\sigma_0^2}{\sigma_\rho^2}\right)\leq \delta^2$$
		\item $$h_2(\rho)=\frac{1}{2\sigma_\rho^2}\leq \frac{1}{2\sigma_0^2(1-\delta)^2}$$
	\end{enumerate}
	
\end{lemma}

\begin{proof}
	$|\rho-\rho_0|<\delta\log (1+e^{\rho_0})$ implies $$\log(1+e^{\rho})-\log (1+e^{\rho_0})\leq  \delta\log(1+e^{\rho_0})$$
	Similarly,
	$$\log(1+e^{\rho})-\log (1+e^{\rho_0})\geq- \delta\log(1+e^{\rho_0})$$	
	Thus, $|\sigma_\rho/\sigma_0-1|<\delta$. 
The remaining part of the proof follows on the same lines as Lemma \ref{lem:sig-bound}.
\end{proof}

\begin{lemma}
	\label{lem:h-sig-bound}
	\noindent With $q(\sigma^2)=((n\sigma_0^2)^{n}/\Gamma(n))(1/\sigma^2)^{(n+1)}e^{-n\sigma_0^2
		/\sigma^2}$
	and 
	$h(\sigma^2)=(1/2)(\log (\sigma^2/\sigma_0^2)-(1-\sigma_0^2/\sigma^2))$, for every $0\leq \delta<1$, we have  
	\begin{align*}
	\int h(\sigma^2)q(\sigma^2)d\sigma^2=o(n^{-\delta})
	\end{align*}
	
\end{lemma}

\begin{proof}
	\begin{align*}
	\int h(\sigma^2)q(\sigma^2)d\sigma^2&=\int \frac{1}{2} \left(\log \frac{\sigma^2}{\sigma_0^2}-\left(1-\frac{\sigma_0^2}{\sigma^2}\right)\right) \frac{(n\sigma_0^2)^{n}}{\Gamma(n)}\left(\frac{1}{\sigma^2}\right)^{n+1} e^{-\frac{n\sigma_0^2}{\sigma^2}}d\sigma^2\\
	&=\int \frac{1}{2} \left(\log \frac{\sigma}{\sigma_0^2}-\left(1-\frac{\sigma_0^2}{\sigma}\right)\right) \frac{(n\sigma_0^2)^{n}}{\Gamma(n)}\left(\frac{1}{\sigma}\right)^{n+1} e^{-\frac{n\sigma_0^2}{\sigma}}d\sigma\\
	&=\frac{1}{2}\left(\log n\sigma_0^2-\log \psi(n)-\log \sigma_0^2\right)-\frac{1}{2}\left(1-\frac{\sigma_0^2}{\sigma_0^2}\right)\\
	&=\frac{1}{2}\left(\log n -\log n+O(n^{-1})\right)=o(n^{-\delta})
	\end{align*}	
	where the last step holds because $\psi(n)=\log n +O(n^{-1})$ (see Lemma 4  in \cite{ENG}).
\end{proof}

\begin{lemma}
	\label{lem:h-siginv-bound}
	\noindent With
$q(\sigma^2)=((n\sigma_0^2)^{n}/\Gamma(n))(1/\sigma^2)^{(n+1)}e^{-n\sigma_0^2/		\sigma^2}$
	and 	$h(\sigma^2)=1/(2\sigma^2)$,  for every $0\leq \delta<1$, \begin{align*}
	\int h(\sigma^2)q(\sigma^2)d\sigma^2= \frac{1}{2\sigma_0^2}
	\end{align*}
\end{lemma}

\begin{proof}
	\begin{align*}
	\int h(\sigma^2)q(\sigma^2)d\sigma^2&=\int \frac{1}{2\sigma^2} \frac{(n\sigma_0^2)^{n}}{\Gamma(n)}\left(\frac{1}{\sigma^2}\right)^{n+1} e^{-\frac{n\sigma_0^2}{\sigma^2}}d\sigma^2\\
	&=\int \frac{1}{2\sigma} \frac{(n\sigma_0^2)^{n}}{\Gamma(n)}\left(\frac{1}{\sigma}\right)^{n+1} e^{-\frac{n\sigma_0^2}{\sigma}}d\sigma\\
	&=\frac{n}{2n\sigma_0^2}=\frac{1}{\sigma_0^2}
	\end{align*}	
\end{proof}

\begin{lemma}
	\label{lem:h-rho-bound}
	\noindent With $\sigma_\rho=\log(1+e^{\rho})$ and $\sigma_0=\log(1+e^{\rho_0})$, let
	$h(\rho)=(1/2)\log (\sigma_\rho^2/\sigma_0^2)-(1/2)(1-\sigma_0^2/\sigma_\rho^2)$ and $q(\rho)=\sqrt{n/(2\pi \nu^2)}e^{-n(\rho-\rho_0)^2/2\nu^2}$. Then, for every $0\leq \delta<1$, we have
	\begin{align*}
	\int h(\rho)q(\rho)d\rho=o(n^{-\delta})
	\end{align*}
\end{lemma}

\begin{proof}
First note that $h(\rho)\geq 0$, thus it suffices to show $\int h(\rho)q(\rho)d\rho\leq o(n^{-\delta})$. In this direction,
$$\int h(\rho)q(\rho)d\rho=\underbrace{\int_{|\rho-\rho_0|<1/n^{\delta/2}} h(\rho)q(\rho)d\rho}_{\textcircled{1}}+\underbrace{\int_{|\rho-\rho_0|>1/n^{\delta/2}} h(\rho)q(\rho)d\rho}_{\textcircled{2}}$$		

\noindent We can apply Taylor expansion to $\textcircled{1}$ as
\begin{align*}
\textcircled{1}=\int_{|\rho-\rho_0|<1/n^{\delta/2}}& \left( h(\rho_0)+(\rho-\rho_0)h'(\rho_0)+\frac{(\rho-\rho_0)^2}{2}h''(\rho_0)+o((\rho-\rho_0)^2)\right)q(\rho)d\rho\\
=\int_{|\rho-\rho_0|<1/n^{\delta/2}}&\frac{(\rho-\rho_0)^2}{2}h''(\rho_0)q(\rho)d\rho+o(n^{-\delta})
\end{align*}
	where the equality follows since $h(\rho_0)=0$ and $q(\rho)$ is symmetric  around $\rho=\rho_0$.
	
	\noindent It is easy to check $h''(\rho_0)>0$, which implies
	$$\int_{|\rho-\rho_0|<1/n^{\delta/2}}\frac{(\rho-\rho_0)^2}{2}h''(\rho_0)q(\rho)d\rho\leq \int\frac{(\rho-\rho_0)^2}{2}h''(\rho_0)q(\rho)d\rho=\frac{h''(\rho_0)\nu^2}{2n}=O(n^{-1})$$
	Thus, 	for every $0\leq \delta <1$, $\textcircled{1}\leq O(n^{-1})+o(n^{-\delta})=o(n^{-\delta})$.

\noindent For the remaining part of the proof,  we shall make use of the Mill's ratio approximation as follows.
\begin{equation}
    \label{e:mills}
    1-\Phi(a_n)\sim \frac{\phi(a_n)}{a_n}
\end{equation}
where $\Phi$ and $\phi$ are the cdf and pdf of standard normal distribution respectively.

\noindent For $\textcircled{2}$,
		\begin{align*}
\textcircled{2}&=\int_{|\rho-\rho_0|>1/n^{\delta/2}}\left(\frac{1}{2}\log\frac{ \sigma_\rho^2}{\sigma_0^2}-\frac{1}{2}\left(1-\frac{\sigma_0^2}{\sigma_\rho^2}\right)\right)\sqrt{\frac{n}{2\pi \nu^2}}e^{-\frac{n}{2\nu^2}(\rho-\rho_0)^2}d\rho\\
	&\leq-\frac{1}{2}\log \sigma_0^2 \underbrace{\int_{|\rho-\rho_0|>1/n^{\delta/2}} \sqrt{\frac{n}{2\pi \nu^2}}e^{-\frac{n}{2\nu^2}(\rho-\rho_0)^2}d\rho}_{\textcircled{3}}+\frac{1}{2}\underbrace{\int_{|\rho-\rho_0|>1/n^{\delta/2}} \log \sigma_\rho^2\sqrt\frac{n}{2\pi \nu^2}e^{-\frac{n}{2\nu^2}(\rho-\rho_0)^2}d\rho}_{\textcircled{4}}\\
	&+\sigma_0^2\underbrace{\int_{|\rho-\rho_0|>1/n^{\delta/2}} \frac{1}{ \sigma_\rho^2}\sqrt{\frac{n}{2\pi \nu^2}}e^{-\frac{n}{2\nu^2}(\rho-\rho_0)^2}d\rho}_{\textcircled{5}}
	\end{align*}
Let $c=\log(e-1)$, then $c>0$.

\noindent If $\rho_0\geq c$, then $-\log \sigma_0^2\leq 0$ and $\textcircled{3}$ can be dropped. 
\noindent If $\rho_0<c\implies -\log \sigma_0^2>0$, then
	\begin{equation}
	\label{e:3-bound}
	\textcircled{3}=2\left(1-\Phi\left(\frac{\sqrt{n}}{\nu n^{\delta/2}}\right)\right)\sim O\left(\frac{1}{n^{\frac{1}{2}-\frac{\delta}{2}}}e^{-n^{1-\delta}}\right)=o(n^{-
		\delta})
	\end{equation}

\noindent For \textcircled{4}, we make use of the following result
\begin{equation}
\label{e:rho-4}
    \text{ If $\rho<c$, $\log \sigma_\rho<0$. For $\rho>c$, $\log \sigma_\rho\leq  \log(2e^{\rho})$.}
\end{equation}

\noindent If $\rho_0<c$, then $\rho_0-1/n^{\delta/2}, \rho_0+1/n^{\delta/2}<c$ for $n$ sufficiently large. 
	
\noindent Using \eqref{e:rho-4} and getting rid of negative terms, we get
	\begin{align*}
	\textcircled{4}&\leq \int_{c}^{\infty}\log \sigma_\rho^2 \sqrt{\frac{n}{2\pi \nu^2}}e^{-\frac{n}{2\nu^2}(\rho-\rho_0)^2}d\rho\leq\int_{c}^{\infty}2(\log 2+\rho) \sqrt{\frac{n}{2\pi \nu^2}}e^{-\frac{n}{2\nu^2}(\rho-\rho_0)^2}d\rho\\
	&=2\log 2\int_{c}^{\infty}\sqrt{\frac{n}{2\pi \nu^2}}e^{-\frac{n}{2\nu^2}(\rho-\rho_0)^2}d\rho+2 \int_{c}^{\infty}\rho\sqrt{\frac{n}{2\pi \nu^2}}e^{-\frac{n}{2\nu^2}(\rho-\rho_0)^2}d\rho  \\
	&=2\log 2\int_{\sqrt{n}(c-\rho_0)/\nu}^{\infty}\frac{1}{\sqrt{2\pi }}e^{-\frac{1}{2}u^2}d\rho+2 \int_{\sqrt{n}(c-\rho_0)/\nu}^{\infty}\left(\frac{u\nu}{\sqrt{n}}+\rho_0\right)\frac{1}{\sqrt{2\pi }}e^{-\frac{1}{2}u^2}d\rho \\
	&=(2\log 2+2\rho_0)\left(1-\Phi\left(\frac{\sqrt{n}(c-\rho_0)}{\nu}\right)\right)+\frac{2\nu}{\sqrt{n}} \int_{\sqrt{n}(c-\rho_0)/\nu}^{\infty}\frac{u}{\sqrt{2\pi }}e^{-\frac{1}{2}u^2}d\rho \\
	&=(2\log 2+4\rho_0)\Phi\left(-\frac{\sqrt{n}(c-\rho_0)}{\nu}\right)+\frac{4\nu}{\sqrt{2\pi n}}e^{-\frac{n(c-\rho_0)^2}{2\nu^2}}\\
	&=O\left(\frac{1}{\sqrt{n}}e^{-n}\right)+O\left(\frac{1}{\sqrt{n}}e^{-n}\right)=o(n^{-\delta})\hspace{10mm}\text{ follows from \eqref{e:mills}}
	\end{align*}

	\noindent If $\rho_0>c$, then $\rho_0-1/n^{\delta/2}, \rho_0+1/n^{\delta/2}>c$ for $n$ sufficiently large.
	
	\noindent Using \eqref{e:rho-4} and getting rid of negative terms, we get
	\begin{align*}
\textcircled{4}&\leq \int_{c}^{\rho_0-1/n^{\delta/2}}\log \sigma_\rho^2 \sqrt{\frac{n}{2\pi \nu^2}}e^{-\frac{n}{2\nu^2}(\rho-\rho_0)^2}d\rho+\int_{\rho_0+1/n^{\delta/2}}^{\infty}\log \sigma_\rho^2 \sqrt{\frac{n}{2\pi \nu^2}}e^{-\frac{n}{2\nu^2}(\rho-\rho_0)^2}d\rho\\
	&=2(\log 2+\rho)\left(\int_{c}^{\rho_0-1/n^{\delta/2}} \sqrt{\frac{n}{2\pi \nu^2}}e^{-\frac{n}{2\nu^2}(\rho-\rho_0)^2}d\rho+\int_{\rho_0+1/n^{\delta/2}}^{\infty} \sqrt{\frac{n}{2\pi \nu^2}}e^{-\frac{n}{2\nu^2}(\rho-\rho_0)^2}d\rho\right)\\
	&=(2\log 2+2\rho_0)\left\{\Phi\left(\frac{-\sqrt{n}}{n^{\delta/2}\nu}\right)-\Phi\left(\frac{\sqrt{n}(c-\rho_0)}{\nu}\right)+1-\Phi\left(\frac{\sqrt{n}}{n^{\delta/2}\nu}\right)\right\}\\
	&+\frac{2\nu}{\sqrt{2\pi n}}\left(e^{-\frac{n(c-\rho_0)^2}{2\nu^2}}-e^{-\frac{n^{1-\delta}}{2\nu^2}}\right)+\frac{2\nu}{\sqrt{2\pi n}}\left(e^{-\frac{n^{1-\delta}}{2\nu^2}}\right)\end{align*}
	\begin{align*}
	   \textcircled{4} &\leq (2\log 2+2\rho_0)\Phi\left(-\frac{\sqrt{n}}{n^{\delta/2}\nu}\right)+\frac{2\nu}{\sqrt{2\pi n}}\left(e^{-\frac{n(c-\rho_0)}{2\nu^2}}\right)\\
	&=O\left( \frac{1}{\sqrt{n^{1-\delta}}}e^{-n^{1-\delta}}\right)+O\left(\frac{1}{\sqrt{n}}e^{-n}\right)=o(n^{-\delta})\hspace{10mm}\text{ follows from \eqref{e:mills}}
	\end{align*}
	
	\noindent If $\rho_0=c$, then $\rho_0-1/n^{\delta/2}<c$ and  $\rho_0+1/n^{\delta/2}>c$ for $n$ sufficiently large, thus 
	\begin{align*}
	\textcircled{4}&\leq \int_{\rho_0+1/n^{\delta/2}}^{\infty}\log \sigma_\rho^2 \sqrt{\frac{n}{2\pi \nu^2}}e^{-\frac{n}{2\nu^2}(\rho-\rho_0)^2}d\rho=(2\log 2+2\rho_0)\left\{1-\Phi\left(\frac{\sqrt{n}}{n^{\delta/2}\nu}\right)\right\}+\frac{2\nu}{\sqrt{2\pi n}}\left(e^{-\frac{n^{1-\delta}}{2\nu^2}}\right)\\
	&= O\left(\frac{1}{\sqrt{n}}e^{-n^{1-\delta}}\right)+O\left(\frac{1}{\sqrt{n^{1-\delta}}}e^{-n^{1-\delta}}\right)=o(n^{-\delta})\hspace{10mm}\text{ follows from \eqref{e:mills}}
	\end{align*}
	
	\noindent For $\textcircled{5}$, we shall make use of the following result:
	\begin{align}
	\label{e:exp-simp}
\nonumber e^{-2\rho}\sqrt{\frac{n}{2\pi \nu^2}}e^{-\frac{n}{2\nu^2}(\rho-\rho_0)^2}&=e^{-\left(\rho_0-\frac{\nu^2}{n}\right)}\sqrt{\frac{n}{2\pi \nu^2}}e^{-\frac{n}{2\nu^2}\left(\rho-\left(\rho_0-\frac{\nu^2}{n}\right)\right)^2}\\
	\frac{1}{\sigma_\rho^2}\leq 3e^{-2\rho}, \rho<0 &\hspace{10mm} 	\frac{1}{\sigma_\rho^2}\leq\frac{1}{(\log 2)^2}, \rho>0
	\end{align}	

	\noindent If $\rho<0$, then $\rho_0-1/n^{\delta/2}$, $\rho_0+1/n^{\delta/2}$ $<0$ for $n$ sufficiently large. Thus, using \eqref{e:exp-simp}, we get		
	\begin{align*}
	\textcircled{5}&=\int_{-\infty}^{\rho_0-1/n^{\delta/2}}\frac{1}{ \sigma_\rho^2}\sqrt{\frac{n}{2\pi \nu^2}}e^{-\frac{n}{2\nu^2}(\rho-\rho_0)^2}d\rho+\int_{\rho_0+1/n^{\delta/2}}^{0}\frac{1}{ \sigma_\rho^2}\sqrt{\frac{n}{2\pi \nu^2}}e^{-\frac{n}{2\nu^2}(\rho-\rho_0)^2}d\rho\\
	&+\int_{0}^{\infty}\frac{1}{ \sigma_\rho^2}\sqrt{\frac{n}{2\pi \nu^2}}e^{-\frac{n}{2\nu^2}(\rho-\rho_0)^2}d\rho\\
	&\leq 3\int_{-\infty}^{\rho_0-1/n^{\delta/2}}e^{-2\rho}\sqrt{\frac{n}{2\pi \nu^2}}e^{-\frac{n}{2\nu^2}(\rho-\rho_0)^2}d\rho+3\int_{\rho_0+1/n^{\delta/2}}^{0}e^{-2\rho}\sqrt{\frac{n}{2\pi \nu^2}}e^{-\frac{n}{2\nu^2}(\rho-\rho_0)^2}d\rho\\
	&+\frac{1}{(\log 2)^2}\int_{0}^{\infty}\sqrt{\frac{n}{2\pi \nu^2}}e^{-\frac{n}{2\nu^2}(\rho-\rho_0)^2}d\rho\\
	&\leq 3\int_{|\rho-\rho_0|>1/n^{\delta/2}}e^{-2\rho}\sqrt{\frac{n}{2\pi \nu^2}}e^{-\frac{n}{2\nu^2}(\rho-\rho_0)^2}d\rho+\frac{1}{(\log 2)^2}\int_{0}^{\infty}\sqrt{\frac{n}{2\pi \nu^2}}e^{-\frac{n}{2\nu^2}(\rho-\rho_0)^2}d\rho\\
	&\leq 3e^{-\left(\rho_0-\frac{\nu^2}{n}\right)}\int_{|\rho-\rho_0|>1/n^{\delta/2}}\sqrt{\frac{n}{2\pi \nu^2}}e^{-\frac{n}{2\nu^2}\left(\rho-\left(\rho_0-\frac{\nu^2}{n}\right)\right)^2}+\frac{1}{(\log 2)^2}\int_{0}^{\infty}\sqrt{\frac{n}{2\pi \nu^2}}e^{-\frac{n}{2\nu^2}(\rho-\rho_0)^2}d\rho\\
	&= 6 e^{-\left(\rho_0-\frac{\nu^2}{n}\right)}\Phi\left(-\frac{\sqrt{n}}{\nu}\left(\frac{1}{n^{\delta/2}}-\frac{\nu^2}{n}\right)\right)+\frac{1}{(\log 2)^2}\Phi(-\sqrt{n}(-\rho_0))\\
	&=O\left( \frac{1}{\sqrt{n^{1-\delta}}}e^{-n^{1-\delta}}\right)+O\left(\frac{1}{\sqrt{n}}e^{-n}\right)=o(n^{-\delta})\hspace{10mm}\text{ follows from \eqref{e:mills}}
	\end{align*}
	\noindent If $\rho>0$, then $\rho_0-1/n^{\delta/2}$, $\rho_0+1/n^{\delta/2}$ $>0$ for $n$ sufficiently large. Thus, using \eqref{e:exp-simp}, we get
	\begin{align*}
	\textcircled{5}&=\int_{-\infty}^{0}\frac{1}{ \sigma_\rho^2}\sqrt{\frac{n}{2\pi \nu^2}}e^{-\frac{n}{2\nu^2}(\rho-\rho_0)^2}d\rho+\int_0^{\rho_0-1/n^{\delta/2}}\frac{1}{ \sigma_\rho^2}\sqrt{\frac{n}{2\pi \nu^2}}e^{-\frac{n}{2\nu^2}(\rho-\rho_0)^2}d\rho\\
	&+\int_{\rho_0+1/n^{\delta/2}}^{\infty}\frac{1}{ \sigma_\rho^2}\sqrt{\frac{n}{2\pi \nu^2}}e^{-\frac{n}{2\nu^2}(\rho-\rho_0)^2}d\rho\\
	&\leq \int_{-\infty}^{0}3e^{-2\rho}\sqrt{\frac{n}{2\pi \nu^2}}e^{-\frac{n}{2\nu^2}(\rho-\rho_0)^2}d\rho+\frac{1}{(\log 2)^2}\int_{0}^{\rho_0-1/n^{\delta/2}}\sqrt{\frac{n}{2\pi \nu^2}}e^{-\frac{n}{2\nu^2}(\rho-\rho_0)^2}d\rho\\
	&+\frac{1}{(\log 2)^2}\int_{\rho_0+1/n^{\delta/2}}^{\infty}\sqrt{\frac{n}{2\pi \nu^2}}e^{-\frac{n}{2\nu^2}(\rho-\rho_0)^2}d\rho\\
	&\leq 3\int_{-\infty}^{0}e^{-2\rho}\sqrt{\frac{n}{2\pi \nu^2}}e^{-\frac{n}{2\nu^2}(\rho-\rho_0)^2}d\rho+\frac{1}{(\log 2)^2}\int_{|\rho-\rho_0|>1/n^{\delta/2}}\sqrt{\frac{n}{2\pi \nu^2}}e^{-\frac{n}{2\nu^2}(\rho-\rho_0)^2}d\rho
	\end{align*}
	\begin{align*}
\textcircled{5}	&\leq 3e^{-\left(\rho_0-\frac{\nu^2}{n}\right)}\int_{-\infty}^{0}\sqrt{\frac{n}{2\pi \nu^2}}e^{-\frac{n}{2\nu^2}\left(\rho-\left(\rho_0-\frac{\nu^2}{n}\right)\right)^2}+\frac{1}{(\log 2)^2}\int_{|\rho-\rho_0|>1/n^{\delta/2}}\sqrt{\frac{n}{2\pi \nu^2}}e^{-\frac{n}{2\nu^2}(\rho-\rho_0)^2}d\rho\\
	&\leq 3e^{-\left(\rho_0-\frac{\nu^2}{n}\right)}\int_{-\infty}^{0}\sqrt{\frac{n}{2\pi \nu^2}}e^{-\frac{n}{2\nu^2}\left(\rho-\left(\rho_0-\frac{\nu^2}{n}\right)\right)^2}+\frac{1}{(\log 2)^2}\int_{|\rho-\rho_0|>1/n^{\delta/2}}\sqrt{\frac{n}{2\pi \nu^2}}e^{-\frac{n}{2\nu^2}(\rho-\rho_0)^2}d\rho\\
	&= 3e^{-\left(\rho_0-\frac{\nu^2}{n}\right)} \Phi\left(\frac{-\sqrt{n}}{\nu}\left(\rho_0-\frac{\nu^2}{n}\right)\right)+\frac{2}{(\log 2)^2}\Phi\left(-\frac{\sqrt{n}\rho_0}{\nu n^{\delta/2}}\right)\\
	&=O\left( \frac{1}{\sqrt{n}}e^{-n}\right)+O\left(\frac{1}{\sqrt{n^{1-\delta}}}e^{-n^{1-\delta}}\right)=o(n^{-\delta})\hspace{10mm}\text{ follows from \eqref{e:mills}}
	\end{align*}
	If $\rho_0=0$, then $\rho_0-1/n^{\delta/2}<0$, $\rho_0+1/n^{\delta/2}>0$ for $n$ sufficiently large. Thus, using \eqref{e:exp-simp}, we get 
	\begin{align*}
	\textcircled{5}&=\int_{-\infty}^{\rho_0-1/n^{\delta/2}}\frac{1}{ \sigma_\rho^2}\sqrt{\frac{n}{2\pi \nu^2}}e^{-\frac{n}{2\nu^2}(\rho-\rho_0)^2}d\rho+\int_{\rho_0+1/n^{\delta/2}}^{\infty}\frac{1}{ \sigma_\rho^2}\sqrt{\frac{n}{2\pi \nu^2}}e^{-\frac{n}{2\nu^2}(\rho-\rho_0)^2}d\rho\\
	&\leq 3\int_{-\infty}^{\rho_0-1/n^{\delta/2}}e^{-2\rho}\sqrt{\frac{n}{2\pi \nu^2}}e^{-\frac{n}{2\nu^2}(\rho-\rho_0)^2}d\rho+\frac{1}{(\log 2)^2}\int_{\rho_0+1/n^{\delta/2}}^{\infty}\sqrt{\frac{n}{2\pi \nu^2}}e^{-\frac{n}{2\nu^2}(\rho-\rho_0)^2}d\rho\\
&\leq 3e^{-\left(\rho_0-\frac{\nu^2}{n}\right)}\int_{-\infty}^{\rho_0-1/n^{\delta/2}}\sqrt{\frac{n}{2\pi \nu^2}}e^{-\frac{n}{2\nu^2}\left(\rho-\left(\rho_0-\frac{\nu^2}{n}\right)\right)^2}+\frac{1}{(\log 2)^2}\int_{\rho_0+1/n^{\delta/2}}^{\infty}\sqrt{\frac{n}{2\pi \nu^2}}e^{-\frac{n}{2\nu^2}(\rho-\rho_0)^2}d\rho\\
&=3e^{-\left(\rho_0-\frac{\nu^2}{n}\right)} \Phi\left(\frac{-\sqrt{n}}{\nu}\left(\frac{1}{n^{\delta/2}}-\frac{\nu^2}{n}\right)\right)+\frac{1}{(\log 2)^2}\Phi\left(-\frac{\sqrt{n}\rho_0}{\nu n^{\delta/2}}\right)\\
&=O\left(\frac{1}{\sqrt{n^{1-\delta}}}e^{-n^{1-\delta}}\right)+O\left(\frac{1}{\sqrt{n}}e^{-n}\right)=o(n^{-\delta})\hspace{10mm}\text{ follows from \eqref{e:mills}}
	\end{align*}

\end{proof}

\begin{lemma}
	\label{lem:h-rhoinv-bound}
	\noindent 	\noindent With $\sigma_\rho=\log(1+e^{\rho})$ and $\sigma_0=\log(1+e^{\rho_0})$, let
	$h(\rho)=1/(2\sigma_\rho^2)$ and $q(\rho)=\sqrt{n/(2\pi \nu^2)}e^{-n(\rho-\rho_0)^2/2\nu^2}$.  Then, for every $0\leq \delta<1$, we have
	\begin{align*}
	\int h(\rho)q(\rho)d\rho= \frac{1}{2\sigma_0^2}+o(n^{-\delta})
	\end{align*}
\end{lemma}
\begin{proof}
	$$\int h(\rho)q(\rho)d\rho=\underbrace{\int_{|\rho-\rho_0|<1/n^{\delta/2}} h(\rho)q(\rho)d\rho}_{\textcircled{1}}+\underbrace{\int_{|\rho-\rho_0|>1/n^{\delta/2}} h(\rho)q(\rho)d\rho}_{\textcircled{2}}$$		
	We can apply Taylor expansion to $\textcircled{1}$,
	\begin{align*}\textcircled{1}&=\int_{|\rho-\rho_0|<1/n^{\delta/2}} \left( h(\rho_0)+(\rho-\rho_0)h'(\rho_0)+\frac{(\rho-\rho_0)^2}{2}h''(\rho_0)+o((\rho-\rho_0)^2)\right)q(\rho)d\rho\\
	&=\frac{1}{2\sigma_0^2}+\int_{|\rho-\rho_0|<1/n^{\delta/2}}\frac{(\rho-\rho_0)^2}{2}h''(\rho_0)q(\rho)d\rho+o(n^{-\delta})
	\end{align*}
	where the equality follows since $h(\rho_0)=1/(2\sigma_0^2)$ and $q(\rho)$ is symmetric around $\rho_0$.
	
	\noindent Since $(\rho-\rho_0)^2$ and  $h''(\rho_0)>0$, it suffices to show $\int_{|\rho-\rho_0|<1/n^{\delta/2}}\frac{(\rho-\rho_0)^2}{2}h''(\rho_0)q(\rho)d\rho\leq o(n^{-\delta})$. 
	
	\noindent In this direction, 
	$$\int_{|\rho-\rho_0|<1/n^{\delta/2}}\frac{(\rho-\rho_0)^2}{2}h''(\rho_0)q(\rho)d\rho\leq \int\frac{(\rho-\rho_0)^2}{2}h''(\rho_0)q(\rho)d\rho=\frac{h''(\rho_0)\nu^2}{2n}=O(n^{-1})=o(n^{-\delta})$$
	
	\noindent Since $h(\rho)>0$, to prove $\textcircled{2}=o(n^{-\delta})$ it suffices to show $\textcircled{2}\leq o(n^{-\delta})$.
	
	\noindent Note,	$\textcircled{2}$ is same as $\textcircled{5}$ of Lemma \ref{lem:h-rho-bound}, except for  a constant. 	Thus, $\textcircled{2}\leq  o(n^{-\delta})$
	which completes the proof.
	
\end{proof}

\begin{lemma}
	\label{lem:f-bound}
	Suppose condition (C1) and assumption (A1) holds for some $0<a<1$ and $0\leq \delta<1-a$.
	Let
	$$h(\bb{\theta}_n)=\int (f_{\bb{\theta}_n}(\bb{x})-f_0(\bb{x}))^2 d\bb{x}$$
	we have
	\begin{equation}
	\label{e:f-bound}
	\int h(\bb{\theta}_n)q(\bb{\theta}_n)d\bb{\theta}_n= o(n^{-\delta})
	\end{equation}
	provided
	\begin{enumerate}
		\item Assumption (A2) holds with same $\delta$ as (A1) and $$q(\bb{\theta}_n)=\prod_{i=1}^{K(n)}\sqrt{\frac{n}{2\pi \tau^2}}e^{-\frac{n}{2\tau^2}(\theta_{in}-\theta_{0in})^2}$$
		\item Assumption (A3) holds and
		$$q(\bb{\theta}_n)=\prod_{i=1}^{K(n)}\sqrt{\frac{n^{v+1}}{2\pi \tau^2}}e^{-\frac{n^{v+1}}{2\tau^2}(\theta_{in}-\theta_{0in})^2}$$
	\end{enumerate}	
\end{lemma}

\begin{proof}
Note that since $h(\bb{\theta}_n)>0$, to prove \eqref{e:f-bound}, it suffices to show $\int h(\bb{\theta}_n)q(\bb{\theta}_n)d\bb{\theta}_n =o(n^{-\delta})$.

\noindent We begin by proving statement 1. of the lemma. 
	\noindent Let $A=\{\bb{\theta}_n:\cap_{i=1}^{K(n)}|\theta_{in}-\theta_{0in}|\leq 1/n^{\delta/2}\}$, then
	$$\int h(\bb{\theta}_n)q(\bb{\theta}_n)d\bb{\theta}_n=\underbrace{\int_{A}h(\bb{\theta}_n)q(\bb{\theta}_n)d\bb{\theta}_n}_{\textcircled{1}}+\underbrace{\int_{A^c}h(\bb{\theta}_n)q(\bb{\theta}_n)d\bb{\theta}_n}_{\textcircled{2}}$$
	
	\noindent For $\textcircled{1}$,	we do a Taylor expansion of $h(\bb{\theta}_n)$ around $\bb{\theta}_{0n}$ as 
	\begin{align*}
\textcircled{1}	&=\int_A\left(h(\bb{\theta}_{0n})+(\bb{\theta}_n-\bb{\theta}_{0n})^{\top}\nabla h(\bb{\theta}_{0n})+\frac{1}{2}(\bb{\theta}_n-\bb{\theta}_{0n})^{\top}\nabla^2 h(\bb{\theta}_{0n})(\bb{\theta}_n-\bb{\theta}_{0n})\right)q(\bb{\theta}_n)d\bb{\theta}_n\\
&+\int_A o(||\bb{\theta}_n-\bb{\theta}_{0n})||^2)q(\bb{\theta}_n)d\bb{\theta}_n\\
	&=\underbrace{\int_A(\bb{\theta}_n-\bb{\theta}_{0n})^{\top}\nabla h(\bb{\theta}_{0n}) q(\bb{\theta}_n)d\bb{\theta}_n}_{\textcircled{3}}+\frac{1}{2}\underbrace{\int_A (\bb{\theta}_n-\bb{\theta}_{0n})^{\top}\nabla^2 h(\bb{\theta}_{0n})(\bb{\theta}_n-\bb{\theta}_{0n})q(\bb{\theta}_n)d\bb{\theta}_n}_{\textcircled{4}}+o(n^{-\delta}) 
	\end{align*}
	where the last equality follows since $h(\bb{\theta}_{0n})=o(n^{-\delta})$ by assumption (A1).
	
	\noindent With $I=\{1,\cdots, K(n)\}$, let $\nabla h(\theta_{0n})=(a_i)_{i \in I}$ and  $\nabla^2 h(\theta_{0n})=((b_{ij}))_{i\in I, j\in I}$
	\begin{align}
	\label{e:nh-bound}
	 \nonumber
\textcircled{3}&=\sum_{i=1}^{K(n)} a_i \int_{|\theta_{in}-\theta_{i0n}|<1/n^{\delta/2}}(\theta_{in}-\theta_{i0n})q(\theta_{in})d\theta_{in}\\
&=\sum_{i=1}^{K(n)}a_i\int_{\theta_{i0n}-1/n^{\delta/2}}^{\theta_{i0n}+1/n^{\delta/2}}(\theta_{in}-\theta_{i0n})\sqrt{\frac{n}{2\pi \tau^2}}e^{-\frac{n}{2\tau^2}(\theta_{in}-\theta_{i0n})^2}=\sum_{i=1}^{K(n)} a_i \int_{-\sqrt{n^{1-\delta}}/\tau}^{\sqrt{n^{1-\delta}}/\tau}\frac{u}{\sqrt{2\pi}}e^{-\frac{1}{2}u^2}du=0 \vspace{-5mm}
	\end{align}
	since $ue^{-1/2 u^2}$ is an odd function.
	Also, 
	\begin{align*}
	\textcircled{4}&=\sum_{i=1}^{K(n)} b_{ii}\int_{|{\theta}_{in}-{\theta}_{i0n}|\leq 1/n^{\delta/2}}(\theta_{in}-\theta_{i0n})^2 q(\theta_{in})d\theta_{in}\\
	&+\sum_{i=1}^{K(n)} \sum_{j=1,i \neq j}^{K(n)} \int_{|{\theta}_{in}-{\theta}_{i0n}|\leq 1/n^{\delta/2}} (\theta_{in}-\theta_{i0n})q(\theta_{in})d\theta_{in} \int_{|{\theta}_{jn}-{\theta}_{j0n}|\leq 1/n^{\delta/2}}(\theta_{jn}-\theta_{j0n}) q(\theta_{jn})d\theta_{jn}\\
	&=\sum_{i=1}^{K(n)} b_{ii}\int_{|{\theta}_{in}-{\theta}_{i0n}|\leq 1/n^{\delta/2}}(\theta_{in}-\theta_{i0n})^2 q(\theta_{in})d\theta_{in}
	\end{align*}
	where second equality to third equality is a consequence of 	\eqref{e:nh-bound}. Thus, 
	\begin{align*}
	\textcircled{4}&\leq \sum_{i=1}^{K(n)}|b_{ii}|\int(\theta_{in}-\theta_{i0n})^2 q(\theta_{in})d\theta_{in}=\frac{\tau^2}{n}\sum_{i=1}^{K(n)}|b_{ii}| 
	\end{align*}
	We next try to bound the quantities $|b_{ii}|$. First note that $$\nabla^2 h(\bb{\theta}_n)=2\int \nabla f_{\bb{\theta_{0n}}}(\bb{x}) \nabla {f_{\bb{\theta_{0n}}}(\bb{x})}^{\top} d\bb{x}+2 \int (f_{\bb{\theta}_{0n}}(\bb{x})-f_0(\bb{x}))\nabla^2 {f_{\bb{\theta_{0n}}}(\bb{x})} d\bb{x}$$
	Let $\theta_{0n}=[\beta_0,\beta_1,\cdots, \beta_{k_n},\gamma_{11}, \cdots, \gamma_{1p}, \gamma_{21}, \cdots, \gamma_{2p}, \cdots, \gamma_{K(n)1}, \cdots, \gamma_{K(n)p} ]^\top$. Then,
	$$\bb{b}=[2, \bb{c}_0, \bb{c}_1, \cdots, \bb{c}_{K(n)}]^{\top}$$
	where for $ i=1, \cdots, k_n$, $j=1, \cdots, p$, we have
	\begin{align*}
	\bb{c}_{0i}&=2\int (\psi(\bb{\gamma}_{i0}^\top \bb{x}))^2 d\bb{x} \\
	\bb{c}_{ij}&=2\beta_{i0}^2\int (\psi'(\bb{\gamma}_{i0}^\top \bb{x}))^2 d\bb{x}+2\beta_{i0}^2\int (f_{\bb{\theta}_{0n}}(\bb{x})-f_0(\bb{x}))(\psi''(\bb{\gamma}_{i0}^\top \bb{x}))^2d \bb{x},\:\: j=0\\
	&=2\beta_{i0}^2\int (\psi'(\bb{\gamma}_{i0}^\top \bb{x}))^2 x_{ij}^2 d\bb{x}+2\beta_{i0}^2\int (f_{\bb{\theta}_{0n}}(\bb{x})-f_0(\bb{x}))(\psi''(\bb{\gamma}_{i0}^\top \bb{x}))^2x_{ij}^2d \bb{x}, j>0
	\end{align*}
	Using the fact that $|\psi(u)|,|\psi'(u)|,|\psi''(u)| \leq 1$ and $|x_{ij}|\leq 1$ we get
	\begin{align*}
	\textcircled{4}&\leq \frac{\tau^2}{n}\left(2(k_n+1)+2(p+1)\sum_{i=1}^{k_n}\beta^2_{j0}+(p+1)\int |f_{\theta_{0n}}-f_0(\bb{x})|d\bb{x}\sum_{i=1}^{k_n}\beta^2_{j0}\right)\\
	&\leq \frac{\tau^2}{n}(2(K(n)+1)+2(p+1)\sum_{i=1}^{K(n)}\theta_{i0n}^2+(p+1)\sum_{i=1}^{K(n)}\theta_{i0n}^2 ||f_{\bb{\theta}_{0n}}-f_0||^2_2= o(n^{-\delta}) 
	\end{align*}
	where the last equality is a consequence of assumptions (A1), (A2) and  condition (C1).
	
	\noindent For $\textcircled{2}$, note that
	\begin{align*}
	\int_{A^c} h(\bb{\theta}_n)d\bb{\theta}_n&=2\int_{A^c} \int f^2_{\bb{\theta}_n}(\bb{x})d\bb{x}q(\bb{\theta}_n) d\bb{\theta}_n+2\int_{A^c} \int f_0^2(\bb{x})d\bb{x}q(\bb{\theta}_n) d\bb{\theta}_n
	\end{align*}
	First, note that $|f_{\bb{\theta}_n}(\bb{x})|\leq \sum_{j=0}^{k_n}|\beta_j|\leq \sum_{j=0}^{k_n} |\beta_{j0}|+\sum_{j=0}^{k_n}|\beta_{j}-\beta_{j0}|^2$ since $|\psi(u)|\leq 1$. Thus, 
	\begin{align*}
\int_{A^c} h(\bb{\theta}_n)d\bb{\theta}_n
	&\leq 4\underbrace{\int_{A^c}(\sum_{j=1}^{k_n} |\beta_{j0}|)^2q(\bb{\theta}_n)d\bb{\theta}_n}_{\textcircled{5}}+4\underbrace{\int_{A^c} (\sum_{j=1}^{k_n} |\beta_j-\beta_{j0}|)^2q(\bb{\theta}_n)d\bb{\theta}_n}_{\textcircled{6}}+2\underbrace{\int f_0^2(\bb{x})d\bb{x} \int_{A^c} q(\bb{\theta}_n)d\bb{\theta}_n}_{\textcircled{7}}
	\end{align*}
	where the last equality follows since using $(a+b)^2\leq 2(a^2+b^2)$.
	
	\noindent First note that $A^c=\cup_{i=1}^{K(n)}A_i^c$ where  $A_i=\{|\theta_{in}-\theta_{i0n}|\leq 1/n^{\delta/2}\}$. Therefore,
	\begin{align}
	    \label{e:a-form}
	    \nonumber Q(A^c)&=Q(\cup_{i=1}^{K(n)}A_i^c)\leq\sum_{i=1}^{K(n)}Q(A_i^c)\\
	    &=\sum_{i=1}^{K(n)}\int_{|\bb{\theta}_{in}-\bb{\theta}_{i0n}|>1/n^{\delta/2}}q(\theta_{in})d\theta_{in}=2K(n)\left(1-\Phi\left(\frac{\sqrt{n}}{\tau n^{\delta/2}}\right)\right) = O\left(\frac{n^a e^{-n^{1-\delta}}}{\sqrt{n^{1-\delta}}}\right)
	\end{align}
	where the last asymptotic equality is a consequence of \eqref{e:mills} and condition (C1).
	
\noindent For $\textcircled{7}$, note that $\int f_0^2(\bb{x})d\bb{x}\leq M$ for some $M>0$. Therefore,
	$$\textcircled{7}=O\left(\frac{n ^a}{\sqrt{n^{1-\delta}}}e^{-n^{1-\delta}}\right)= o(n^{-\delta})$$
	for any $0\leq \delta<1$.
	
	\noindent For $\textcircled{5}$, note that $\sum_{j=1}^{k_n} \theta_{i0n}^2=o(n^{1-\delta})$ by assumption (A2). Using this together with \eqref{e:a-form}, we get  
	$$\textcircled{5}=(\sum_{j=1}^{k_n}|\beta_{j0}|)^2 Q(A^c)\leq k_n \sum_{i=1}^{k_n} \beta_{j0}^2 Q(A^c)\leq K(n)\sum_{j=1}^{K(n)}\theta_{i0n}^2  Q(A^c)\leq  o(n^{1-\delta})O\left(\frac{n ^{2a}}{\sqrt{n^{1-\delta}}}e^{-n^{1-\delta}}\right)=o(n^{-\delta})$$

	\noindent For $\textcircled{6}$, using Cauchy Schwartz, we get
	\begin{align*}
	\int_{A^c} (\sum_{j=1}^{k_n} |\beta_j-\beta_{j0}|)^2 q(\bb{\theta}_n)d\bb{\theta}_n&\leq k_n\sum_{j=1}^{k_n}\int_{A^c} (\beta_j-\beta_{j0})^2q(\bb{\theta}_n)d\bb{\theta}_n=O(k_n^2 e^{-n^{1-\delta}})=O(n^{2a}e^{-n^{1-\delta}}=o(n^{-\delta})
	\end{align*}
where the fact $\int_{A^c} (\beta_j-\beta_{j0})^2q(\bb{\theta}_n)d\bb{\theta}_n\sim e^{-n^{1-\delta}}$ is shown below.
\noindent Now, let $A_{\beta_j}=\{|\beta_{j}-\beta_{j0}|>1/n^{\delta/2}\}$
\begin{align}
\label{e:beta-def}
    \int_{A^c} \nonumber (\beta_j-\beta_{j0})^2q(\bb{\theta}_n)d\bb{\theta}_n&=\int_{A^c\cap A_{\beta_j}} (\beta_j-\beta_{j0})^2q(\bb{\theta}_n)d\bb{\theta}_n+\int_{A^c \cap A_{\beta_j}^c}(\beta_j-\beta_{j0})^2q(\bb{\theta}_n)d\bb{\theta}_n\\
    &\leq \int_{A_{\beta_j}} (\beta_j-\beta_{j0})^2q(\beta_j)d\beta_j +\frac{\tau^2}{n}\int_{\tilde{A}^c} q(\tilde{\bb{\theta}}_n)d\tilde{\bb{\theta}}_n
\end{align}
where $\tilde{\bb{\theta}}_n$ includes all coordinates of  $\bb{\theta}_n$ except $\beta_j$ and $\tilde{A}^c$ is the union of all $A_i^c$ except $A_{\beta_j}^c$.
\begin{align}
\label{e:beta-def-1}
    \int_{A_{\beta_j}} (\beta_j-\beta_{j0})^2q(\beta_j)d\beta_j
\nonumber &=\int_{|\beta_{j}-\beta_{j0}|>1/n^{\delta/2}}\sqrt{\frac{n}{2\pi\tau^2}}(\beta_j-\beta_{j0})^2e^{-\frac{n}{2\tau^2}(\beta_j-\beta_{j0}^2)}\\
\nonumber&=2\int_{\sqrt{n^{1-\delta}}\tau}^{\infty} \frac{u^2}{\sqrt{2\pi}}e^{-\frac{1}{2}u^2}\lesssim \sqrt{\frac{2}{\pi}}\int_{\sqrt{n^{1-\delta}}\tau}^{\infty} e^{-u} du \hspace{5mm}x^2e^{-x^2/2} \leq e^{-x}, x \to \infty \\\
&=O(e^{-n^{1-\delta}})
\end{align}
Using \eqref{e:a-form}, we get 
\begin{equation}
\label{e:beta-def-2}
\nonumber    \int_{\tilde{A}^c} q(\tilde{\bb{\theta}}_n)d\tilde{\bb{\theta}}_n =O\left( \frac{e^{-n^{1-\delta}}}{\sqrt{n^{1-\delta}}}\right)
\end{equation}
Using \eqref{e:beta-def-1} and \eqref{e:beta-def-2} in \eqref{e:beta-def}, we get
\begin{align}
\label{e:beta-def-3}
    \int_{A^c} (\beta_j-\beta_{j0})^2q(\beta_j)d\beta_j=O(e^{-n^{1-\delta}})+O\left(\frac{n^a}{n} \frac{e^{-n^{1-\delta}}}{\sqrt{n^{1-\delta}}}\right) =O( e^{-n^{1-\delta}})
    \end{align}
    
\noindent The only difference with statement 2. is that $\sum_{i=1}^{K(n)}\theta_{i0n}^2=O(n^v)$ and $\tau^2=\tau^2/n^{v+1}$. 

\noindent The proof is similar and details have been omitted.

\end{proof}

\begin{lemma}	
	\label{lem:b-bound-0}
	Suppose $N_{\varepsilon}=\{\bb{\omega}_n: d_{KL}(l_0,l_{\bb{\omega}_n})<\varepsilon\}$ and $p(\bb{\omega}_n)$ satisfies
	\begin{equation}
	    \label{e:nk-prior}
	    \int_{N_{\kappa/n^{\delta}}} p(\bb{\omega}_n)d\bb{\omega}_n \geq e^{-\tilde{\kappa}n^{1-\delta}}, n\to \infty
	\end{equation}
	for every $\kappa$ and $\tilde{\kappa}$ for some $0\leq \delta<1$.	Then, 
	\begin{equation}
	\label{e:b-bound-0}
	\log  \int \frac{L(\bb{\omega}_n)}{L_0}p(\bb{\omega}_n)d\bb{\omega}_n=o_{P_0^n}( n^{1-\delta})
	\end{equation}
	provided $0\leq \delta<1$.
\end{lemma}

\begin{proof} This proof uses ideas from the proof of Lemma 5 in \cite{LEE}.
	By Markov's inequality,
	\begin{align}
	\label{e:b-bound-1}
	\nonumber P_0^n\left(\left|\int \log  \frac{L(\bb{\omega}_n)}{L_0}p(\bb{\omega}_n)d\bb{\omega}_n\right|\geq \epsilon n^{1-\delta}\right)&\leq\frac{1}{\epsilon n^{1-\delta} }E_0^n\left(\left|\log\int  \frac{L(\bb{\omega}_n)}{L_0}p(\bb{\omega}_n)d\bb{\omega}_n\right|\right)\\
	\nonumber	&=\frac{1}{\epsilon n^{1-\delta}}\int \left|\log\int  \frac{L(\bb{\omega}_n)}{L_0}p(\bb{\omega}_n)d\bb{\omega}_n\right|L_0 d\mu\\
	&\leq \frac{1}{\epsilon n^{1-\delta}}\left(d_{KL}(L_0,L^*)+\frac{2}{e}\right)
	\end{align}
	where  $L^*=\int L(\bb{\omega}_n)p(\bb{\omega}_n)d\bb{\omega}_n$ and  the last equality follows from Lemma \ref{lem:mod-kl}. Further, 
	\begin{align}
	\label{e:b-bound-2}
	\nonumber	d_{KL}(L_0,L^*)&=E_0^n\left(\log \frac{L_0}{L^*}\right)=E_0^n\left(\log \frac{L_0}{\int L(\bb{\omega}_n)p(\bb{\omega}_n)d\bb{\omega}_n}\right)\\
	\nonumber	&\leq E_0^n\left(\log \frac{L_0}{\int_{N_{\kappa/n^{\delta}}} L(\bb{\omega}_n)p(\bb{\omega}_n)d\bb{\omega}_n }\right)\\
	\nonumber 	&\leq \int_{N_{\kappa/n^{\delta}}} p(\bb{\omega}_n)d\bb{\omega}_n+\int_{N_{\kappa/n^{\delta}}} d_{KL}(L_0,L(\bb{\omega}_n))p(\bb{\omega}_n)d\bb{\omega}_n \hspace{5mm}\text{Jensen's inequality}\\
	&\leq -\log e^{-\tilde{\kappa}n^{1-\delta}}+\kappa n^{1-\delta}=n^{1-\delta}(\kappa+\tilde{\kappa})
	\end{align}
	where the last equality follow from \eqref{e:nk-prior}. 
	
\noindent Using \eqref{e:b-bound-2} in \eqref{e:b-bound-1}, the result follows and taking $\tilde{\kappa} \to 0$ and $\kappa \to 0$.
\end{proof}

\begin{lemma}
	\label{lem:c-bound-0}
	Suppose $q$ satisfies
	$$\int d_{KL}(l_0,l(\bb{\omega}_n)) q(\bb{\omega}_n) d\bb{\omega}_n=o(n^{-\delta}),$$ then
	$$\int q(\bb{\omega}_n) \log  \frac{L(\bb{\omega}_n)}{L_0}d\bb{\omega}_n=o_{P_0^n}(n^{1-\delta})$$
\end{lemma}

	\noindent In this direction, note that
\begin{align*}
P_0^n\left(\left|\int q(\bb{\omega}_n) \log  \frac{L(\bb{\omega}_n)}{L_0}d\bb{\omega}_n\right|> n^{1-\delta} \epsilon\right) &\leq  P_0^n\left(\left|\int q(\bb{\omega}_n)\log  \frac{L(\bb{\omega}_n)}{L_0}d\bb{\omega}_n\right|\geq n^{1-\delta} \epsilon\right)\\
&\leq \frac{1}{n^{1-\delta}\epsilon}E_0^n\left(\left|\int q(\bb{\omega}_n) \log  \frac{L(\bb{\omega}_n)}{L_0}d\bb{\omega}_n\right|\right)\end{align*}
where the last result follows from Markov's Inequality
\begin{align*}
&\leq \frac{1}{n^{1-\delta}\epsilon}E_0^n\left(\int q(\bb{\omega}_n) \left|\log  \frac{L(\bb{\omega}_n)}{L_0}\right|d\bb{\omega}_n\right)\\
&=\frac{1}{n^{1-\delta}\epsilon}\int q(\bb{\omega}_n)\int \left|\log \frac{L_0}{L(\bb{\omega}_n)}\right| L_0 d\mu d\bb{\omega}_n
\end{align*}
Using Lemma \ref{lem:mod-kl}, we get
\begin{align*}
    &\leq \frac{1}{n^{1-\delta}\epsilon} \int q(\bb{\omega}_n)\left(d_{KL}(L_0,L(\bb{\omega}_n))+\frac{2}{e}\right)d\bb{\omega}_n \to 0
\end{align*}
since $\int q(\bb{\omega}_n)d_{KL}(L_0,L(\bb{\omega}_n))d\bb{\omega}_n =n \int q(\bb{\omega}_n)d_{KL}(l_0,l(\bb{\omega}_n))d\bb{\omega}_n=o(n^{1-\delta})$ .

\begin{lemma}
\label{lem:h-bd}
Let $H_{[]}(u,\widetilde{\mathcal{G}}_n,||.||_2)\leq K(n)\log\left(\frac{M_n}{u}\right)$
then	$$\int_0^{\varepsilon}H_{[]}(u,\widetilde{\mathcal{G}}_n,||.||_2)du\leq \varepsilon O(\sqrt{K(n)\log M_n})$$
\end{lemma}
\begin{proof} This proof uses some ideas from the proof of Lemma 1 in \cite{LEE}
\begin{align*}
\int_0^\varepsilon\sqrt{H(u,\widetilde{\mathcal{G}}_n,||.||_2)}	&\leq \sqrt{K(n)}\int_0^\varepsilon\sqrt{\log \left(\frac{M_{n}}{u}\right)}du=\frac{K(n)^{1/2}M_{n}}{2}\int_{\sqrt{\log\frac{ M_{n}}{\varepsilon}}}^{\infty} \nu^2 e^{-\nu^2/2}d\nu\\
	&=\frac{K(n)^{1/2}M_{n}}{2}\left(\frac{\varepsilon}{M_{n}}\sqrt{\log\frac{ M_{n}}{\varepsilon}}+\sqrt{2\pi}\int_{\sqrt{\log\frac{M_{n}}{\varepsilon}}}^\infty \frac{1}{\sqrt{2\pi}}e^{-\nu^2/2}d\nu\right)\\
&\sim \frac{K(n)^{1/2}M_{n}}{2}\left(\frac{\varepsilon}{M_{n}}\sqrt{\log \frac{M_{n}}{\varepsilon}}+\sqrt{2\pi} \frac{\phi\left(\sqrt{\log\frac{ M_{n}}{\varepsilon}}\right)}{\sqrt{\log\frac{ M_{n}}{\varepsilon}}}\right) \:\:\text{ by } \eqref{e:mills}\\
	&\leq \frac{\varepsilon}{2}\sqrt{K(n)}\sqrt{\log M_{n}-\log\varepsilon} \left(  1+\frac{1}{M_{n} \frac{\log M_{n}}{\varepsilon}}\right)=\varepsilon O(\sqrt{K(n)\log M_{n}})
	\end{align*}
\end{proof}

\begin{lemma}
	\label{lem:d-bound-0}
For any $\varepsilon>0$, suppose
	$$\frac{1}{\sqrt{n}}\int_0^\varepsilon  H(u,\widetilde{\mathcal{G}}_n,||.||_2) \leq \varepsilon^2$$
Then,
\begin{equation}
\label{e:rg-bound}
P_0^n\left(\sup_{\bb{\omega}_n \in \mathcal{V}_\varepsilon^c \cap \mathcal{F}_n} \frac{L(\bb{\omega}_n) }{L_0}\geq  e^{-n\varepsilon^2}\right) \to 0,\:\: n \to \infty
\end{equation}
\end{lemma}

\begin{proof}
Note that,	\begin{align*}
	\int_{\varepsilon^2/8}^{\sqrt{2}\varepsilon} H(u,\widetilde{\mathcal{G}}_n,||.||_2)du
	&\leq  \int_{0}^{\sqrt{2}\varepsilon} H(u,\widetilde{\mathcal{G}}_n,||.||_2)du\leq 2\varepsilon^2\sqrt{n}
	\end{align*}
	
	\noindent Therefore by Theorem 1 in \cite{WS1995}, for some constant $C>0$, we have
	$$P_0^n\left(\sup_{\bb{\omega}_n\in \mathcal{V}_\varepsilon^c \cap \mathcal{F}_n } \frac{L(\bb{\omega}_n)}{L_0}\geq e^{-n\varepsilon^2}\right)\leq 4\exp(-nC\varepsilon^2)$$
\end{proof}

\begin{lemma}
	\label{lem:pp-bound-0}
	Suppose, for some $r>0$, $p(\bb{\omega}_n)$ satisfies 
	$$\int_{\mathcal{F}_n^c} p(\bb{\omega}_n) d\bb{\omega}_n \leq e^{-\kappa n^r}, n \to \infty$$
	for any $\kappa>0$. Then, 	for every $\tilde{\kappa}<\kappa$.
	$$P_0^n\left( \int_{\bb{\omega}_n \in \mathcal{F}_n^c}\frac{L(\bb{\omega}_n)}{L_0}p(\bb{\omega}_n)d\bb{\omega}_n\geq  e^{-\tilde{\kappa} n^r}\right)\to 0$$
\end{lemma}

\begin{proof} This proof uses ideas from proof of Lemma 3 in \cite{LEE}.
\begin{align*}
	P_0^n\left(\int_{\bb{\omega}_n \in \mathcal{F}_n^c} \frac{L(\bb{\omega}_n)p(\bb{\omega}_n)}{L_0}d\bb{\omega}_n> e^{-\tilde{\kappa} n^r}\right)&=e^{\tilde{\kappa} n^r}E_0^n\left(\int_{\bb{\omega}_n \in \mathcal{F}_n^c} \frac{L(\bb{\omega}_n)}{L_0}p(\bb{\omega}_n)d\bb{\omega}_n\right)\\
	&=e^{\tilde{\kappa} n^r}\int 	\int_{\bb{\omega}_n \in \mathcal{F}_n^c} \frac{L(\bb{\omega}_n)}{L_0}p(\bb{\omega}_n)d\bb{\omega}_n L_0 d\mu\\
	&=e^{\tilde{\kappa} n^r} \int_{\bb{\omega}_n \in \mathcal{F}_n^c} p(\bb{\omega}_n)d\bb{\omega}_n\\
	&\leq e^{\tilde{\kappa}n^r}e^{-\kappa n^r}=e^{-(\kappa-\tilde{\kappa})n^r} \to 0, \:\: n\to \infty
	\end{align*}
\end{proof}

\subsection{Lemmas and Propositions for Theorem \ref{thm:var-cons} and \ref{thm:var-cons-1}}
\begin{lemma}
	\label{lem:e-bound-0}
	Let, $\widetilde{\mathcal{G}}_n=\{\sqrt{g}: g \in \mathcal{G}_n\}$
	where $\mathcal{G}_n$ is given by \eqref{e:G-def} with $K(n)\sim n^a$ and $C_n=e^{n^{b-a}}$.
	Then,
	$$\frac{1}{\sqrt{n}}\int_0^{\varepsilon}\sqrt{H_{[]}(u,\widetilde{\mathcal{G}}_n,||.||_2)}du\leq \varepsilon^2$$
	\end{lemma}
\begin{proof} 
This proof uses some ideas  from the proof of Lemma 2 in \cite{LEE}.

\noindent First, note that, by Lemma 4.1 in \cite{POL1990},
	$$N(\varepsilon,\mathcal{F}_n,||.||_\infty)\leq \left(\frac{3C_n}{\varepsilon}\right)^{K(n)}.$$
\noindent For $\bb{\omega}_1, \bb{\omega}_2 \in \mathcal{F}_n$, let $\widetilde{L}(u)=\sqrt{L_{u\bb{\omega}_1+(1-u)\bb{\omega}_2}(\bb{x},y)}$.
	Then,
	\begin{align}
	\label{e:l-bd-0}
	\nonumber \sqrt{L_{\bb{\omega}_1}(\bb{x},y)}-\sqrt{L_{\bb{\omega}_2}(\bb{x},y)}&=\int_0^1\frac{d\widetilde{L}}{du}du=\int_0^1\sum_{i=1}^{K(n)}\pp{\widetilde{L}}{\omega_i} \pp{\omega_i}{u}du=\sum_{i=1}^{K(n)}(\omega_{1i}-\omega_{2i})\int_0^1 \pp{\widetilde{L}}{\omega_i}du\\
	\nonumber&\leq  \sup_{i}|\omega_{1i}-\omega_{2i}|\int_0^1 \sum_{i=1}^{K(n)} \sup_{i}  \Big|\pp{\widetilde{L}}{\omega_i}\Big|du=K(n)\sup_{i}  \Big|\pp{\widetilde{L}}{\omega_i}\Big|||\omega_1-\omega_2||_{\infty} \\
	&\leq F(\bb{x},y)||\omega_1-\omega_2||_{\infty}
	\end{align}
	where the upper bound $F(\bb{x},y)= MK(n)C_n \sigma_0^{3/2}$ for a constant $M$. This is because
	\begin{align*}
	|\pp{\widetilde{L}}{\beta_j}|&\leq (8\pi e^2)^{-1/4}\sigma_0^{3/2}, j=0,\cdots, k_n\\
	|\pp{\widetilde{L}}{\gamma_{jh}}|&\leq (8\pi e^2)^{-1/4} C_n \sigma_0^{3/2}, j=0,\cdots, k_n, h=0,\cdots, p
	\end{align*}
	\noindent In view of \eqref{e:l-bd-0}  and Theorem 2.7.11 in \cite{VW1996}, we have
	$$N_{[]}(\varepsilon, \widetilde{\mathcal{G}}_n, ||.||_2)\leq \left(\frac{MK(n)C_n^2}{\varepsilon}\right)^{K(n)}$$
for some constant $M>0$. Therefore, 
$$H_{[]}(\varepsilon, \widetilde{\mathcal{G}}_n, ||.||_2)\lesssim K(n) \log \frac{K(n)C_n^2}{u}$$
Using, Lemma \ref{lem:h-bd} with $M_n=K(n)C_n^2$, we get
$$\int_0^{\varepsilon} \sqrt {H_{[]}(u, \widetilde{\mathcal{G}}_n, ||.||_2)} du\leq \varepsilon O(\sqrt{K(n)\log  K(n)C_n^2})=\varepsilon O(\sqrt{n^b}) $$
where the last equality holds since $K(n)\sim n^a$ and $C_n=e^{n^{b-a}}$.
Therefore, 
$$\frac{1}{\sqrt{n}}\int_0^{\varepsilon} {H_{[]}(u, \widetilde{\mathcal{G}}_n, ||.||_2)} du\leq \varepsilon^2$$
\end{proof}

\begin{lemma}
	\label{lem:p-bound-0}
	Let $$\mathcal{F}_n=\Big\{\bb{\theta}_n:|\theta_{in}|\leq C_n, i=1,\cdots, K(n)\Big\}\:\:\:K(n)\sim n^a, C_n=e^{n^{b-a}}$$
	\begin{enumerate}
		\item Suppose  $p(\bb{\omega}_n)$ satisfies \eqref{e:prior-t}.  
		\item Suppose  $p(\bb{\omega}_n)$ satisfies \eqref{e:prior-t-1}.
	\end{enumerate}
	Then for every $\kappa>0$,
	$$\int_{\bb{\omega}_n \in \mathcal{F}_n^c}p(\bb{\omega}_n)d\bb{\omega}_n\leq e^{-n\kappa},\:\: n \to \infty.$$
	
\end{lemma}

\begin{proof}
This proof uses some ideas from the proof of Theorem 1 in \cite{LEE}.
	Let $\mathcal{F}_{in}=\{\theta_{in}: |\theta_{in}|\leq C_n\}$, 
	$$\mathcal{F}_n=\cap_{i=1}^{K(n)} \mathcal{F}_{in}\implies \mathcal{F}_n^c= \cap_{i=1}^{K(n)}\mathcal{F}_{in}^c$$
	We first prove the Lemma for prior in 1.
	\begin{align*}
	\int_{\bb{\omega}_n \in \mathcal{F}_n^c}p(\bb{\omega}_n)d\bb{\omega}_n&\leq \sum_{i=1}^{K(n)}\int_{\mathcal{F}_{in}^c}\frac{1}{\sqrt{2\pi \zeta^2}}e^{-\frac{\theta_{in}^2}{2\zeta^2}}d\theta_{in} =2\sum_{i=1}^{K(n)}\int_{C_n}^{\infty}\frac{1}{\sqrt{2\pi \zeta^2}}e^{-\frac{\theta_{in}^2}{2\zeta^2}}d\theta_{in}\\
	&=2K(n)\left(1-\Phi\left(\frac{C_n}{\zeta}\right)\right)\sim \frac{K(n)}{C_n\zeta}e^{-\frac{C_n^2}{2\zeta^2}} \hspace{10mm}\text{ by }\eqref{e:mills}\\
	&\sim n^a  \zeta^{-1} e^{n^{b-a}}e^{-(e^{2{n^{b-a}}})/\zeta^2}\leq e^{-n\kappa}, n \to \infty 
	\end{align*}
	We next prove the Lemma for prior in 2.
	Analogous to the proof for prior in 1. we get,
	\begin{align*}
	\int_{\bb{\omega}_n \in \mathcal{F}_n^c}p(\bb{\omega}_n)d\bb{\omega}_n&\leq 2K(n)\left(1-\Phi\left(\frac{C_n}{\zeta n^{u/2}}\right)\right)\sim \frac{K(n)}{C_n\zeta n^{u/2}}e^{-\frac{C_n^2}{2\zeta^2 n^u}} \hspace{10mm}\text{ by }\eqref{e:mills}\\
	&\sim n^a\zeta^{-1}n^{-u/2} e^{n^{b-a}}e^{-(e^{2{n^{b-a}}}/\zeta^2 n^u})\leq e^{-n\kappa}, n \to \infty 
	\end{align*}
\end{proof}

\begin{proposition}
	\label{lem:v-bound}
	Suppose condition (C1) holds for some $0<a<1$ and one of the following two  hold.
	\begin{enumerate}
		\item Suppose  $p(\bb{\omega}_n)$ satisfies \eqref{e:prior-t}.  
		\item Suppose  $p(\bb{\omega}_n)$ satisfies \eqref{e:prior-t-1}.
	\end{enumerate}
	Then,
	$$\log \int_{\mathcal{V}_\varepsilon^c} \frac{L(\bb{\omega}_n)}{L_0} p(\bb{\omega}_n)d\bb{\omega}_n\leq \log 2-n\varepsilon^2 + o_{P_0^n}(1)$$	
\end{proposition}

\begin{proof}
This proof uses some ideas from the proof of Lemma 3 in \cite{LEE}.
	We shall first show
	$$P_0^n\left(\log \int_{\mathcal{V}_\varepsilon^c} \frac{L(\bb{\omega}_n)}{L_0} p(\bb{\omega}_n)d\bb{\omega}_n\geq \log 2-n\varepsilon^2\right)\to 0,\:\: n \to \infty$$
	\begin{align*}
	&P_0^n\left(\log \int_{\mathcal{V}_\varepsilon^c} \frac{L(\bb{\omega}_n)}{L_0} p(\bb{\omega}_n)d\bb{\omega}_n\geq \log 2-n\varepsilon^2\right)=
	P_0^n\left(\int_{\mathcal{V}_\varepsilon^c} \frac{L(\bb{\omega}_n)}{L_0} p(\bb{\omega}_n)d\bb{\omega}_n\geq 2e^{-n\varepsilon^2}\right)\\
	&\leq P_0^n\left(\int_{\mathcal{V}_\varepsilon^c \cap \mathcal{F}_n} \frac{L(\bb{\omega}_n)}{L_0} p(\bb{\omega}_n)d\bb{\omega}_n\geq e^{-n\varepsilon^2}\right)+P_0^n\left(\int_{\mathcal{F}_n^c} \frac{L(\bb{\omega}_n)}{L_0} p(\bb{\omega}_n)d\bb{\omega}_n\geq e^{-n\varepsilon^2}\right)
	\end{align*}
	Let $\mathcal{F}_n=\{\bb{\theta}_n:|\theta_{in}|\leq C_n=e^{n^{b-a}}, 0<a<b<1\}$.
	
	\noindent By Lemma  \ref{lem:e-bound-0}, $$\frac{1}{\sqrt{n}}\int_0^{\varepsilon}{H_{[]}(u, \widetilde{\mathcal{G}}_n, ||.||_2)} du \leq \varepsilon^2$$
	Therefore, by Lemma \ref{lem:d-bound-0}, we have 
	$$ P_0^n\left(\int_{\mathcal{V}_\varepsilon^c \cap \mathcal{F}_n} \frac{L(\bb{\omega}_n)}{L_0} p(\bb{\omega}_n)d\bb{\omega}_n\geq e^{-n\varepsilon^2}\right) \to 0$$
	
	\noindent In view of Lemma \ref{lem:p-bound-0}, for $p(\bb{\omega}_n)$ as in \eqref{e:prior-t} and \eqref{e:prior-t-1}, 
	$$\int_{\bb{\omega}_n \in \mathcal{F}_n^c} p(\bb{\omega}_n)d\bb{\omega}_n \leq e^{-2n \varepsilon^2}$$
	Therefore, using Lemma \ref{lem:pp-bound-0} with $r=1$, $\kappa=2\varepsilon^2$ and $\tilde{\kappa}=\varepsilon^2$, we have
	$$P_0^n\left(\int_{\mathcal{F}_n^c} \frac{L(\bb{\omega}_n)}{L_0} p(\bb{\omega}_n)d\bb{\omega}_n\geq e^{-n\varepsilon^2}\right) \to 0$$
	Finally to complete the proof, let 
	$$A_n=\left\{
	\log \int_{\mathcal{V}_\varepsilon^c} \frac{L(\bb{\omega}_n)}{L_0} p(\bb{\omega}_n)d\bb{\omega}_n\leq \log 2 -n \varepsilon^2\right\}$$ 
	then,
	\begin{align*}
	\log \int_{\mathcal{V}_\varepsilon^c} \frac{L(\bb{\omega}_n)}{L_0} p(\bb{\omega}_n)d\bb{\omega}_n&= 
	\left(\log \int_{\mathcal{V}_\varepsilon^c} \frac{L(\bb{\omega}_n)}{L_0} p(\bb{\omega}_n)d\bb{\omega}_n \right)1_{A_n}+\left( 
	\log \int_{\mathcal{V}_\varepsilon^c} \frac{L(\bb{\omega}_n)}{L_0} p(\bb{\omega}_n)d\bb{\omega}_n\right) 1_{A_n^c}\\
	&\leq (\log 2-n\varepsilon^2)+\underbrace{\left(n\varepsilon^2-\log 2+ \log \int_{\mathcal{V}_\varepsilon^c} \frac{L(\bb{\omega}_n)}{L_0} p(\bb{\omega}_n)d\bb{\omega}_n\right) 1_{A_n^c}}_{\tilde{A}_n}
	\end{align*}
	$$P_0^n(|\tilde{A}_n|>\epsilon)\leq P_0^n(1_{A_n^c}=1)\to 0$$
	as shown before. Thus, $\tilde{A}_n=o_{P_0^n}(1)$.

\end{proof}

\begin{proposition}
	\label{lem:kl-bound}
	Suppose condition (C1) holds with some $0<a<1$. Let $f_{\bb{\theta}_n}$ be a neural network satisfying assumption (A1) for some  $0\leq \delta<1-a$.
	With $\bb{\omega}_n=\bb{\theta}_n$, define, 
	\begin{equation}
	\label{e:Nk-def}
	N_{\kappa/n^{\delta}}=\{\bb{\omega}_n:(1/\sigma_0^2)\int (f_{\bb{\theta}_n}(\bb{x})-f_0(\bb{x}))^2<\kappa/n^{\delta}\}
	\end{equation}
For every $\tilde{\kappa}>0$, 	
\begin{enumerate}
	\item Suppose (A2) holds with same $\delta$ as (A1).	With $p(\bb{\omega}_n)$ as in \eqref{e:prior-t}  
	$$\int_{\bb{\omega}_n\in N_{\kappa/n^{\delta}}} p(\bb{\omega}_n)d\bb{\omega}_n \geq e^{-\tilde{\kappa}n^{1-\delta}}, \:\: n \to \infty.$$
	\item Suppose (A3) holds with some $v>1$.
	With $p(\bb{\omega}_n)$ as in \eqref{e:prior-t-1}  
	$$\int_{\bb{\omega}_n\in N_{\kappa/n^{\delta}}} p(\bb{\omega}_n)d\bb{\omega}_n \geq e^{-\tilde{\kappa}n^{1-\delta}},\:\: n \to \infty$$
\end{enumerate}	
\end{proposition}
\begin{proof}
This proof uses some ideas from the proof of Theorem 1 in \cite{LEE}. 

\noindent By assumption (A1), let $f_{\bb{\theta}_{0n}}(\bb{x})=\beta_{00}+\sum_{j=1}^{k_n}\beta_{j0}\psi(\gamma_{j0}^\top \bb{x})$ be a neural network such that
	\begin{equation}
	\label{e:n-delta-0}
	||f_{\bb{\theta}_{0n}}-f_0||_2\leq \frac{\kappa }{4n^{\delta}}
	\end{equation}
	
\noindent Define neighborhood $M_\kappa$ as follows
	\begin{align*}
	M_{\kappa}=\{&\bb{\omega}_n:|{\theta}_{in}-{\theta}_{i0n}|<\sqrt{\kappa/( 4n^{\delta}m_n)}\sigma_0, i=1,\cdots, K(n)\}
	\end{align*}
	where $m_n=8K(n)^2+8(p+1)^2(\sum_{j=1}^{K(n)}|\theta_{i0n}|)^2$.
	
	\noindent Note that $m_n\geq 8k_n+8(p+1)^2(\sum_{j=1}^{k_n}|\beta_{j0}|)^2$, thereby using Lemma \ref{lem:theta-bound} with $\epsilon=\sqrt{\kappa/(4n^\delta m_n)}\sigma_0$, we get,
	\begin{equation}
	\label{e:n-delta-01}
	\int (f_{\bb{\theta}_n}(\bb{x})-f_{\bb{\theta}_{0n}}(\bb{x}))^2 d\bb{x} \leq  \frac{\kappa}{4n^{\delta}}\sigma_0^2
	\end{equation}
for every $\bb{\omega}_n\in M_{k}$. In view of \eqref{e:n-delta-0} and \eqref{e:n-delta-01}, we have
	\begin{equation}
	\label{e:n-delta-1}
	\int (f_{\bb{\theta}_n}(\bb{x})-f_0(\bb{x}))^2 d\bb{x} \leq 2||f_{\bb{\theta}_n}-f_{\bb{\theta}_{0n}}||_2+2||f_{\bb{\theta}_{0n}}-f_{0}||_2 \leq \frac{\kappa\sigma_0^2}{n^{\delta}}\hspace{10mm}\text{ by \eqref{e:aplusb}}
	\end{equation}

	\noindent Using \eqref{e:n-delta-1} in \eqref{e:Nk-def} we get $\bb{\omega}_n \in N_{\kappa/n^\delta}$ for every $\bb{\omega}_n \in M_\kappa$.
	Therefore, 
	$$\int_{\bb{\omega}_n \in N_{\kappa/n^\delta}} p(\bb{\omega}_n)d\bb{\omega}_n\geq \int_{\bb{\omega}_n \in M_{\kappa}} p(\bb{\omega}_n)d\bb{\omega}_n$$
	
	\noindent We next show that,
	$$\int_{\bb{\omega}_n \in M_{\kappa}} p(\bb{\omega}_n)d\bb{\omega}_n>e^{-\tilde{\kappa}n^{1-\delta}}$$
	
	\noindent For notation simplicity, let $\delta_{n}=\sqrt{\kappa/(4 n^{\delta} m_n)}\sigma_0$

\noindent We first prove statement 1. of Proposition \ref{lem:kl-bound}.

\begin{align}
\label{e:the-bound}
\int_{\bb{\omega}_n \in \nonumber M_{\kappa}}p(\bb{\omega}_n)d\bb{\omega}_n&=\prod_{i=1}^{K(n)}\int_{\theta_{i0n}-\delta_{n}}^{\theta_{i0n}+\delta_{n}}\frac{1}{\sqrt{2\pi\zeta^2}}e^{-\frac{\theta_{in}^2}{2\zeta^2}}d\theta_{in}\\
\nonumber &= \prod_{i=1}^{K(n)}\frac{2\delta_{n}}{\zeta\sqrt{2\pi}}e^{-\frac{t_i^2}{2\zeta^2}},\:\: t_i\in [\theta_{i0n}-\delta_{n},\theta_{i0n}+\delta_{n}] \hspace{5mm}\text{by mean value theorem}\\
\nonumber &=\exp\left(-K(n)\left(\frac{1}{2}\log \frac{\pi\zeta^2}{2}-\log \delta_n\right)-\sum_{i=1}^{K(n)}\frac{t_i^2}{2\zeta^2}\right)\\
&\geq \exp\left(-K(n)\left(\frac{1}{2}\log \frac{\pi\zeta^2}{2}-\log \delta_n\right)-\sum_{i=1}^{K(n)}\frac{\max((\theta_{i0n}-\epsilon)^2,(\theta_{i0n}+\epsilon)^2)}{2\zeta^2}\right)
\end{align}
for any $\epsilon>0$ since $t_i \in [\theta_{i0n}-\epsilon,\theta_{i0n}+\epsilon]$ when $\delta_n \to 0$.

\noindent Using assumption (A2) and condition (C1) together with \eqref{e:aplusb}, we get
\begin{align}
\label{e:p1-ub}
\nonumber \sum_{i=1}^{K(n)}\max((\theta_{i0n}-\epsilon)^2,(\theta_{i0n}+\epsilon)^2)&\leq 2\sum_{i=1}^{K(n)}{\theta}^2_{i0n}+2\epsilon K(n)\leq \tilde{\kappa}n^{1-\delta}\hspace{5mm}\\
\nonumber K(n)\left(\frac{1}{2}\log \frac{\pi\zeta^2}{2}-\log \delta_n\right)&=K(n)\left(\frac{1}{2}\log \frac{\pi}{2}+\frac{1}{2}\delta  \log n+\frac{1}{2}\log 4+\frac{1}{2}\log m_n-\frac{1}{2}\log \kappa-\log \sigma_0\right) \\
&\leq \tilde{\kappa}n^{1-\delta} 
\end{align}
where the last inequality is a consequence of (C1) and the fact that $\log m_n =O(\log n)$ as shown next.
$$\log m_n \leq  \log(8K(n)^2+8(p+1)^2K(n)\sum_{j=1}^{K(n)}\theta_{i0n}^2)\leq \log (V_1 n^{2a}+V_2 n^{a}n^{1-\delta}) \leq V_3 \log n.$$
where the first inequality is a consequence of Cauchy Schwartz and the second  inequality is a consequence condition (C1) and assumption (A2).

\noindent Therefore, replacing \eqref{e:p1-ub} in \eqref{e:the-bound}, we get 
	\begin{align*}
\int_{\bb{\omega}_n \in M_{\kappa}} p(\bb{\omega}_n)d\bb{\omega}_n&\geq 
\exp(-\tilde{\kappa}n^{1-\delta})
\end{align*}

\noindent We next prove statement 2. of Proposition \ref{lem:kl-bound}.
\begin{align}
\label{e:the-bound-1}
\nonumber \int_{\bb{\omega}_n \in M_{\kappa}}p(\bb{\omega}_n)d\bb{\omega}_n&=\prod_{i=1}^{K(n)}\int_{\theta_{i0n}-\delta_{n}}^{\theta_{i0n}+\delta_{n}}\frac{1}{\sqrt{2\pi\zeta^2 n^u}}e^{-\frac{\theta_{in}^2}{2\zeta^2n^u}}d\theta_{in}\\
\nonumber &= \left(\frac{2\delta_{n}}{\sqrt{2\pi \zeta^2 n^u}}\right)^{K(n)}e^{-\sum_{i=1}^{K(n)}\frac{t_i^2}{2\zeta^2n^u}}, t_i\in [\theta_{i0n}-\delta_{n},\theta_{i0n}+\delta_{n}], \hspace{5mm}\text{by mean value theorem}\\
 &\geq \exp\left(-K(n)\left(\frac{1}{2}\log \frac{\pi\zeta^2}{2}+\frac{u}{2}\log n-\log \delta_n\right)-\sum_{i=1}^{K(n)}\frac{\max((\theta_{i0n}-\epsilon)^2,(\theta_{i0n}+\epsilon)^2)}{2\zeta^2 n^{u}}\right)
\end{align}
since for any $\epsilon>0$ since $t_i \in [\theta_{i0n}-\epsilon,\theta_{i0n}+\epsilon]$ for any $\epsilon>0$ when $\delta_n \to 0$.

\noindent Under assumption (A3) and condition  (C1) together with \eqref{e:aplusb}, we have
\begin{align}
\label{e:p1-ub-1}
\nonumber \frac{1}{n^{u}}\sum_{i=1}^{K(n)}\max((\theta_{i0n}-\epsilon)^2,(\theta_{i0n}+\epsilon)^2)\leq \frac{2}{n^u}\left(\sum_{i=1}^{K(n)}{\theta}^2_{i0n}+\epsilon K(n)\right)&\leq \tilde{\kappa}n^{1-\delta}\\
K(n)\left(\frac{1}{2}\log \frac{\pi}{2}+\frac{u}{2}\log n-\log \delta_n\right)&\leq \tilde{\kappa}n^{1-\delta}
\end{align}
where the last inequality holds by mimicking the argument in for the proof of part 1.

\noindent Therefore, replacing \eqref{e:p1-ub-1} in \eqref{e:the-bound-1}, we get 
\begin{align*}
\int_{\bb{\omega}_n \in M_{\kappa}} p(\bb{\omega}_n)d\bb{\omega}_n&\geq 
\exp(-\tilde{\kappa}n^{1-\delta})
\end{align*}
which completes the proof.
\end{proof}

\begin{proposition}
	\label{lem:q-bound}
Suppose condition (C1) and  assumption (A1) hold for some $0<a<1$ and $0\leq \delta<1-a$.
\begin{enumerate}
\item Suppose (A2)  holds with same $\delta$ as (A1) and  $p(\bb{\omega}_n)$ satisfies \eqref{e:prior-t}.  
\item Suppose (A3)  holds for some $v>1$ and  $p(\bb{\omega}_n)$ satisfies \eqref{e:prior-t-1}.	\end{enumerate}
Then, there exists a $q \in \mathcal{Q}_n$ with $ \mathcal{Q}_n$  as in \eqref{e:var-family} such that
\begin{equation}
\label{e:q-bound}
d_{KL}(q(.),\pi(.|\bb{y}_n,\bb{X}_n))=o_{P_0^n}(n^{1-\delta})
\end{equation}

	\end{proposition}

\begin{proof}
	\begin{align*}
	d_{KL}(q(.),\pi(.|\bb{y}_n,\bb{X}_n))&=\int q(\bb{\omega}_n)\log q(\bb{\omega}_n)d\bb{\omega}_n-\int q(\bb{\omega}_n) \log \pi(\bb{\omega}_n|\bb{y}_n,\bb{X}_n)d\bb{\omega}_n\\
	&=\int q(\bb{\omega}_n)\log q(\bb{\omega}_n)d\bb{\omega}_n- \int q(\bb{\omega}_n) \log \frac{L(\bb{\omega}_n)p(\bb{\omega}_n)}{\int L(\bb{\omega}_n)p(\bb{\omega}_n)d\bb{\omega}_n} d\bb{\omega}_n\\
	&=\underbrace{d_{KL}(q(.),p(.))}_{\textcircled{1}}\underbrace{-\int q(\bb{\omega}_n) \log  \frac{L(\bb{\omega}_n)}{L_0}d\bb{\omega}_n}_{\textcircled{2}}+\underbrace{\log \int p(\bb{\omega}_n)\frac{L(\bb{\omega}_n)}{L_0}  d\bb{\omega}_n}_{\textcircled{3}}
	\end{align*}
	
\noindent We first prove statement 1. of the Lemma. 

\noindent Here, we have  
\begin{equation}
\label{e:q-def}
p(\bb{\omega}_n)=\prod_{i=1}^{K(n)}\frac{1}{\sqrt{2\pi \zeta^2}}e^{-\frac{\theta_{in}^2}{2\zeta^2}}\hspace{5mm} q(\bb{\omega}_n)=\prod_{i=1}^{K(n)}\sqrt{\frac{n}{2\pi \tau^2}}e^{-\frac{n}{2\tau^2}(\theta_{in}-\theta_{0in})^2}
\end{equation}

\begin{align} 
\label{e:q-q}
\nonumber d_{KL}(q(.),p(.))&=\int q(\bb{\omega}_n)\log q(\bb{\omega}_n)d\bb{\omega}_n-\int q(\bb{\omega}_n)\log p(\bb{\omega}_n)d\bb{\omega}_n\\
\nonumber&=\sum_{i=1}^{K(n)}\int \left(\frac{1}{2}\log n -\frac{1}{2}\log 2\pi -\log \tau-\frac{n(\theta_{in}-\theta_{i0n})^2}{2\tau^2}\right)\frac{n}{\sqrt{2\pi\tau^2}}e^{-\frac{n(\theta_{in}-\theta_{i0n})^2}{2\tau^2}}d\theta_{in}\\
\nonumber&-\sum_{i=1}^{K(n)}\int \left(-\frac{1}{2}\log 2\pi -\log \zeta-\frac{\theta_{in}^2}{2\zeta^2}\right)\frac{n}{\sqrt{2\pi\tau^2}}e^{-\frac{n(\theta_{in}-\theta_{i0n})^2}{2\tau^2}}d\theta_{in}\\
&=\frac{K(n)}{2}(\log n-\log 2\pi-2\log \tau-1)+\frac{K(n)}{2}(-\log 2\pi-2\log \zeta)+\sum_{i=1}^{K(n)}\frac{\theta_{i0n}^2+\tau^2/n}{2\zeta^2}
\end{align}
Thus,
$$\textcircled{1}=\frac{K(n)}{2}\log n+K(n)\log \frac{\zeta}{\tau\sqrt{e}}+\frac{1}{2\zeta^2}\sum_{i=1}^{K(n)}\theta_{i0n}^2+\frac{\tau^2}{2\zeta^2 n}=o(n^{1-\delta})$$
where the last equality is a consequence of condition (C1) and assumption (A2).

\noindent For, $\textcircled{2}$ note that
\begin{align}
\label{e:dk-bound-0}
\nonumber d_{KL}(l_0,l_{\bb{\omega}_n})&= \int \int \left( \frac{1}{2}\log\frac{\sigma_0^2}{\sigma_0^2}-\frac{1}{2\sigma_0^2}(y-f_0(\bb{x}))^2+\frac{1}{2\sigma_0^2}(y-f_{\bb{\theta}_n}(\bb{x}))^2\right)\frac{1}{\sqrt{2\pi\sigma_0^2}}e^{-\frac{(y-f_0(\bb{x}))^2}{2\sigma_0^2}}dy d\bb{x}\\
&=\frac{1}{2\sigma_0^2}\int (f_{\bb{\theta}_n}(\bb{x})-f_0(\bb{x}))^2 d\bb{x}
\end{align}
By Lemma \ref{lem:f-bound} part 1., $d_{KL}(l_0,l_{\bb{\omega}_n})=o(n^{-\delta})$. Therefore, by Lemma \ref{lem:c-bound-0}, $\textcircled{2}=o_{P_0^n}(n^{1-\delta})$.

\noindent Using part 1. of Proposition \ref{lem:kl-bound} in Lemma \ref{lem:b-bound-0}, we get $\textcircled{3}= o_{P_0^n}(n^{1-\delta})$.

\noindent Next we prove statement 2. of the Lemma. 

\noindent Here, we have  
\begin{equation}
\label{e:q-def-1}
p(\bb{\omega}_n)\prod_{i=1}^{K(n)}\frac{1}{\sqrt{2\pi \zeta^2n^u}}e^{-\frac{\theta_{in}^2}{2\zeta^2n^u}}\hspace{5mm} q(\bb{\theta}_n)=\prod_{i=1}^{K(n)}\sqrt{\frac{n^{v+1}}{2\pi \tau^2}}e^{-\frac{n^{v+1}}{2\tau^2}(\theta_{in}-\theta_{0in})^2}
\end{equation}

\begin{align} 
\label{e:q-q-1}
\nonumber d_{KL}(q(.),p(.))&= \int q(\bb{\omega}_n)\log q(\bb{\omega}_n)d\bb{\omega}_n-\int q(\bb{\omega}_n)\log p(\bb{\omega}_n)d\bb{\omega}_n\\
\nonumber &=\frac{1}{2}\sum_{i=1}^{K(n)}\int \left(\log n^{v+1} -\log 2\pi -2\log \tau-\frac{(\theta_{in}-\theta_{i0n})^2}{\tau^2/n^{v+1}}\right)\frac{n^{v+1}}{\sqrt{2\pi\tau^2}}e^{-\frac{(\theta_{in}-\theta_{i0n})^2}{2\tau^2/n^{v+1}}}d\theta_{in}\\
\nonumber &-\frac{1}{2}\sum_{i=1}^{K(n)}\int \left(-\log 2\pi -2\log \zeta-\log n^u-\frac{\theta_{in}^2}{\zeta^2 n^u}\right)\frac{n^{v+1}}{\sqrt{2\pi\tau^2}}e^{-\frac{n(\theta_{in}-\theta_{i0n})^2}{2\tau^2/n^{v+1}}}d\theta_{in}\\
\nonumber &=\frac{K(n)}{2}((v+1)\log n-\log 2\pi-2\log \tau-1)+\frac{(K(n)}{2}(-\log 2\pi-2\log \zeta-u\log n)\\
&+\sum_{i=1}^{K(n)}\frac{\theta_{i0n}^2+\frac{\tau^2}{n^{v+1}}}{2\zeta^2 n^u}
\end{align}
Thus,
$$\textcircled{1}=(v+1+u)\frac{K(n)}{2}\log n+K(n)\log \frac{\zeta}{\tau\sqrt{e}}+\frac{1}{2\zeta^2 n^{u}}\sum_{i=1}^{K(n)}\theta_{i0n}^2+\frac{\tau^2}{2\zeta^2 n^{u+v+1}}=o(n^{1-\delta})$$
where the last equality is a consequence of condition (C1) and assumption (A3).

\noindent By Lemma \ref{lem:f-bound} part 2., $d_{KL}(l_0,l_{\bb{\omega}_n})=o(n^{-\delta})$. Therefore, by Lemma \ref{lem:c-bound-0}, $\textcircled{2}=o_{P_0^n}(n^{-\delta})$.

\noindent Using part 2. of Proposition \ref{lem:kl-bound} in Lemma \ref{lem:b-bound-0}, we get $\textcircled{3}= o_{P_0^n}(n^{1-\delta})$.

\end{proof}

\subsection{Lemmas and Propositions for Theorem \ref{thm:var-var}}
\begin{lemma}
	\label{lem:e-bound-v}
	Let, $
	\widetilde{\mathcal{G}}_n=\{\sqrt{g}: g \in \mathcal{G}_n\}$ where $\mathcal{G}_n$ is given by \eqref{e:G-def-v} with $K(n)\sim n^a$, $C_n=e^{n^{b-a}}$, $D_n=e^{n^b}$.
		Then,
	$$\frac{1}{\sqrt{n}}\int_0^{\varepsilon}\sqrt{H_{[]}(u,\widetilde{\mathcal{G}}_n,||.||_2)}du\leq \varepsilon^2$$
\end{lemma}
\begin{proof}
This proof uses some ideas from the proof of Lemma 2 in \cite{LEE}.
First, note that by Lemma 4.1 in \cite{POL1990}, we have
	$$N(\varepsilon,\mathcal{F}_n,||.||_\infty)\leq\Big(\frac{3C_n }{\varepsilon}\Big)^{K(n)}\Big(\frac{3D_n}{\varepsilon}\Big)$$
	
\noindent For $\bb{\omega}_1, \bb{\omega}_2 \in \mathcal{F}_n$, let $\widetilde{L}(u)=\sqrt{L_{u\bb{\omega}_1+(1-u)\bb{\omega}_2}(\bb{x},y)}$.

\noindent Using \eqref{e:l-bd-0}, we get
\begin{align}
\label{e:l-bd-v} \sqrt{L_{\bb{\omega}_1}(\bb{x},y)}-\sqrt{L_{\bb{\omega}_2}(\bb{x},y)}&\leq  \underbrace{(K(n)+1)\sup_{i}  \Big|\pp{\widetilde{L}}{\omega_i}}_{F(\bb{x},y)}\Big|||\omega_1-\omega_2||_{\infty} \leq F(\bb{x},y)||\omega_1-\omega_2||_{\infty}
\end{align}
where the upper bound on $F(\bb{x},y)$ is calculated as:
	\begin{align*}
	|\pp{\widetilde{L}}{\beta_j}|&\leq (8\pi e^2)^{-1/4}C_n^{3/2}, j=0,\cdots, k_n\\
	|\pp{\widetilde{L}}{\gamma_{jh}}|&\leq (8\pi e^2)^{-1/4} C_n^{5/2}, j=0,\cdots, k_n, h=0,\cdots, p\\
	|\pp{\widetilde{L}}{\rho}|&\leq ((16\pi)^{-1/4}+(\pi e^2/8)^{-1/4}) C_n^{5/2}
	\end{align*}
	\noindent In view of \eqref{e:l-bd-0}  and Theorem 2.7.11 in \cite{VW1996}, we have
$$N_{[]}(\varepsilon, \widetilde{\mathcal{G}}_n, ||.||_2)\leq \Big(\frac{MK(n)C_n^{7/2}}{\varepsilon}\Big)^{K(n)}\Big(\frac{MD_n K(n)C_n^{5/2}}{\varepsilon}\Big)$$
for some constant $M>0$. Therefore, 
$$H_{[]}(\varepsilon, \widetilde{\mathcal{G}}_n, ||.||_2)\lesssim K(n) \log \frac{K(n)C_n^{7/2} (D_n K(n)C_n^{5/2})^{1/K(n)}}{\varepsilon}$$
Using, Lemma \ref{lem:h-bd} with $M_n=K(n)C_n^{7/2} (D_n K(n)C_n^{5/2})^{1/K(n)}$, we get
$$\int_0^{\varepsilon} \sqrt {H_{[]}(u, \widetilde{\mathcal{G}}_n, ||.||_2)} du\lesssim \varepsilon O
\left(\sqrt{K(n)\log(  K(n)C_n^{7/2} (D_n K(n)C_n^{5/2})^{1/K(n)}}\right)=\varepsilon O(\sqrt{n^b}) $$
where the last equality holds since $K(n)\sim n^a$, $C_n=e^{n^{b-a}}$, $D_n=e^{n^b}$.

\noindent Therefore, 
$$\frac{1}{\sqrt{n}}\int_0^{\varepsilon} {H_{[]}(u, \widetilde{\mathcal{G}}_n, ||.||_2)} du\leq \varepsilon^2$$
\end{proof}

\begin{lemma}
	\label{lem:d-bound-1}
	Let $$\mathcal{F}_n=\Big\{ (\bb{\theta}_n,\sigma):|\theta_{in}|\leq C_n, i=1,\cdots, K(n), 1/C_n\leq \sigma \leq D_n,\Big\}$$
	where $D_n\sim n^a$, $C_n=e^{n^{b-a}}$, $D_n=e^{n^b}$, $0<a<b<1$. Suppose $p(\bb{\omega}_n)$ satisfies \eqref{e:prior-t-v}, then for any $\kappa>0$ and $0<r<b$, $$\int_{\bb{\omega}_n \in \mathcal{F}_n^c}p(\bb{\omega}_n)d\bb{\omega}_n\leq e^{-\kappa n^r}, n \to \infty$$	
\end{lemma}

\begin{proof}
This proof uses some ideas from the proof of Theorem 1 in \cite{LEE}.

\noindent Let $\mathcal{F}_{in}=\{\theta_{in}: |\theta_{in}|\leq C_n\}$ and $\mathcal{F}_{0n}=\{\sigma: 1/C_n\leq \sigma\leq  D_n\}$. 
	$$\mathcal{F}_n=\mathcal{F}_{0n}\cap_{i=1}^{K(n)} \mathcal{F}_{in}\implies \mathcal{F}_n^c=\mathcal{F}_{0n}^c\cup \cup_{i=1}^{K(n)}\mathcal{F}_{in}^c$$
	\begin{align*}
	\int_{\bb{\omega}_n \in \mathcal{F}_n^c}p(\bb{\omega}_n)d\bb{\omega}_n&\leq \int_{\mathcal{F}_{0n}^c} \frac{\lambda^{\alpha}}{\Gamma(\alpha)}\Big(\frac{1}{\sigma^2}\Big)^{\alpha+1}e^{-\frac{\lambda}{\sigma^2}}d\sigma^2+\sum_{i=1}^{K(n)}\int_{\mathcal{F}_{in}^c}\frac{1}{\sqrt{2\pi \zeta^2}}e^{-\frac{\theta_{in}^2}{2\zeta^2}}d\theta_{in} \\
	&=\int_{0}^{1/C_n^2} \frac{\lambda^{\alpha}}{\Gamma(\alpha)}\Big(\frac{1}{\sigma^2}\Big)^{\alpha+1}e^{-\frac{\lambda}{\sigma^2}}d\sigma^2+\int_{D_n^2}^{\infty}\frac{\lambda^{\alpha}}{\Gamma(\alpha)}\Big(\frac{1}{\sigma^2}\Big)^{\alpha+1}e^{-\frac{\lambda}{\sigma^2}}d\sigma^2+e^{-n\kappa} \leq 
	\end{align*}
	where the last equality is a consequence of Lemma \ref{lem:p-bound-0}.
	
\begin{align*}
\hspace{20mm}	&=\int_{0}^{1/C_n} \frac{\lambda^{\alpha}}{\Gamma(\alpha)}\Big(\frac{1}{\sigma}\Big)^{\alpha+1}e^{-\frac{\lambda}{\sigma}}d\sigma+\int_{D_n}^{\infty}\frac{\lambda^{\alpha}}{\Gamma(\alpha)}\Big(\frac{1}{\sigma}\Big)^{\alpha+1}e^{-\frac{\lambda}{\sigma}}d\sigma+e^{-n\kappa}\\
&=\int_{C_n}^{\infty} \frac{\lambda^{\alpha}}{\Gamma(\alpha)}u^{\alpha-1}e^{-u}du+\int_{0}^{1/D_n}\frac{\lambda^{\alpha}}{\Gamma(\alpha)}u^{\alpha-1}e^{-\lambda u}du+e^{-n\kappa}\\
&\lesssim \int_{C_n}^{\infty} \frac{\lambda^{\alpha}}{\Gamma(\alpha)}e^{-u/2}du+\int_{0}^{1/D_n}\frac{\lambda^{\alpha}}{\Gamma(\alpha)}u^{\alpha-1}du+e^{-n\kappa} \hspace{5mm} x^\alpha e^{-x}\leq e^{-x/2}, x\to \infty \\
&\sim e^{-e^{n^{b-a}}/2}+e^{-\alpha n^b}+e^{-n\kappa}\leq e^{-\kappa n^r}
\end{align*}
for any $\kappa>0$ and $b<r<1$.
\end{proof}

\begin{proposition}
	\label{lem:v-bound-v}
	Suppose condition  (C1) holds with $0<a<1$ and  $p(\bb{\omega}_n)$ satisfies \eqref{e:prior-t-v}.
	Then,
	$$\log \int_{\mathcal{V}_\varepsilon^c} \frac{L(\bb{\omega}_n)}{L_0} p(\bb{\omega}_n)d\bb{\omega}_n\leq \log 2-n^r\varepsilon^2 + o_{P_0^n}(1)$$
	for every $0<r<1$.	
\end{proposition}

\begin{proof}
This proof uses some ideas from the proof of Lemma 3 in \cite{LEE}.
	We shall first show
	$$P_0^n\left(\log \int_{\mathcal{V}_\varepsilon^c} \frac{L(\bb{\omega}_n)}{L_0} p(\bb{\omega}_n)d\bb{\omega}_n\geq \log 2-n^r\varepsilon^2\right)\to 0,\:\: n \to \infty$$
	\begin{align*}
	&P_0^n\left(\log \int_{\mathcal{V}_\varepsilon^c} \frac{L(\bb{\omega}_n)}{L_0} p(\bb{\omega}_n)d\bb{\omega}_n\geq \log 2-n^r\varepsilon^2\right)=
	P_0^n\left(\int_{\mathcal{V}_\varepsilon^c} \frac{L(\bb{\omega}_n)}{L_0} p(\bb{\omega}_n)d\bb{\omega}_n\geq 2e^{-n^r\varepsilon^2}\right)\\
	&=P_0^n\left(\int_{\mathcal{V}_\varepsilon^c \cap \mathcal{F}_n} \frac{L(\bb{\omega}_n)}{L_0} p(\bb{\omega}_n)d\bb{\omega}_n\geq e^{-n^r\varepsilon^2}\right)+P_0^n\left(\int_{\mathcal{V}_\varepsilon^c \cap \mathcal{F}_n^c} \frac{L(\bb{\omega}_n)}{L_0} p(\bb{\omega}_n)d\bb{\omega}_n\geq e^{-n^r\varepsilon^2}\right)\\
	&\leq P_0^n\left(\int_{\mathcal{V}_\varepsilon^c \cap \mathcal{F}_n} \frac{L(\bb{\omega}_n)}{L_0} p(\bb{\omega}_n)d\bb{\omega}_n\geq e^{-n\varepsilon^2}\right)+P_0^n\left(\int_{\mathcal{F}_n^c} \frac{L(\bb{\omega}_n)}{L_0} p(\bb{\omega}_n)d\bb{\omega}_n\geq e^{-n^r\varepsilon^2}\right)\hspace{3mm} \text{since } e^{-n^r \varepsilon^2}\geq e^{-n\varepsilon^2}
	\end{align*}
	With $\mathcal{F}_n$ as in \eqref{e:G-def-v} with $k_n\sim n^a$, $C_n=e^{n^{b-a}}$ and $D_n=e^{n^b}$ where $0<a<b<1$
	
	\noindent By Lemma  \ref{lem:e-bound-v}, $$\frac{1}{\sqrt{n}}\int_0^{\varepsilon}{H_{[]}(u, \widetilde{\mathcal{G}}_n, ||.||_2)} du \leq \varepsilon^2$$
	Therefore, by Lemma \ref{lem:d-bound-1}, we have 
	$$ P_0^n(\int_{\mathcal{V}_\varepsilon^c \cap \mathcal{F}_n} \frac{L(\bb{\omega}_n)}{L_0} p(\bb{\omega}_n)d\bb{\omega}_n\geq e^{-n\varepsilon^2}) \to 0$$
	
	\noindent In view of Lemma \ref{lem:p-bound-0}, for $p(\bb{\omega}_n)$ as in \eqref{e:prior-t-v}, 	for any $0<r<b$, 
	$$\int_{\bb{\omega}_n \in \mathcal{F}_n^c} p(\bb{\omega}_n)d\bb{\omega}_n  \leq  e^{-2n^r\varepsilon^2 }, n \to \infty$$
Therefore, by Lemma \ref{lem:pp-bound-0} with $r=r$, $\kappa=2\varepsilon^2$ and $\tilde{\kappa}=\varepsilon^2$, we have
	$$P_0^n(\int_{\mathcal{F}_n^c} \frac{L(\bb{\omega}_n)}{L_0} p(\bb{\omega}_n)d\bb{\omega}_n\geq e^{-n^r\varepsilon^2}) \to 0$$
	Since $b$ can be arbitrarily close to 1, the remaining part of the proof follows on lines of Proposition \ref{lem:v-bound}
	
\end{proof}

\begin{proposition}
	\label{lem:kl-bound-v}
		Suppose condition (C1) holds with some $0<a<1$. Let $f_{\bb{\theta}_n}$ be a neural network satisfying assumption (A1) and (A2) for some  $0\leq \delta<1-a$. With $\bb{\omega}_n=(\bb{\theta}_n,\sigma^2)$, define, 
	\begin{equation}
	\label{e:Nk-def-v}
		N_{\kappa/n^{\delta}}=\left\{\bb{\omega}_n:d_{KL}(l_0,l(\bb{\omega}_n))=\frac{1}{2}\log\frac{\sigma^2}{\sigma_0^2}-\frac{1}{2}\Big(1-\frac{\sigma_0^2}{\sigma^2}\Big)+\frac{1}{2\sigma^2}\int (f_{\bb{\theta}_n}(\bb{x})-f_0(\bb{x}))^2 d\bb{x}<\epsilon\right\}
	\end{equation}
\item For every $\tilde{\kappa}>0$, with $p(\bb{\omega}_n)$ as in \eqref{e:prior-t-v}, we have  
$$\int_{\bb{\omega}_n\in N_{\kappa/n^{\delta}}} p(\bb{\omega}_n)d\bb{\omega}_n \geq e^{-\tilde{\kappa}n^{1-\delta}}, \:\: n \to \infty.$$

\end{proposition}
\begin{proof}
This proof uses some ideas from the proof of Theorem 1 in \cite{LEE}. 

\noindent By assumption (A1), let $f_{\bb{\theta}_{0n}}(\bb{x})=\beta_{00}+\sum_{j=1}^{k_n}\beta_{j0}\psi(\gamma_{j0}^\top \bb{x})$ be a neural network such that
	\begin{equation}
	\label{e:n-delta-1-v}
	||f_{\bb{\theta}_{0n}}-f_0||_2\leq \frac{\kappa }{8n^{\delta}}
	\end{equation}
	
	\noindent Define neighborhood $M_\kappa$ as follows
	\begin{align*}
	M_{\kappa}=\{\bb{\omega}_n:|\sigma-\sigma_0|<\sqrt{\kappa/2n^{\delta}}\sigma_0,|{\theta}_{in}-{\theta}_{i0n}|<\sqrt{\kappa/( 8n^{\delta}m_n)}\sigma_0, i=1,\cdots, K(n)\}
	\end{align*}
	where $m_n=8K(n)^2+8(p+1)^2(\sum_{j=1}^{K(n)}|\theta_{i0n}|)^2$.
	
	\noindent  Note that $m_n\geq 8k_n+8(p+1)^2(\sum_{j=1}^{k_n}|\beta_{j0}|)^2$, thereby using Lemma \ref{lem:theta-bound} with $\epsilon=\sqrt{\kappa/(8n^\delta m_n)}\sigma_0$, we get
	\begin{equation}
	\label{e:n-delta-2-v}
	\int (f_{\bb{\theta}_n}(\bb{x})-f_{\bb{\theta}_{0n}}(\bb{x}))^2 d\bb{x} \leq  \frac{\kappa}{8n^{\delta}}\sigma_0^2
	\end{equation}
for any  $\bb{\omega}_n\in M_{k}$,
	
\noindent In view of \eqref{e:n-delta-1-v} and \eqref{e:n-delta-2-v} together with \eqref{e:a-form}, we have
	\begin{equation}
	\label{e:n-delta-3-v}
	\int (f_{\bb{\theta}_n}(\bb{x})-f_0(\bb{x}))^2 d\bb{x} \leq 2||f_{\bb{\theta}_n}-f_{\bb{\theta}_{0n}}||_2+2||f_{\bb{\theta}_{0n}}-f_{0}||_2 \leq \frac{\kappa\sigma_0^2}{2n^{\delta}}
	\end{equation}

	\noindent By Lemma \ref{lem:sig-bound},
	\begin{align}
	\label{e:n-delta-4-v}
	\nonumber\frac{1}{2}\log\frac{\sigma^2}{\sigma_0^2}-\frac{1}{2}\Big(1-\frac{\sigma_0^2}{\sigma^2}\Big)&\leq \frac{\kappa}{2n^{\delta}}\\
	\frac{1}{2\sigma^2}\leq \frac{1}{2\sigma_0^2(1-\sqrt{\kappa/2n^{\delta}})^2}&\leq \frac{1}{\sigma_0^2}
	\end{align}

	\noindent Using \eqref{e:n-delta-3-v} and \eqref{e:n-delta-4-v} in \eqref{e:Nk-def-v} we get $\bb{\omega}_n \in N_{\kappa/n^\delta}$ for every $\bb{\omega}_n \in M_\kappa$.
	Therefore, 
	$$\int_{\bb{\omega}_n \in N_{\kappa/n^\delta}} p(\bb{\omega}_n)\geq \int_{\bb{\omega}_n \in M_{\kappa}} p(\bb{\omega}_n)$$
	
	\noindent We next show that,
	$$\int_{\bb{\omega}_n \in M_{\kappa}} p(\bb{\omega}_n)d\bb{\omega}_n>e^{-\tilde{\kappa}n^{1-\delta}}$$
	
	\noindent For notation simplicity, let $\delta_{1n}=\sqrt{\kappa/2n^{\delta}}\sigma_0$ and $\delta_{2n}=\sqrt{\kappa/(8 n^{\delta} m_n)}\sigma_0$
	
	\begin{align*}
 	\int_{\bb{\omega}_n \in M_{\kappa}}p(\bb{\omega}_n)d\bb{\omega}_n&=\int_{(\sigma_0-\delta_{1n})^2}^{(\sigma_0+\delta_{1n})^2}p(\sigma^2)d\sigma^2\prod_{i=1}^{K(n)}\int_{\theta_{i0n}-\delta_{2n}}^{\theta_{i0n}+\delta_{2n}} p(\theta_{in})d\theta_{in}\\
	&\geq \int_{(\sigma_0-\delta_{1n})^2}^{(\sigma_0+\delta_{1n})^2}p(\sigma^2)d\sigma^2 e^{-(\tilde{\kappa}/2)n^{1-\delta}}
	 \end{align*}
	 where first to second step follows from part 1. of Lemma \ref{lem:kl-bound} since $p(\bb{\theta}_{n})$ satisfies \eqref{e:prior-t}. Next,
	\begin{align}
	\label{e:p-rel-1-v}
	\nonumber\int_{(\sigma_0-\delta_{1n})^2}^{(\sigma_0+\delta_{1n})^2}p(\sigma^2)d\sigma^2 \nonumber&=\int_{(\sigma_0-\delta_{1n})^2}^{(\sigma_0+\delta_{1n})^2}\frac{\beta^{\alpha}}{\Gamma(\alpha)}\Big(\frac{1}{\sigma^2}\Big)^{\alpha+1}e^{-\frac{\beta}{\sigma^2}}d\sigma^2=\int_{\sigma_0-\delta_{1n}}^{\sigma_0+\delta_{1n}}\frac{\beta^{\alpha}}{\Gamma(\alpha)}\Big(\frac{1}{\sigma}\Big)^{\alpha+1}e^{-\frac{\beta}{\sigma}}d\sigma\\
\nonumber&=2\delta_{1n}  \underbrace{\frac{\beta^{\alpha}}{\Gamma(\alpha)}\Big(\frac{1}{t}\Big)^{\alpha+1}e^{-\frac{\beta}{t}}}_{f(t)},\:\: t \in [\sigma_0-\delta_{1n},\sigma_0+\delta_{1n}]\hspace{5mm}\text{ by mean value theorem}\\
\nonumber&\geq  \frac{\delta_{1n}\beta^{\alpha}}{\Gamma(\alpha)} \Big(\frac{1}{\sigma_0+\epsilon}\Big)^{\alpha+1}e^{-\frac{\beta}{\sigma_0-\epsilon}}\\
&=\exp\left(-\left(-\log \delta_{1n}-\alpha \log \beta+\log \Gamma(\alpha)+(\alpha+1)\log (\sigma_0+\epsilon)+\frac{\beta}{\sigma_0-\epsilon}\right)\right)
\end{align}
where the third inequality holds since for any $\epsilon>0$, $t\in [\sigma_0-\epsilon,\sigma_0+\epsilon]$  when $\delta_n \to 0$.
\noindent Now, \begin{align}
\label{e:s-bound-v}
\nonumber&-\log \delta_{1n}-\alpha \log \lambda+\log \Gamma(\alpha)+(\alpha+1)\log (\sigma_0+\epsilon)+ \frac{\lambda}{\sigma_0-\epsilon}\\
&=\frac{1}{2}\delta \log n +\frac{1}{2} \log 2 -\frac{1}{2}\log \kappa-\log \sigma_0-\alpha \log \lambda+\log \Gamma(\alpha)+(\alpha+1)\log (\sigma_0+\epsilon)+ \frac{\lambda}{\sigma_0-\epsilon}\leq (\tilde{\kappa}/2)n^{1-\delta}
\end{align}

\noindent Using \eqref{e:s-bound-v} in \eqref{e:p-rel-1-v}, we get
$$\int_{\bb{\omega}_n \in M_\kappa} p(\bb{\omega}_n)d\bb{\omega}_n \geq e^{-\tilde{\kappa}n^{1-\delta}}$$
which completes the proof.
\end{proof}

\begin{proposition}
	\label{lem:q-bound-v}
Suppose condition (C1) and  assumptions (A1) and (A2) hold for some $0<a<1$ and $0\leq \delta<1-a$. Suppose the prior $p(\bb{\omega}_n)$ satisfies \eqref{e:prior-t-v}. 
	
	\noindent 
	Then, there exists a $q \in \mathcal{Q}_n$ with  $\mathcal{Q}_n$ as in \eqref{e:var-family-v} such that
	\begin{equation}
	\label{e:q-bound-v}
	d_{KL}(q(.),\pi(.|\bb{y}_n,\bb{X}_n))=o_{P_0^n}(n^{1-\delta})
	\end{equation}
	
\end{proposition}

\begin{proof}
	\begin{align*}
	d_{KL}(q(.),\pi(.|\bb{y}_n,\bb{X}_n))&=\int q(\bb{\omega}_n)\log q(\bb{\omega}_n)d\bb{\omega}_n-\int q(\bb{\omega}_n) \log \pi(\bb{\omega}_n|\bb{y}_n,\bb{X}_n)d\bb{\omega}_n\\
	&=\int q(\bb{\omega}_n)\log q(\bb{\omega}_n)d\bb{\omega}_n- \int q(\bb{\omega}_n) \log \frac{L(\bb{\omega}_n)p(\bb{\omega}_n)}{\int L(\bb{\omega}_n)p(\bb{\omega}_n)d\bb{\omega}_n} d\bb{\omega}_n\\
	&=\underbrace{d_{KL}(q(.),p(.))}_{\textcircled{1}}\underbrace{-\int q(\bb{\omega}_n) \log  \frac{L(\bb{\omega}_n)}{L_0}d\bb{\omega}_n}_{\textcircled{2}}+\underbrace{\log \int p(\bb{\omega}_n)\frac{L(\bb{\omega}_n)}{L_0}  d\bb{\omega}_n}_{\textcircled{3}}
	\end{align*}
	
	\noindent We first deal with $\textcircled{1}$ as follows
	\begin{equation}
	p(\bb{\omega}_n)=\underbrace{\frac{\lambda^{\alpha}}{\Gamma(\alpha)}\Big(\frac{1}{\sigma^2}\Big)^{\alpha+1}e^{-\frac{\lambda}{\sigma^2}}}_{p(\sigma^2)}\underbrace{\prod_{i=1}^{K(n)}\frac{1}{\sqrt{2\pi \zeta^2}}e^{-\frac{\theta_{in}^2}{2\zeta^2}}}_{p(\bb{\theta}_n)}
	\:\:\:\:q(\bb{\omega}_n)=\underbrace{\frac{(n\sigma_0^2)^{n}}{\Gamma(n)}\Big(\frac{1}{\sigma^2}\Big)^{n+1}e^{-\frac{n\sigma_0^2}{\sigma^2}}}_{q(\sigma^2)}\underbrace{\prod_{i=1}^{K(n)}\sqrt{\frac{n}{2\pi \tau^2}}e^{-\frac{(\theta_{in}-\theta_{i0n})^2}{\tau^2}}}_{q(\bb{\theta}_n)}
	\end{equation}
	\begin{align} 
	\label{e:q-q-0-v}
	\nonumber & d_{KL}(q(.),p(.))=\int q(\bb{\omega}_n)\log q(\bb{\omega}_n)d\bb{\omega}_n-\int q(\bb{\omega}_n)\log p(\bb{\omega}_n)d\bb{\omega}_n\\
	\nonumber&=\int q(\sigma^2) \log q(\sigma^2)d\sigma^2-\int q(\sigma^2) \log p(\sigma^2)d\sigma^2+\int q(\bb{\theta}_n)\log q(\bb{\theta}_n)d\bb{\theta}_n-\int q(\bb{\theta}_n)\log p(\bb{\theta}_n)d\bb{\theta}_n\\
&=\int q(\sigma^2) \log q(\sigma^2)d\sigma^2-\int q(\sigma^2) \log p(\sigma^2)d\sigma^2+o(n^{1-\delta})
	\end{align}
where the last inequality is a consequence of Proposition \ref{lem:q-bound}.
Simplifying further, we get 
\begin{align*}
\int q(\sigma^2) \log q(\sigma^2)d\sigma^2&=\int \left(n \log n\sigma_0^2-\log \Gamma(n)-(n+1)\log \sigma^2-\frac{n\sigma_0^2}{\sigma^2}\right) \frac{(n\sigma_0^2)^{n}}{\Gamma(n)}\Big(\frac{1}{\sigma^2}\Big)^{n+1}e^{-\frac{n\sigma_0^2}{\sigma^2}}d\sigma^2\\
&=n\log n\sigma_0^2-\log \Gamma(n)-(n+1)(\log n\sigma_0^2-\psi(n))-n\\
&= -\log \sigma_0^2-(n+1)\psi(n)-\log (n-1)!-n\\
&= -\log \sigma_0^2-(n+1)\log n -(n-1)\log (n-1)+(n-1)-n+O(\log n)\\
&=-\log \sigma_0^2+O(\log n)=o(n^{1-\delta})
\end{align*}
where the equality in step 4 follows by approximating $\psi(n)$ using Lemma 4  in \cite{ENG} and approximating $(n-1)!$ by Stirling's formula.
\begin{align*}
\int q(\sigma^2) \log p(\sigma^2)d\sigma^2&=\int \left(\alpha \log\lambda-\log \Gamma(\alpha)-(\alpha+1)\log \sigma^2-\frac{\lambda}{\sigma^2}\right) \frac{(n\sigma_0^2)^{n}}{\Gamma(n)}\Big(\frac{1}{\sigma^2}\Big)^{n+1}e^{-\frac{n\sigma_0^2}{\sigma^2}}d\sigma^2\\
&=\alpha\log\lambda-\log \Gamma(\alpha)-(\alpha+1)(\log n\sigma_0^2-\psi(n)))-\frac{\lambda}{\sigma_0^2}\\
&=\alpha\log\lambda-\log \Gamma(\alpha)-(\alpha+1)(\log n\sigma_0^2-\log n)-\frac{\lambda}{\sigma_0^2}+O(\log n)=o(n^{1-\delta})
\end{align*}	
where the last equality follows by approximating $\psi(n)$ using Lemma 4  in \cite{ENG}.

	\noindent For, $\textcircled{2}$ note that
	\begin{align}
	\label{e:dk-bound-1}
	\nonumber d_{KL}(l_0,l_{\bb{\omega}_n})&= \int \int \Big( \frac{1}{2}\log\frac{\sigma^2}{\sigma_0^2}-\frac{1}{2\sigma_0^2}(y-f_0(\bb{x}))^2+\frac{1}{2\sigma^2}(y-f_{\bb{\theta}_n}(\bb{x}))^2\Big)\frac{1}{\sqrt{2\pi\sigma_0^2}}e^{-\frac{(y-f_0(\bb{x}))^2}{2\sigma_0^2}}dy d\bb{x}\\
	&=\frac{1}{2}\log\frac{\sigma^2}{\sigma_0^2}-\frac{1}{2}+\frac{\sigma_0^2}{2\sigma^2}+\frac{1}{2\sigma^2}\int (f_{\bb{\theta}_n}(\bb{x})-f_0(\bb{x}))^2 d\bb{x}
		\end{align}
	By Lemmas \ref{lem:h-sig-bound}, \ref{lem:h-siginv-bound} and Lemma \ref{lem:f-bound} part 1, we have
	$$\int d_{KL}(l_0,l_{\bb{\omega}_n})q(\bb{\omega}_n)d\bb{\omega}_n=o(n^{-\delta})$$
	
	\noindent Therefore, by Lemma \ref{lem:c-bound-0}, $\textcircled{2}=o_{P_0^n}(n^{-\delta})$.
	
	\noindent Using  Proposition \ref{lem:kl-bound-v} in Lemma \ref{lem:b-bound-0}, we get $\textcircled{3}= o_{P_0^n}(n^{1-\delta})$.

\end{proof}

\subsection{Lemmas and Propositions for Theorem \ref{thm:var-rho}}
\begin{lemma}
	\label{lem:e-bound-r}
	For $\mathcal{G}_n$ as in \eqref{e:G-def-r}, let $\widetilde{\mathcal{G}}_n=\{\sqrt{g}: g \in \mathcal{G}_n\}$.  If $K(n)\sim n^a$, $C_n=e^{n^{b-a}}$, $0<a<b<1$, then	$$\frac{1}{\sqrt{n}}\int_0^{\varepsilon}\sqrt{H_{[]}(u,\widetilde{\mathcal{G}}_n,||.||_2)}du\leq \varepsilon^2$$
\end{lemma}
\begin{proof} 
	First, by Lemma 4.1 in \cite{POL1990},
	$$N(\varepsilon,\mathcal{F}_n,||.||_\infty)\leq \Big(\frac{3C_n }{\varepsilon}\Big)^{K(n)}\Big(\frac{3\log C_n}{\varepsilon}\Big)$$
	
	\noindent For $\bb{\omega}_1, \bb{\omega}_2 \in \mathcal{F}_n$, let $\widetilde{L}(u)=\sqrt{L_{u\bb{\omega}_1+(1-u)\bb{\omega}_2}(\bb{x},y)}$.

\noindent	Using \eqref{e:l-bd-0}, we get
	\begin{align}
	\label{e:l-bd-r} \sqrt{L_{\bb{\omega}_1}(\bb{x},y)}-\sqrt{L_{\bb{\omega}_2}(\bb{x},y)}&\leq  \underbrace{(K(n)+1)\sup_{i}  \Big|\pp{\widetilde{L}}{\omega_i}\Big|}_{F(\bb{x},y)}||\omega_1-\omega_2||_{\infty}\leq F(\bb{x},y)||\omega_1-\omega_2||_{\infty}
	\end{align}
	where the upper bound on $F(\bb{x},y)$ is calculated as:
	\begin{align*}
	|\pp{\widetilde{L}}{\beta_j}|&\leq 2^{3/2}(8\pi e^2)^{-1/4}C_n^{3/2}, j=0,\cdots, k_n\\
	|\pp{\widetilde{L}}{\gamma_{jh}}|&\leq 2^{3/2}(8\pi e^2)^{-1/4} C_n^{5/2}, j=0,\cdots, k_n, h=0,\cdots, p\\
	|\pp{\widetilde{L}}{\rho}|&\leq 2^{3/2}((16\pi)^{-1/4}+(\pi e^2/8)^{-1/4}) C_n^{5/2}
	\end{align*}
since	$\log(1+e^{\rho})\geq \log(1+e^{-\log C_n}) \sim 1/C_n \geq  1/(2C_n)$ and $|\partial{\log(1+e^\rho)}/\partial{\rho}|\leq 1$. 

\noindent In view of \eqref{e:l-bd-r}  and Theorem 2.7.11 in \cite{VW1996}, we have
$$N_{[]}(\varepsilon, \widetilde{\mathcal{G}}_n, ||.||_2)\leq \Big(\frac{MK(n)C_n^{7/2}}{\varepsilon}\Big)^{K(n)}\Big(\frac{MK(n)C_n^{5/2} \log C_n }{\varepsilon}\Big)$$
for some $M>0$. Therefore, 
$$H_{[]}(\varepsilon, \widetilde{\mathcal{G}}_n, ||.||_2)\lesssim K(n) \log \frac{K(n)C_n^{7/2} (K(n)C_n^{5/2} \log C_n)^{1/K(n)}}{\varepsilon}$$
Using, Lemma \ref{lem:h-bd} with $M_n=K(n)C_n^{7/2} (K(n)C_n^{5/2} \log C_n)^{1/K(n)}$, we get
$$\int_0^{\varepsilon} \sqrt {H_{[]}(u, \widetilde{\mathcal{G}}_n, ||.||_2)} du\leq \varepsilon O(\sqrt{K(n)\log( K(n)C_n^{7/2} (K(n)C_n^{5/2} \log C_n)^{1/K(n)}})=\varepsilon O(\sqrt{n^b}) $$
where the last equality holds since $K(n)\sim n^a$, $C_n=e^{n^{b-a}}$, $0<a<b<1$.

\noindent Therefore, 
$$\frac{1}{\sqrt{n}}\int_0^{\varepsilon} {H_{[]}(u, \widetilde{\mathcal{G}}_n, ||.||_2)} du\leq \varepsilon^2$$
\end{proof}

\begin{lemma}
	\label{lem:p-bound-r}
	Let $$\mathcal{F}_n=\Big\{ (\bb{\theta}_n,\rho):|\theta_{in}|\leq C_n, i=1,\cdots, K(n), |\rho|\leq \log C_n\Big\}$$
	where $K(n)\sim n^a$, $C_n=e^{n^{b-a}}$, $0<a<1/2$, $a+1/2<b<1$. Then with 
	$$p(\bb{\omega}_n)=\frac{1}{\sqrt{2\pi\eta^2}}e^{-\frac{\rho^2}{2\eta^2}}\prod_{i=1}^{K(n)}\frac{1}{\sqrt{2\pi \zeta^2}}e^{-\frac{\theta_{in}^2}{2\zeta^2}}$$
	we have for every $\kappa>0$
	$$\int_{\bb{\omega}_n \in \mathcal{F}_n^c}p(\bb{\omega}_n)d\bb{\omega}_n\leq e^{-n\kappa},\:\: n \to \infty $$		
\end{lemma}

\begin{proof}
	Let $\mathcal{F}_{in}=\{\theta_{in}: |\theta_{in}|\leq C_n\}$ and $\mathcal{F}_{0n}=\{\rho: |\rho|<\log C_n\}$. 
	$$\mathcal{F}_n=\mathcal{F}_{0n}\cap_{i=1}^{K(n)} \mathcal{F}_{in}\implies \mathcal{F}_n^c=\mathcal{F}_{0n}^c\cup \cup_{i=1}^{K(n)}\mathcal{F}_{in}^c$$
	\begin{align*}
	\int_{\bb{\omega}_n \in \mathcal{F}_n^c}p(\bb{\omega}_n)d\bb{\omega}_n&\leq \int_{\mathcal{F}_{0n}^c} \frac{1}{\sqrt{2\pi\eta^2}} e^{-\frac{\rho^2}{2\eta^2}}d\rho+\sum_{i=1}^{K(n)}\int_{\mathcal{F}_{in}^c}\frac{1}{\sqrt{2\pi \zeta^2}}e^{-\frac{\theta_{in}^2}{2\zeta^2}}d\theta_{in}^2 \hspace{10mm} \text{Countable sub-additivity.}\\
	&=2\int_{\log C_n}^{\infty} \frac{1}{\sqrt{2\pi\eta^2}} e^{-\frac{\rho^2}{2\eta^2}}d\rho+2\sum_{i=1}^{K(n)}\int_{C_n}^{\infty}\frac{1}{\sqrt{2\pi \zeta^2}}e^{-\frac{\theta_{in}^2}{2\zeta^2}}d\theta_{in}^2\\
	&=2\left(1-\Phi\left(\frac{\log C_n}{\eta}\right)\right)+2K(n)\left(1-\Phi\left(\frac{C_n}{\zeta}\right)\right)\\
	&\sim \frac{1}{\log C_n}e^{-\frac{(\log C_n)^2}{2\eta^2}}+\frac{K(n)}{C_n}e^{-\frac{C_n^2}{2\zeta^2}} \leq e^{-n\kappa} \hspace{10mm}\text{ By Mill's Ratio}
	\end{align*}
	since $(\log C_n )^2=n^{2(b-a)}> n$ for $a+1/2<b<1$ and $C_n^2=e^{2n^{b-a}}>n$ for $0<a<b<1$.
\end{proof}

\begin{proposition}
	\label{lem:v-bound-r}
	Suppose condition (C1) holds with $0<a<1/2$ and  $p(\bb{\omega}_n)$ satisfies \eqref{e:prior-t-r}.
	Then,
	$$\log \int_{\mathcal{V}_\varepsilon^c} \frac{L(\bb{\omega}_n)}{L_0} p(\bb{\omega}_n)d\bb{\omega}_n\leq \log 2-n\varepsilon^2 + o_{P_0^n}(1)$$	
\end{proposition}

\begin{proof}
Let $\mathcal{F}_n=\{\bb{\omega}_n:|\theta_{in}|\leq C_n, |\rho|<\log C_n\}$. Let $C_n =e^{n^{b-a}}$ and $K(n) \sim n^a$ for $0<a<1/2$.

\noindent By Lemma  \ref{lem:e-bound-r}, we have $$\frac{1}{\sqrt{n}}\int_0^{\varepsilon}{H_{[]}(u, \widetilde{\mathcal{G}}_n, ||.||_2)} du \leq \varepsilon^2$$
Therefore, by Lemma \ref{lem:d-bound-0}, we have 
$$ P_0^n\left(\int_{\mathcal{V}_\varepsilon^c \cap \mathcal{F}_n} \frac{L(\bb{\omega}_n)}{L_0} p(\bb{\omega}_n)d\bb{\omega}_n\geq e^{-n\varepsilon^2}\right) \to 0$$
\noindent In view of Lemma \ref{lem:p-bound-r}, for $p(\bb{\omega}_n)$ as in \eqref{e:prior-t-r},
	$$\int_{\bb{\omega}_n \in \mathcal{F}_n^c} p(\bb{\omega}_n)d\bb{\omega}_n \leq e^{-2n \varepsilon^2}$$
	Therefore, by Lemma \ref{lem:pp-bound-0} with $r=1$, $\kappa=2\varepsilon^2$ and $\tilde{\kappa}=\varepsilon^2$, we have
	$$P_0^n\left(\int_{\mathcal{F}_n^c} \frac{L(\bb{\omega}_n)}{L_0} p(\bb{\omega}_n)d\bb{\omega}_n\geq e^{-n\varepsilon^2}\right) \to 0$$
The remaining part of the proof follows on the same lines as Proposition \ref{lem:v-bound}	
\end{proof}

\begin{proposition}
	\label{lem:kl-bound-r}
	Suppose condition (C1) holds with some $0<a<1$. Let $f_{\bb{\theta}_n}$ be a neural network satisfying assumption (A1) and (A2) for some  $0\leq \delta<1-a$. With $\bb{\omega}_n=(\bb{\theta}_n,\rho)$, define, 
	\begin{equation}
	\label{e:Nk-def-r}
	N_{\kappa/n^{\delta}}=\{\bb{\omega}_n:d_{KL}(l_0,l(\bb{\omega}_n))=\frac{1}{2}\log\frac{\sigma_\rho^2}{\sigma_0^2}-\frac{1}{2}\Big(1-\frac{\sigma_0^2}{\sigma_\rho^2}\Big)+\frac{1}{2\sigma_\rho^2}\int (f_{\bb{\theta}_n}(\bb{x})-f_0(\bb{x}))^2 d\bb{x}<\epsilon\}
	\end{equation}
	\item 	For every $\tilde{\kappa}>0$, with $p(\bb{\omega}_n)$ as in \eqref{e:prior-t-r}, we have  
	$$\int_{\bb{\omega}_n\in N_{\kappa/n^{\delta}}} p(\bb{\omega}_n)d\bb{\omega}_n \geq e^{-\tilde{\kappa}n^{1-\delta}}, \:\: n \to \infty.$$
\end{proposition}
\begin{proof}
This proof uses some ideas from the proof of Theorem 1 in \cite{LEE}. 

\noindent By assumption (A1), let $f_{\bb{\theta}_{0n}}(\bb{x})=\beta_{00}+\sum_{j=1}^{k(n)}\beta_{j0}\psi(\gamma_{j0}^\top \bb{x})$ satisfy
	\begin{equation}
	\label{e:n-delta-r}
	||f_{\bb{\theta}_{0n}}-f_0||_2\leq \frac{\kappa }{8n^{\delta}}
	\end{equation}
	
	\noindent With $\sigma_0=\log(1+e^{\rho_0})$, define neighborhood $M_\kappa$ as follows
	\begin{align*}
	M_{\kappa}=\{&\bb{\omega}_n:|\rho-\rho_0|<\sqrt{\kappa/2n^{\delta}}\sigma_0,|{\theta}_{in}-{\theta}_{i0n}|<\sqrt{\kappa/( 8n^{\delta}m_n)}\sigma_0, i=1,\cdots, K(n)\}
	\end{align*}
	where $m_n=8K(n)^2+8(p+1)^2(\sum_{j=1}^{K(n)}|\theta_{i0n}|)^2$. Note that $m_n\geq 8k_n+8(p+1)^2(\sum_{j=1}^{k_n}|\beta_{j0}|)^2$. 
	
	\noindent Thereby, using Lemma \ref{lem:theta-bound} with $\epsilon=\sqrt{\kappa/(8n^\delta m_n)}\sigma_0$ and \eqref{e:aplusb},  we get
	\begin{equation}
	\label{e:n-delta-1-r}
	\int (f_{\bb{\theta}_n}(\bb{x})-f_0(\bb{x}))^2 d\bb{x} \leq 2||f_{\bb{\theta}_n}-f_{\bb{\theta}_{0n}}||_2+2||f_{\bb{\theta}_{0n}}-f_{0}||_2 \leq \frac{\kappa\sigma_0^2}{2n^{\delta}}
	\end{equation}
	\noindent By Lemma \ref{lem:rho-bound},
	\begin{align}
	\label{e:n-delta-2-r}
	\nonumber\frac{1}{2}\log\frac{\sigma_\rho^2}{\sigma_0^2}-\frac{1}{2}\Big(1-\frac{\sigma_0^2}{\sigma_\rho^2}\Big)&\leq \frac{\kappa}{2n^{\delta}}\\
	\frac{1}{2\sigma_\rho^2}\leq \frac{1}{2\sigma_0^2(1-\sqrt{\kappa/2n^{\delta}})^2}&\leq \frac{1}{\sigma_0^2}
	\end{align}
	
	\noindent Using \eqref{e:n-delta-1-r} and \eqref{e:n-delta-2-r} in \eqref{e:Nk-def-r}, we get $\bb{\omega}_n \in N_{\kappa/n^\delta}$, for every $\bb{\omega}_n \in M_\kappa$.
	Therefore, 
	$$\int_{\bb{\omega}_n \in N_{\kappa/n^\delta}} p(\bb{\omega}_n)d\bb{\omega}_n\geq \int_{\bb{\omega}_n \in M_{\kappa}} p(\bb{\omega}_n)d\bb{\omega}_n$$
	
	\noindent We next show that,
	$$\int_{\bb{\omega}_n \in M_{\kappa}} p(\bb{\omega}_n)d\bb{\omega}_n>e^{-\tilde{\kappa}n^{1-\delta}}$$
	
		\noindent For notation simplicity, let $\delta_{1n}=\sqrt{\kappa/2n^{\delta}}\sigma_0$ and $\delta_{2n}=\sqrt{\kappa/(8 n^{\delta} m_n)}\sigma_0$
	\begin{align*}
	\int_{\bb{\omega}_n \in \nonumber M_{\kappa}}p(\bb{\omega}_n)d\bb{\omega}_n&=\int_{\rho_0-\delta_{1n}}^{\rho_0+\delta_{1n}}p(\rho)d\rho\prod_{i=1}^{K(n)}\int_{\theta_{i0n}-\delta_{2n}}^{\theta_{i0n}+\delta_{2n}} p(\theta_{in})d\theta_{in}\\
	&\geq \int_{\rho_0-\delta_{1n}}^{\rho_0+\delta_{1n}}p(\rho)d\rho e^{-(\tilde{\kappa}/2)n^{1-\delta}}
	\end{align*}
where first to second step follows from part 1. of Lemma \ref{lem:kl-bound} since $p(\bb{\theta}_{n})$ satisfies \eqref{e:prior-t}. Next, 
\begin{align}
	\label{e:p-rel-1-r}
	\nonumber\int_{\rho_0-\delta_{1n}}^{\rho_0+\delta_{1n}}p(\rho)d\rho&=\int_{\rho_0-\delta_{1n}}^{\rho_0+\delta_{1n}} \frac{1}{\sqrt{2\pi\eta^2}} e^{-\frac{\rho^2}{2\eta^2}}\\
\nonumber	&=2\delta_{1n} \frac{1}{\sqrt{2\pi \eta^2}}e^{-\frac{t^2}{2\eta^2}},t \in [\rho_0-\delta_{1n},\rho_0+\delta_{1n}]\hspace{10mm} \text{by mean value theorem}\\
\nonumber&\geq \frac{2\delta_{1n}}{\sqrt{2\pi\eta^2}}e^{-\frac{\max((\rho_0-\epsilon)^2,(\rho_0+\epsilon)^2)}{2\eta^2}}\\
&=\exp\left(-\left(-\log \delta_{1n}+\frac{1}{2}\log \frac{\pi}{2}+\log \eta +\frac{\max((\rho_0-\epsilon)^2,(\rho_0+\epsilon)^2)}{2\eta^2}\right)\right)
	\end{align}
	where the third inequality holds since for any $\epsilon>0$, $t\in [\rho_0-\epsilon,\rho_0+\epsilon]$  when $\delta_n \to 0$. Now, \begin{align}
	\label{e:s-bound-r}
	\nonumber&-\log \delta_{1n}+\frac{1}{2} \log \frac{\pi}{2}+\log\eta + \frac{\max(\rho_0-\epsilon,\rho_0+\epsilon)}{2\eta^2}\\
&=\frac{1}{2}\delta \log n +\frac{1}{2} \log 2 -\frac{1}{2}\log \kappa-\log \sigma_0+\log\eta +\frac{\max(\rho_0-\epsilon,\rho_0+\epsilon)}{2\eta^2}\leq (\tilde{\kappa}/2)n^{1-\delta}
	\end{align}
	
\noindent Using \eqref{e:s-bound-r} in \eqref{e:p-rel-1-r}, we get
	$$\int_{\bb{\omega}_n \in M_\kappa} p(\bb{\omega}_n)d\bb{\omega}_n \geq e^{-\tilde{\kappa}n^{1-\delta}}$$
	which completes the proof.
\end{proof}

\begin{proposition}
	\label{lem:q-bound-r}
	Suppose condition (C1) and  assumption (A1) hold for some $0<a<1/2$ and $0\leq \delta<1-a$.
 Suppose the prior $p(\bb{\omega}_n)$ satisfies as \eqref{e:prior-t-r}.
	
\noindent	Then, there exists a $q \in \mathcal{Q}_n$ with  $\mathcal{Q}_n$ as in \eqref{e:var-family-r}, such that
	\begin{equation}
	\label{e:q-bound-1}
	d_{KL}(q(.),\pi(.|\bb{y}_n,\bb{X}_n))=o_{P_0^n}(n^{1-\delta})
	\end{equation}
	
\end{proposition}

\begin{proof}
	\begin{align*}
	d_{KL}(q(.),\pi(.|\bb{y}_n,\bb{X}_n))&=\int q(\bb{\omega}_n)\log q(\bb{\omega}_n)d\bb{\omega}_n-\int q(\bb{\omega}_n) \log \pi(\bb{\omega}_n|\bb{y}_n,\bb{X}_n)d\bb{\omega}_n\\
	&=\int q(\bb{\omega}_n)\log q(\bb{\omega}_n)d\bb{\omega}_n- \int q(\bb{\omega}_n) \log \frac{L(\bb{\omega}_n)p(\bb{\omega}_n)}{\int L(\bb{\omega}_n)p(\bb{\omega}_n)d\bb{\omega}_n} d\bb{\omega}_n\\
	&=\underbrace{d_{KL}(q(.),p(.))}_{\textcircled{1}}\underbrace{-\int q(\bb{\omega}_n) \log  \frac{L(\bb{\omega}_n)}{L_0}d\bb{\omega}_n}_{\textcircled{2}}+\underbrace{\log \int p(\bb{\omega}_n)\frac{L(\bb{\omega}_n)}{L_0}  d\bb{\omega}_n}_{\textcircled{3}}
	\end{align*}
	
	\noindent We first deal with $\textcircled{1}$ as follows
	\begin{equation}
	p(\bb{\omega}_n)=\underbrace{\frac{1}{\sqrt{2\pi\eta^2}}e^{-\frac{\rho^2}{2\eta^2}}}_{p(\rho)}\underbrace{\prod_{i=1}^{K(n)}\frac{1}{\sqrt{2\pi \zeta^2}}e^{-\frac{\theta_{in}^2}{2\zeta^2}}}_{p(\bb{\theta}_n)}
	\:\:\:\:q(\bb{\omega}_n)=\underbrace{\sqrt{\frac{n}{2\pi \nu^2}}e^{-\frac{n(\rho-\rho_0)^2}{\nu^2}}}_{q(\rho)}\underbrace{\prod_{i=1}^{K(n)}\sqrt{\frac{n}{2\pi \tau^2}}e^{-\frac{(\theta_{in}-\theta_{i0n})^2}{\tau^2}}}_{q(\bb{\theta}_n)}
	\end{equation}
	
	\begin{align} 
	\label{e:q-q-0-r}
	\nonumber d_{KL}(q(.),p(.))&=\int q(\rho) \log q(\rho)d\rho-\int q(\rho) \log p(\rho)d\rho+\int q(\bb{\theta}_n)\log q(\bb{\theta}_n)d\bb{\theta}_n-\int q(\bb{\theta}_n)\log p(\bb{\theta}_n)d\bb{\theta}_n\\
	&=\int q(\rho) \log q(\rho)d\rho-\int q(\rho) \log p(\rho)d\rho+o(n^{1-\delta})
	\end{align}
	where the last equality is a consequence of Proposition \ref{lem:q-bound}.
	Simplifying further, we get 
\begin{align*}	
\int q(\rho)\log q(\rho) d\rho-\int q(\rho)\log q(\rho) d\rho&=\int \Big(\frac{1}{2}\log n -\frac{1}{2}\log 2\pi -\log \nu-\frac{n(\rho-\rho_0)^2}{2\nu^2}\Big)\frac{n}{\sqrt{2\pi\nu^2}}e^{-\frac{n(\rho-\rho_0)^2}{2\nu^2}}d\rho\\
\nonumber&-\int \Big(-\frac{1}{2}\log 2\pi -\log \eta-\frac{\rho^2}{2\eta^2}\Big)\frac{n}{\sqrt{2\pi\nu^2}}e^{-\frac{n(\rho-\rho_0)^2}{2\nu^2}}d\rho\\
&=\frac{1}{2}(\log n-\log 2\pi-2\log \nu-1)+\frac{1}{2}(-\log 2\pi-2\log \eta)+\frac{\rho_0^2+\nu^2/n}{2\eta^2}\\
&
=o(n^{1-\delta})
\end{align*}	
\noindent For, $\textcircled{2}$ note that
	\begin{align}
	\label{e:dk-bound-r}
	\nonumber d_{KL}(l_0,l_{\bb{\omega}_n})&= \int \int \Big( \frac{1}{2}\log\frac{\sigma_\rho^2}{\sigma_0^2}-\frac{1}{2\sigma_0^2}(y-f_0(\bb{x}))^2+\frac{1}{2\sigma_\rho^2}(y-f_{\bb{\theta}_n}(\bb{x}))^2\Big)\frac{1}{\sqrt{2\pi\sigma_0^2}}e^{-\frac{(y-f_0(\bb{x}))^2}{2\sigma_0^2}}dy d\bb{x}\\
	&=\frac{1}{2}\log\frac{\sigma_\rho^2}{\sigma_0^2}-\frac{1}{2}+\frac{\sigma_0^2}{2\sigma_\rho^2}+\frac{1}{2\sigma_\rho^2}\int (f_{\bb{\theta}_n}(\bb{x})-f_0(\bb{x}))^2 d\bb{x}\end{align}
	By Lemmas \ref{lem:h-rho-bound}, \ref{lem:h-rhoinv-bound} and Lemma \ref{lem:f-bound} part 1, we have
	$$\int d_{KL}(l_0,l_{\bb{\omega}_n}) q(\bb{\omega}_n)d\bb{\omega}_n=o_{P_0}(n^{-\delta})$$
	
	\noindent Therefore, by Lemma \ref{lem:c-bound-0}, $\textcircled{2}=o_{P_0}(n^{-\delta})$.
	
	\noindent Using  Lemma \ref{lem:kl-bound-r} in Lemma \ref{lem:b-bound-0}, we get $\textcircled{3}= o_{P_0}(n^{1-\delta})$.

\end{proof}

\bibliography{refVBNN}

\end{document}